\documentclass{article} %
\usepackage{iclr2023_conference,times}

\usepackage{amsmath,amsfonts,bm}

\def\eqref#1{equation~\ref{#1}}

\def\1{\bm{1}}

\DeclareMathAlphabet{\mathsfit}{\encodingdefault}{\sfdefault}{m}{sl}
\SetMathAlphabet{\mathsfit}{bold}{\encodingdefault}{\sfdefault}{bx}{n}

\usepackage{colortbl}
\usepackage{tabulary}
\usepackage{etoolbox}
\newcommand{\bgl}{\cellcolor[HTML]{DDDDDD}}
\newcommand{\bgd}{\cellcolor[HTML]{BBBBBB}}

\usepackage{hyperref}
\usepackage{url}
\usepackage{graphicx}
\usepackage{amsmath}
\usepackage{amsthm}
\usepackage{amssymb}
\usepackage{booktabs}
\usepackage{bbm}
\usepackage{multicol}
\usepackage{multirow}
\usepackage{graphicx}
\usepackage{xcolor}
\usepackage{enumitem}
\usepackage{makecell}

\newtheorem{proposition}{Proposition}
\newtheorem{definition}{Definition}
\newtheorem*{prop1}{Proposition 1 in the main paper}
\newtheorem*{prop3}{Proposition 3 in the main paper}
\newtheorem*{prop4}{Proposition 4 in the main paper}

\newcommand{\yingda}[1]{{\color{black}#1}}
\newcommand{\ree}[1]{{\color{black}{#1}}}

\newcommand{\SO}{\mathrm{SO}(3)}
\newcommand{\len}[1]{\left\lVert {#1} \right\rVert }
\newcommand{\tr}[1]{\operatorname{tr (#1)}}

\makeatletter
\def\@fnsymbol#1{\ensuremath{\ifcase#1\or \dagger\or \ddagger\or
   \mathsection\or \mathparagraph\or \|\or **\or \dagger\dagger
   \or \ddagger\ddagger \else\@ctrerr\fi}}
\makeatother

\title{A Laplace-inspired Distribution on SO(3) for Probabilistic Rotation Estimation}

\author{
Yingda Yin \hskip 3em 
Yang Wang \hskip 3em 
He Wang\thanks{He Wang and Baoquan Chen are the corresponding authors
(\{hewang, baoquan\}@pku.edu.cn).
}  \hskip 3em 
Baoquan Chen\footnotemark[1]  \\
Peking University \\
}

\iclrfinalcopy %
\begin{document}

\maketitle

\begin{abstract}
\yingda{
Estimating the 3DoF rotation from a single RGB image is an important yet challenging problem. Probabilistic rotation regression has raised more and more attention with the benefit of 
expressing uncertainty information along with the prediction.
Though modeling noise using Gaussian-resembling Bingham distribution and matrix Fisher distribution is natural, they are shown to be sensitive to outliers for the nature of quadratic punishment to deviations.
In this paper, we draw inspiration from multivariate Laplace distribution and propose a novel Rotation Laplace distribution on $\SO$. Rotation Laplace distribution is robust to the disturbance of outliers and enforces much gradient to the low-error region, resulting in a better convergence.
Our extensive experiments show that our proposed distribution achieves  state-of-the-art performance for rotation regression tasks over both probabilistic and non-probabilistic baselines.
Our project page is at  \href{https://pku-epic.github.io/RotationLaplace/}{pku-epic.github.io/RotationLaplace}.
}
\end{abstract}

\section{Introduction}

\yingda{
Incorporating neural networks to perform rotation regression is of great importance in the field of computer vision, computer graphics and robotics \citep{wang2019normalized,yin2022fishermatch,dong2021robust,breyer2021volumetric}. To close the gap between the $\SO$ manifold and the Euclidean space where neural network outputs exist, one popular line of research discovers learning-friendly rotation representations including 6D continuous representation \citep{zhou2019continuity}, 9D matrix representation with SVD orthogonalization \citep{levinson2020analysis}, etc. Recently, \cite{chen2022projective} focuses on the gradient backpropagating process and replaces the vanilla auto differentiation with a $\SO$ manifold-aware gradient layer, which sets the new state-of-the-art in rotation regression tasks. 

Reasoning about the uncertainty information along with the predicted rotation is also attracting more and more attention, which enables many applications in aerospace \citep{crassidis2003unscented}, autonomous driving \citep{mcallister2017concrete} and localization \citep{fang2020towards}.
}
On this front, recent efforts have been developed to model the uncertainty of rotation regression via probabilistic modeling of rotation space. The most commonly used distributions are Bingham distribution \citep{bingham1974antipodally} on $\mathcal{S}^3$ for unit quaternions and matrix Fisher distribution \citep{khatri1977mises} on $\SO$ for rotation matrices. These two distributions are equivalent to each other \citep{prentice1986orientation} and resemble the Gaussian distribution in Euclidean Space \citep{bingham1974antipodally,khatri1977mises}.
While modeling noise using Gaussian-like distributions is well-motivated by the Central Limit Theorem, Gaussian distribution is well-known to be sensitive to outliers in the probabilistic regression models \citep{murphy2012machine}. This is because Gaussian distribution penalizes deviations quadratically, so predictions with larger errors weigh much more heavily with the learning than low-error ones and thus potentially result in suboptimal convergence when a certain amount of outliers exhibit.

Unfortunately, in certain rotation regression tasks, we fairly often come across large prediction errors, \textit{e.g.} $180^\circ$ error,  due to either the (near) symmetry nature of the objects or severe occlusions \citep{murphy2021implicit}. 
In Fig. \ref{fig:teaser}(left), using training on single image rotation regression as an example, we show the statistics of predictions 
after achieving convergence, assuming matrix Fisher distribution (as done in \cite{mohlin2020probabilistic}). The blue histogram shows the population with different prediction errors and the red dots are the impacts of these predictions on learning, evaluated by computing the sum of their gradient magnitudes  $\|\partial \mathcal{L} / \partial (\text{distribution param.})\|$    within each bin and then normalizing them across bins.  
It is clear that the 180$^\circ$ outliers dominate the gradient as well as the network training though their population is tiny, while the vast majority of points with low error predictions are deprioritized. Arguably, at convergence, the gradient should focus more on refining the low errors rather than fixing the inevitable large errors (\textit{e.g.} arose from symmetry). This motivates us to find a better probabilistic model for rotation.

As pointed out by \cite{murphy2012machine}, Laplace distribution, with heavy tails, is a better option for robust probabilistic modeling. Laplace distribution drops sharply around its mode and thus allocates most of its probability density to a small region around the mode; meanwhile, it also tolerates and assigns higher likelihoods to the outliers, compared to Gaussian distribution.
Consequently, it encourages predictions near its mode to be even closer, thus fitting \textit{sparse} data well, most of whose data points are close to their mean with the exception of several outliers\citep{mitianoudis2012generalized}, which makes Laplace distribution to be favored in the context of deep learning\citep{goodfellow2016deep}.

In this work, we propose a novel Laplace-inspired distribution on $\SO$ for rotation matrices, namely Rotation Laplace distribution, for probabilistic rotation regression. 
We devise Rotation Laplace distribution to be an approximation of multivariate Laplace distribution in the tangent space of its mode.
As shown in the visualization in Fig. \ref{fig:teaser}(right), our Rotation Laplace distribution is robust to the disturbance of outliers, with most of its gradient contributed by the low-error region, and thus leads to a better convergence along with significantly higher accuracy.
Moreover, our Rotation Laplace distribution is simply parameterized by an unconstrained $3\times3$ matrix and thus accommodates the Euclidean output of neural networks with ease. This network-friendly distribution requires neither complex functions to fulfill the constraints of parameterization nor any normalization process from Euclidean to rotation manifold which has been shown harmful for learning \citep{chen2022projective}.

For completeness of the derivations, we also propose the Laplace-inspired distribution on $\mathcal{S}^3$ for quaternions. We show that Rotation Laplace distribution is equivalent to Quaternion Laplace distribution, similar to the equivalence of matrix Fisher distribution and Bingham distribution.

We extensively compare our Rotation Laplace distributions to methods that parameterize distributions on $\SO$ for pose estimation, and also non-probabilistic approaches including multiple rotation representations and recent $\SO$-aware gradient layer \citep{chen2022projective}.
On common benchmark datasets of rotation estimation from RGB images, we achieve a significant and consistent performance improvement over all baselines.

\begin{figure}[t]
    \centering
    \hspace{-3mm}
    \includegraphics[width=0.48\linewidth]{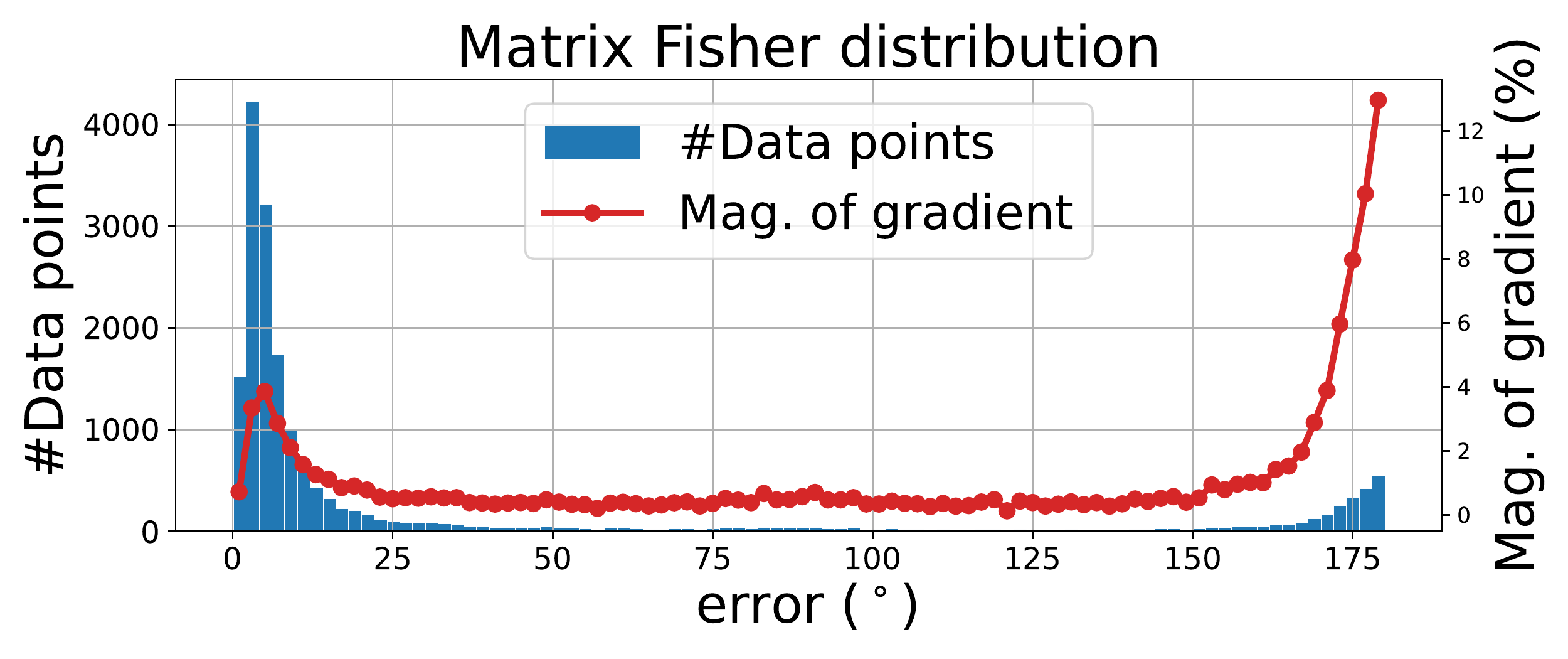}
    \hspace{2mm}
    \includegraphics[width=0.48\linewidth]{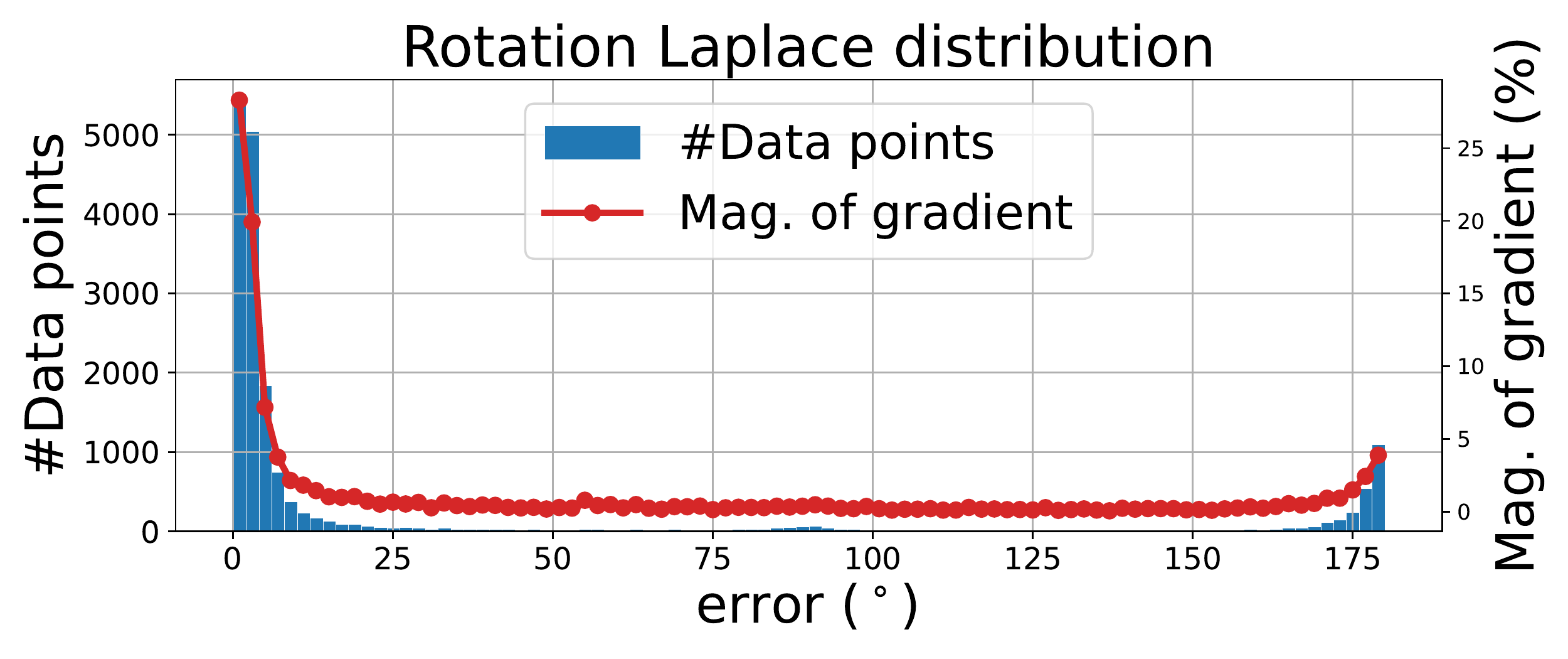}
    \hspace{-3mm}
    \vspace{-3mm}
    \caption{\small \textbf{Visualization of the results of matrix Fisher distribution and Rotation Laplace distribution after convergence.} The horizontal axis is the geodesic distance between the prediction and the ground truth. The blue bins count the number of data points within corresponding errors (2$^\circ$ each bin). \yingda{The red dots illustrate the percentage of the sum of the gradient magnitude  $\|\partial \mathcal{L} / \partial (\text{dist. param.})\|$
    within each bin. }
    The experiment is done on all categories of ModelNet10-SO3 dataset.}
    \vspace{-2mm}
	\label{fig:teaser}
\end{figure}

\section{Related Work}

\paragraph{Probabilistic regression} 
\cite{nix1994estimating} first proposes to model the output of the neural network as a Gaussian distribution and learn the Gaussian parameters by the negative log-likelihood loss function, through which one obtains not only the target but also a measure of prediction uncertainty.
More recently, \cite{kendall2017uncertainties} offers more understanding and analysis of the underlying uncertainties. \cite{lakshminarayanan2017simple} further improves the performance of uncertainty estimation by network ensembling and adversarial training. \cite{makansi2019overcoming} stabilizes the training with the winner-takes-all and iterative grouping strategies.
Probabilistic regression for uncertainty prediction has been widely used in various applications, including optical flow estimation\citep{ilg2018uncertainty}, depth estimation \citep{poggi2020uncertainty}, weather forecasting \citep{wang2019deep}, \textit{etc.}

Among the literature of decades, the majority of probabilistic regression works model the network output by a Gaussian-like distribution, while Laplace distribution is less discovered. 
\cite{li2021human} empirically finds that assuming a Laplace distribution in the process of maximum likelihood estimation yields better performance than a Gaussian distribution, in the field of 3D human pose estimation. Recent work \citep{nair2022maximum} makes use of Laplace distribution to improve the robustness of maximum likelihood-based uncertainty estimation. Due to the heavy-tailed property of Laplace distribution, the outlier data produces comparatively less loss and have an insubstantial impact on training.
Other than in Euclidean space, \cite{mitianoudis2012generalized} develops Generalized
Directional Laplacian distribution in $\mathcal{S}^d$ for the application of audio separation.

\paragraph{Probabilistic rotation regression}

Several works focus on utilizing probability distributions on the rotation manifold for rotation uncertainty estimation. 
\cite{prokudin2018deep} uses the mixture of von Mises distributions \citep{mardia2000directional} over Euler angles using Biternion networks. In \cite{gilitschenski2019deep} and \cite{deng2022deep}, Bingham distribution over unit quaternion is used to jointly estimate a probability distribution over all axes. 
\cite{mohlin2020probabilistic} leverages matrix Fisher distribution \citep{khatri1977mises} on $\SO$ over rotation matrices for deep rotation regression. 
Though both bear similar properties with Gaussian distribution in Euclidean space, matrix Fisher distribution benefits from the continuous rotation representation and unconstrained distribution parameters, which yields better performance \citep{murphy2021implicit}.
Recently, \cite{murphy2021implicit} introduces a non-parametric implicit pdf over $\SO$, with the distribution properties modeled by the neural network parameters. Implicit-pdf especially does good for modeling rotations of symmetric objects.

\paragraph{Non-probabilistic rotation regression}

The choice of rotation representation is one of the core issues concerning rotation regression. The commonly used representations include Euler angles \citep{kundu20183d,tulsiani2015viewpoints}, unit quaternion \citep{kendall2017geometric,kendall2015posenet,xiang2017posecnn} and axis-angle \citep{do2018deep,gao2018occlusion,ummenhofer2017demon}, \textit{etc}. However, Euler angles may suffer from gimbal lock, and unit quaternions doubly cover the group of $\SO$, which leads to two disconnected local minima. Moreover, \cite{zhou2019continuity} points out that all representations in the real Euclidean spaces of four or fewer dimensions are discontinuous and are not friendly for deep learning. To this end, the continuous 6D representation with Gram-Schmidt orthogonalization \citep{zhou2019continuity} and 9D representation with SVD orthogonalization \citep{levinson2020analysis} have been proposed, respectively. More recently, \cite{chen2022projective} investigates the gradient backpropagation in the backward pass and proposes a $\SO$ manifold-aware gradient layer.

\section{Revisit matrix Fisher distribution}

\subsection{Matrix Fisher Distribution}

Matrix Fisher distribution (or von Mises-Fisher matrix distribution) \citep{khatri1977mises} is one of the widely used distributions for probabilistic modeling of rotation matrices. 

\begin{definition} \textbf{\emph{Matrix Fisher distribution}}.
The random variable $\mathbf{R}\in \SO$ follows matrix Fisher distribution with parameter $\mathbf{A}$, if its probability density function is defined as
\begin{equation}
\small
    p(\mathbf{R}; \mathbf{A}) = \frac{1}{F(\mathbf{A})}
    \exp\left(
    \operatorname{tr}(\mathbf{A}^T \mathbf{R})
    \right)
\end{equation}
where $\mathbf{A}\in \mathbb{R}^{3\times 3}$ is an unconstrained matrix, and $F(\mathbf{A})\in \mathbb{R}$ is the normalization factor. Without further clarification, we denote $F$ as the normalization factor of the corresponding distribution in the remaining of this paper. We also denote matrix Fisher distribution as $\mathbf{R} \sim \mathcal{MF}(\mathbf{A})$.
\end{definition}
Suppose the singular value decomposition of matrix $\mathbf{A}$ is given by $\mathbf{A} = \mathbf{U}^\prime \mathbf{S}^\prime (\mathbf{V}^\prime)^T$, \textit{proper} SVD is defined as $\mathbf{A} = \mathbf{USV}^T$
where 
{\small
\begin{equation*}
    \small
    \begin{aligned}
        \mathbf{U} = \mathbf{U}^\prime\operatorname{diag}(1, 1, \det(\mathbf{U}^\prime)) \qquad
        \mathbf{V} = \mathbf{V}^\prime \operatorname{diag}(1, 1, \det(\mathbf{V}^\prime)) \\
        \mathbf{S} = \operatorname{diag}({s_1}, {s_2}, {s_3}) = 
        \operatorname{diag}(s_1^\prime, s_2^\prime, \det(\mathbf{U}^\prime\mathbf{V}^\prime)s_3^\prime)\\
    \end{aligned}
\end{equation*}
}The definition of $\mathbf{U}$ and $\mathbf{V}$ ensures that $\det(\mathbf{U})=\det(\mathbf{V})=1$ and $\mathbf{U}, \mathbf{V} \in \SO$.

\subsection{Relationship between Matrix Fisher Distribution in $\SO$ and Gaussian Distribution in $\mathbb{R}^3$}
\label{sec:fisher_gauss}
It is shown that matrix Fisher distribution is highly relevant with zero-mean Gaussian distribution near its mode
\ree{\citep{lee2018bayesian,lee2018bayesian_approximate}}.
Denote $\mathbf{R}_0$ as the mode of matrix Fisher distribution, and define $\mathbf{\widetilde{R}}=\mathbf{R}_0^T\mathbf{R}$, the relationship is shown as follows. Please refer to supplementary for the proof.
\begin{proposition}
\label{prop:fisher_gaussian}
Let $\boldsymbol{\Phi} = \log \mathbf{\widetilde{R}} \in \mathfrak{so}(3)$ and $\boldsymbol{\phi} = {\boldsymbol{\Phi}^\vee} \in \mathbb{R}^3$. For rotation matrix $\mathbf{R} \in \SO$ following \emph{matrix Fisher distribution}, when 
\ree{$\|\mathbf{R} - \mathbf{R}_0 \| \rightarrow 0$}
, $\boldsymbol{\phi}$ follows zero-mean \emph{multivariate Gaussian distribution}.
\end{proposition}

\section{Probabilistic Rotation Estimation with Rotation Laplace Distribution}

\subsection{Rotation Laplace Distribution}
We get inspiration from multivariate Laplace distribution \ree{\citep{eltoft2006multivariate,kozubowski2013multivariate}}, defined as follows.
\begin{definition} \textbf{\emph{Multivariate Laplace distribution.}}
If means $\boldsymbol{\mu}=\mathbf{0}$, the d-dimensional multivariate Laplace distribution with covariance matrix $\boldsymbol{\Sigma}$ 
is defined as
{\footnotesize
\begin{equation*}
    p(\mathbf{x};\boldsymbol{\Sigma}) = \frac{1}{F}
    \left(\mathbf{x}^T \boldsymbol{\Sigma}^{-1} \mathbf{x}\right)^{v / 2} K_{v}\left(\sqrt{2 \mathbf{x}^T \boldsymbol{\Sigma}^{-1} \mathbf{x}}\right)
\end{equation*}
}where $v = (2 - \ree{d}) / 2$ and $K_v$ is modified Bessel function of the second kind.
\end{definition}
We consider three dimensional Laplace distribution of $\mathbf{x}\in \mathbb{R}^3$
\ree{(i.e. $d=3$ and $v=-\frac{1}{2}$). Given the property $K_{-\frac{1}{2}}(\xi)\propto \xi^{-\frac{1}{2}} \exp (-\xi)$, three dimensional Laplace distribution is} defined as
{\footnotesize
\begin{equation*}
    p(\mathbf{x};\boldsymbol{\Sigma}) = \frac{1}{F}
    \frac{\exp \left( -\sqrt{2 \mathbf{x}^T \boldsymbol{\Sigma}^{-1} \mathbf{x}} \right)}{\sqrt{\mathbf{x}^T \boldsymbol{\Sigma}^{-1} \mathbf{x}}}
\end{equation*}}

In this section, we first give the definition of our proposed Rotation Laplace distribution and then shows its relationship with multivariate Laplace distribution.

\begin{definition} \textbf{\emph{Rotation Laplace distribution.}}
The random variable $\mathbf{R}\in \SO$ follows Rotation Laplace distribution with parameter $\mathbf{A}$, if its probability density function is defined as
\begin{equation}
\label{eq:rl}
\small
    p(\mathbf{R}; \mathbf{A}) = \frac{1}{F(\mathbf{A})}
    \frac{\exp\left(-\sqrt{\operatorname{tr}\left(\mathbf{S} - \mathbf{A}^T \mathbf{R}\right)}\right)}
    {\sqrt{\operatorname{tr}\left(\mathbf{S} -\mathbf{A}^T \mathbf{R}\right)}}
\end{equation}
where $\mathbf{A}\in \mathbb{R}^{3\times 3}$ is an unconstrained matrix, and $\mathbf{S}$ is the diagonal matrix composed of the proper singular values of matrix $\mathbf{A}$, i.e., $\mathbf{A=USV}^T$. We also denote Rotation Laplace distribution as $\mathbf{R} \sim \mathcal{RL}(\mathbf{A})$.
\end{definition}

Denote $\mathbf{R}_0$ as the mode of Rotation Laplace distribution and define $\mathbf{\widetilde{R}}=\mathbf{R}_0^T\mathbf{R}$, the relationship between Rotation Laplace distribution and multivariate Laplace distribution is shown as follows.
\begin{proposition}
Let $\boldsymbol{\Phi} = \log \mathbf{\widetilde{R}} \in \mathfrak{so}(3)$ and $\boldsymbol{\phi} = {\boldsymbol{\Phi}^\vee} \in \mathbb{R}^3$. For rotation matrix $\mathbf{R} \in \SO$ following \emph{Rotation Laplace distribution}, when
\ree{$\|\mathbf{R} - \mathbf{R}_0\|\rightarrow 0$}
, $\boldsymbol{\phi}$ follows zero-mean \emph{multivariate Laplace distribution}.
\end{proposition}

\begin{proof}

Apply proper SVD to matrix $\mathbf{A}$ as $\mathbf{A} = \mathbf{USV}^T$.
For $\mathbf{R}\sim \mathcal{RL}(\mathbf{A})$ , we have
\begin{equation}
\label{eq:prdr}
\scriptsize
    p(\mathbf{R})\mathrm{d}\mathbf{R} \propto 
    \frac{\exp\left(\sqrt{\tr{\mathbf{S}-{\mathbf{A}^T\mathbf{R}}}}\right)}{\sqrt{\tr{\mathbf{S}-{\mathbf{A}^T\mathbf{R}}}}} \mathrm{d}\mathbf{R}
    = \frac{\exp\left(\sqrt{\tr{\mathbf{S}-\mathbf{S}\mathbf{V}^T\mathbf{\widetilde{R}}\mathbf{V}}}\right)}{\sqrt{\tr{\mathbf{S}-\mathbf{S}\mathbf{V}^T\mathbf{\widetilde{R}}\mathbf{V}}}} \mathrm{d}\mathbf{R}
\end{equation}

With $\boldsymbol{\phi}  = (\log{\mathbf{\widetilde{R}}})^\vee \in \mathbb{R}^3$,
$\mathbf{\widetilde{R}}$ can be parameterized as 
{\scriptsize\begin{equation*}
    \mathbf{\widetilde{R}}(\boldsymbol{\phi} )= 
    \exp(\hat{\boldsymbol{\phi}}) = 
    \mathbf{I}+\frac{\sin\len {\boldsymbol{\phi}}  }{\len {\boldsymbol{\phi}}}\hat{\boldsymbol{\phi}}
    + \frac{1-\cos\len {\boldsymbol{\phi}}}{{\len {\boldsymbol{\phi}}}^2}{\hat{\boldsymbol{\phi}}}^2
\end{equation*}
}\ree{We follow the common practice \citep{mohlin2020probabilistic,lee2018bayesian} that the Haar measure $\mathrm{d}\mathbf{R}$ is scaled such that $\int_{SO(3)} \mathrm{d} \mathbf{R}=1$} and thus the Haar measure is given by
\begin{equation}
\label{eq:d_rwave}
\scriptsize
    \mathrm{d}\mathbf{\widetilde{R}} 
    = \frac{1- \cos \len {\boldsymbol{\phi}}}{{4\pi^2 \len {\boldsymbol{\phi}}}^2}\mathrm{d}\boldsymbol{\phi}
    = \left(\frac{1}{8\pi^2}+O(\len {\boldsymbol{\phi}})^2\right)\mathrm{d}\boldsymbol{\phi}.
\end{equation}
Also, $\mathbf{\widetilde{R}}$ expanded at $\boldsymbol{\phi}=\mathbf{0}$ is computed as 
{\small$\mathbf{\widetilde{R}} = \mathbf{I}+\hat{\boldsymbol{\phi}}+\frac{1}{2}\hat{\boldsymbol{\phi}}^2+O({\len {\boldsymbol{\phi}} }^3)$},
we have 
\begin{equation}
\label{eq:vtrv}
\scriptsize
\begin{aligned}
    \mathbf{V}^T\mathbf{\widetilde{R}}\mathbf{V} &= \mathbf{I} + \mathbf{V}^T\hat{\boldsymbol{\phi}}\mathbf{V} + \frac{1}{2}\mathbf{V}^T\hat{\boldsymbol{\phi}}^2\mathbf{V}+O(\len{\boldsymbol{\phi}}^3) 
    = \mathbf{I}+\widehat{\mathbf{V}^T\boldsymbol{\phi}}+\frac{1}{2}\widehat{\mathbf{V}^T\boldsymbol{\phi}}^2+O({\len {\boldsymbol{\phi}} }^3)\\
    &= \left[\begin{array}{ccc}
        1-\frac{1}{2}(\mu_2^2+\mu_3^2) & \frac{1}{2}{\mu_1}{\mu_2}-\mu_3 & \frac{1}{2}{\mu_1}{\mu_3}+\mu_2 \\
        \frac{1}{2}{\mu_1}{\mu_2}+\mu_3 & 1-\frac{1}{2}(\mu_3^2+\mu_1^2) & \frac{1}{2}{\mu_2}{\mu_3}-\mu_1 \\
        \frac{1}{2}{\mu_1}{\mu_3}-\mu_2 & \frac{1}{2}{\mu_2}{\mu_3}+\mu_1 & 1-\frac{1}{2}(\mu_1^2+\mu_2^2)
        \end{array}\right] + O({\len {\boldsymbol{\phi}}}^3),
\end{aligned}
\end{equation}
where $(\mu_1,\mu_2,\mu_3)^T = \mathbf{V}^T\boldsymbol{\phi}$, 
and
\begin{equation}
\label{eq:tr_laplace}
\scriptsize
\begin{aligned}
    \tr{\mathbf{S}-\mathbf{S}\mathbf{V}^T\mathbf{\widetilde{R}}\mathbf{V}} 
    =\sum_{(i,j,k)\in I}\frac{1}{2}(s_j+s_k)\mu_i^2+O({\len {\boldsymbol{\phi}}}^3) 
    =\frac{1}{2}\boldsymbol{\phi}^T\mathbf{V}
    \left[\begin{smallmatrix}
        s_2 + s_3 &  &  \\
        & s_1 + s_3 &  \\
        &  & s_1 + s_2
        \end{smallmatrix}\right]
    \mathbf{V}^T\boldsymbol{\phi}+O({\len {\boldsymbol{\phi}}}^3)
\end{aligned}
\end{equation}
Considering Eq. \ref{eq:prdr}, \ref{eq:d_rwave} and \ref{eq:tr_laplace}, we have
\begin{equation}
\label{eq:laplace}
\scriptsize
\begin{aligned}
    p(\mathbf{R})\mathrm{d}\mathbf{R} &\propto \frac{\exp\left(\sqrt{\tr{\mathbf{S}-{\mathbf{A}^T\mathbf{R}}}}\right)}{\sqrt{\tr{\mathbf{S}-{\mathbf{A}^T\mathbf{R}}}}} \mathrm{d}\mathbf{R}
    = \frac{1}{8\pi^2}\frac{\exp\left(-\sqrt{2\boldsymbol{\phi}^T\boldsymbol{\Sigma} ^{-1}\boldsymbol{\phi}}\right)}{\sqrt{2\boldsymbol{\phi}^T\boldsymbol{\Sigma} ^{-1}\boldsymbol{\phi}}}\left(1+ O({\len {\boldsymbol{\phi}}}^2)\right) \mathrm{d}\boldsymbol{\phi} \\
\end{aligned}
\end{equation}
When 
\ree{$\|\mathbf{R} - \mathbf{R}_0\| \rightarrow 0$}
, we have 
\ree{$\|\mathbf{\widetilde{R}} - \mathbf{I}\|  \rightarrow 0$ }
and $\boldsymbol{\phi} \rightarrow \mathbf{0}$, so Eq. \ref{eq:laplace} follows the multivariate Laplace distribution 
with the covariance matrix as $\boldsymbol{\Sigma}$, where $\boldsymbol{\Sigma} = 4 \mathbf{V}\operatorname{diag}(\frac{1}{s_2+s_3},\frac{1}{s_1+s_3},\frac{1}{s_1+s_2})\mathbf{V}^T$.
\end{proof}

Rotation Laplace distribution bears similar properties with matrix Fisher distribution. Its mode is computed as $\mathbf{UV}^T$. The columns of $\mathbf{U}$ and the proper singular values $\mathbf{S}$ describe the orientation and the strength of dispersions, respectively.

\subsection{Negative Log-likelihood Loss}
Given a collection of observations $\mathcal{X}=\{\boldsymbol{x}_i\}$ and the associated ground truth rotations $\mathcal{R}=\{\mathbf{R}_i\}$, we aim at training the network to best estimate the parameter $\mathbf{A}$ of Rotation Laplace distribution. This is achieved by maximizing a likelihood function so that, under our probabilistic model, the observed data is most probable, which is known as maximum likelihood estimation (MLE).
We use the negative log-likelihood of $\mathbf{R}_{\boldsymbol{x}}$ as the loss function:
\begin{equation*}
    \mathcal{L}(\boldsymbol{x}, \mathbf{R}_{\boldsymbol{x}}) = -\log p\left( \mathbf{R}_{\boldsymbol{x}};\mathbf{A}_{\boldsymbol{x}} \right)
\end{equation*}

\subsection{Discrete Approximation of the Normalization Factor}

Efficiently and accurately estimating the normalization factor for distributions over $\SO$ is non-trivial. 
Inspired by \cite{murphy2021implicit}, we approximate the normalization factor of  Rotation Laplace distribution through equivolumetric discretization over $\SO$ manifold. We employ the discretization method introduced in \cite{yershova2010generating}, which starts with the equal area grids on the 2-sphere \citep{gorski2005healpix} and covers $\SO$ by threading a great circle through each point on the surface of a 2-sphere with Hopf fibration.
\ree{Concretely, we discretize $\SO$ space into a finite set of equivolumetric grids $\mathcal{G}=\left\{\mathbf{R}|\mathbf{R}\in \SO\right\}$},
the normalization factor of Laplace Rotation distribution is computed as
{\scriptsize
\begin{equation*}
    F(\mathbf{A}) = \int_{\SO} \frac{\exp\left(-\sqrt{\operatorname{tr}\left(\mathbf{S} - \mathbf{A}^T \mathbf{R}\right)}\right)}
    {\sqrt{\operatorname{tr}\left(\mathbf{S} -\mathbf{A}^T \mathbf{R}\right)}} \mathrm{d}\mathbf{R} 
    \approx 
    \ree{
    \sum_{\mathbf{R}_i \in \mathcal{G}} \frac{\exp\left(-\sqrt{\operatorname{tr}\left(\mathbf{S} - \mathbf{A}^T \mathbf{R}_i\right)}\right)}
    {\sqrt{\operatorname{tr}\left(\mathbf{S} -\mathbf{A}^T \mathbf{R}_i\right)}} \Delta \mathbf{R}_i 
    }
\end{equation*}
}\ree{where $\Delta \mathbf{R}_i=
\frac{\int_{SO(3)} \mathrm{d} \mathbf{R}}{|\mathcal{G}|}=\frac{1}{|\mathcal{G}|}$}.
\ree{In experiments, we discretize $\SO$ space into about 37k points. Please refer to supplementary for analysis of the effect of different numbers of samples.}

\subsection{Quaternion Laplace Distribution}

In this section, we introduce our extension of Laplace-inspired distribution for quaternions, namely, Quaternion Laplace distribution.

\begin{definition} \textbf{\emph{Quaternion Laplace distribution.}}
The random variable $\mathbf{q}\in \mathcal{S}^3$ follows Quaternion Laplace distribution with parameter $\mathbf{M}$ and $\mathbf{Z}$, if its probability density function is defined as
\begin{equation}
    p(\mathbf{q}; \mathbf{M}, \mathbf{Z}) = \frac{1}{F(\mathbf{Z})}\frac{\exp
    \left(-\sqrt{-\mathbf{q}^T \mathbf{M} \mathbf{Z} \mathbf{M}^T \mathbf{q} }\right)}
    {\sqrt{- \mathbf{q}^T \mathbf{M} \mathbf{Z} \mathbf{M}^T \mathbf{q}}}
\end{equation}
where $\mathbf{M}\in \mathbf{O}(4)$ is a $4\times4$ orthogonal matrix, and $\mathbf{Z}=\operatorname{diag}(0, z_1, z_2, z_3)$ is a $4\times4$ diagonal matrix with $0\ge z_1\ge z_2\ge z_3$. We also denote Quaternion Laplace distribution as $\mathbf{q} \sim \mathcal{QL}(\mathbf{M}, \mathbf{Z}).$
\end{definition}

\begin{proposition}
\label{prop:quat}
Denote $\mathbf{q}_0$ as the mode of Quaternion Laplace distribution. Let $\pi$ be the tangent space of $\mathbb{S}^3$  at $\mathbf{q}_0$, and $\pi(\mathbf{x}) \in \mathbb{R}^4$ be the projection of $\mathbf{x} \in \mathbb{R}^4$ on $\pi$.
For quaternion $\mathbf{q} \in \mathbb{S}^3$ following \emph{Bingham distribution} / \emph{Quaternion Laplace distribution}, when $\mathbf{q}\rightarrow\mathbf{q}_0$, $\pi(\mathbf{q})$ follows zero-mean \emph{multivariate Gaussian distribution} / zero-mean \emph{multivariate Laplace distribution}.
\end{proposition}

Both Bingham distribution and Quaternion Laplace distribution exhibit antipodal symmetry on $\mathcal{S}^3$, i.e., $p(\mathbf{q}) = p(-\mathbf{q})$, which captures the nature that the quaternions $\mathbf{q}$ and $-\mathbf{q}$ represent the same rotation on $\SO$.

\begin{proposition}
\label{prop:eq}
Denote $\gamma$ as the standard transformation from unit quaternions to corresponding rotation matrices. For rotation matrix $\mathbf{R}\in \SO$ following \emph{Rotation Laplace distribution}, $\mathbf{q}=\gamma^{-1}(\mathbf{R})\in \mathbb{S}^3$ follows \emph{Quaternion Laplace distribution}.
\end{proposition}

Prop. \ref{prop:eq} shows that our proposed Rotation Laplace distribution is equivalent to Quaternion Laplace distribution, similar to the equivalence of matrix Fisher distribution and Bingham distribution \citep{prentice1986orientation}, demonstrating the consistency of our derivations.
Please see supplementary for the proofs to the above propositions.

The normalization factor of Quaternion Laplace distribution is also approximated by dense discretization, as follows:
{\scriptsize\begin{equation*}
    F(\mathbf{Z}) = \oint_{\mathcal{S}^3} \frac{\exp
    \left(-\sqrt{-\mathbf{q}^T \mathbf{M} \mathbf{Z} \mathbf{M}^T \mathbf{q} }\right)}
    {\sqrt{- \mathbf{q}^T \mathbf{M} \mathbf{Z} \mathbf{M}^T \mathbf{q}}}
    \mathrm{d}\mathbf{q} 
    \approx 
    \ree{
    \sum_{\mathbf{q}_i \in \mathcal{G}_\mathbf{q}} \frac{\exp
    \left(-\sqrt{-\mathbf{q}_i^T \mathbf{M} \mathbf{Z} \mathbf{M}^T \mathbf{q}_i }\right)}
    {\sqrt{- \mathbf{q}_i^T \mathbf{M} \mathbf{Z} \mathbf{M}^T \mathbf{q}_i}}
    \Delta \mathbf{q}_i 
    }
\end{equation*}
}\ree{where $\mathcal{G}_\mathbf{q}=\left\{ \mathbf{q} | \mathbf{q}\in \mathcal{S}^3 \right\}$ 
denotes the set of equivolumetric grids and $\Delta \mathbf{q}_i=\frac{\oint_{\mathcal{S}^3} \mathrm{d}\mathbf{q}}{|\mathcal{G}_\mathbf{q}|} = \frac{2\pi^2}{|\mathcal{G}_\mathbf{q}|}$.}

\section{Experiment}

\ree{
Following the previous state-of-the-arts \citep{murphy2021implicit,mohlin2020probabilistic}, we evaluate our method on the task of object rotation estimation from single RGB images, where object rotation is the relative rotation between the input object and the object in the canonical pose.
}
Concerning this task,
we find two kinds of independent research tracks with slightly different evaluation settings. One line of research focuses on probabilistic rotation regression with different parametric or non-parametric distributions on $\SO$ \citep{prokudin2018deep,gilitschenski2019deep,deng2022deep,mohlin2020probabilistic,murphy2021implicit}, and the other non-probabilistic track proposes multiple rotation representations \citep{zhou2019continuity,levinson2020analysis,peretroukhin2020smooth} or improves the gradient of backpropagation \citep{chen2022projective}.
To fully demonstrate the capacity of our Rotation Laplace distribution, we leave the baselines in their original optimal states and adapt our method to follow the common experimental settings in each track, respectively.

\subsection{Datasets \& Evaluation Metrics}

\paragraph{Datasets}
\textbf{ModelNet10-SO3} \citep{liao2019spherical} is a commonly used synthetic dataset for single image rotation estimation containing 10 object classes. It is synthesized by rendering the CAD models of ModelNet-10 dataset \citep{wu20153d} that are rotated by uniformly sampled rotations in $\SO$.
\textbf{Pascal3D+} \citep{xiang2014beyond} is a popular benchmark on real-world images for pose estimation. It covers 12 common daily object categories. The images in Pascal3D+ dataset are sourced from Pascal VOC and ImageNet datasets, and are split into ImageNet\_train, ImageNet\_val, PascalVOC\_train, and PascalVOC\_val sets.

\paragraph{Evaluation metrics}
We evaluate our experiments with the geodesic distance of the network prediction and the ground truth. This metric returns the angular error and we measure it in degrees. In addition, we report the prediction accuracy within the given error threshold.

\subsection{Comparisons with Probabilistic Methods}
\label{exp:track1}

\subsubsection{Evaluation Setup}
\paragraph{Settings} In this section, we follow the experiment settings of the latest work \citep{murphy2021implicit} and quote its reported numbers for baselines. Specifically, we train one single model for all categories of each dataset. For Pascal3D+ dataset, we follow \cite{murphy2021implicit} to use (the more challenging) PascalVOC\_val as test set. 
Note that \cite{murphy2021implicit} only measure the coarse-scale accuracy (e.g., Acc@30$^\circ$) which may not adequately satisfy the downstream tasks \citep{wang2019normalized,fang2020towards}. To facilitate finer-scale comparisons (e.g., Acc@5$^\circ$), we further re-run several recent baselines and report the reproduced results in parentheses ($\cdot$).

\paragraph{Baselines}
We compare our method to recent works which utilize probabilistic distributions on $\SO$ for the purpose of pose estimation. 
In concrete, the baselines are with mixture of \textit{von Mises} distributions \textbf{\cite{prokudin2018deep}}, \textit{Bingham} distribution \textbf{\cite{gilitschenski2019deep,deng2022deep}}, \textit{matrix Fisher} distribution \textbf{\cite{mohlin2020probabilistic}} and Implicit-PDF \textbf{\cite{murphy2021implicit}}.
We also compare to the spherical regression work of \textbf{\cite{liao2019spherical}} 
as \cite{murphy2021implicit} does.

\subsubsection{Results}
Table \ref{tab:modelnet} shows the quantitative comparisons of our method and baselines on ModelNet10-SO3 dataset. From the multiple evaluation metrics, we can see that maximum likelihood estimation with the assumption of Rotation Laplace distribution significantly outperforms the other distributions for rotation, including matrix Fisher distribution \citep{mohlin2020probabilistic}, Bingham distribution \citep{do2018deep} and von-Mises distribution \citep{prokudin2018deep}. Our method also gets superior performance than the non-parametric implicit-PDF \citep{murphy2021implicit}. Especially, our method improves the fine-scale Acc@3$^\circ$ and Acc@5$^\circ$ accuracy by a large margin, showing its capacity to precisely model the target distribution.

\begin{table}[t]
  \centering
  \fontsize{7.5}{9}\selectfont
  \vspace{-5mm}
  \caption{\small Numerical comparisons with probabilistic baselines on ModelNet10-SO3 dataset averaged on all categories. Numbers in parentheses ($\cdot$) are our reproduced results. Please refer to supplementary for comparisons with each category.}
    \begin{tabular}{lcccccc}
    \toprule
          & Acc@3$^\circ$$\uparrow$ & Acc@5$^\circ$$\uparrow$ & Acc@10$^\circ$$\uparrow$ & Acc@15$^\circ$$\uparrow$ & Acc@30$^\circ$$\uparrow$ & Med.($^\circ$)$\downarrow$  \\
    \midrule
    
    \cite{liao2019spherical}        &      -&   -&      -&  0.496&  0.658&  28.7  \\
    \cite{prokudin2018deep}         &      -&   -&      -&  0.456&  0.528&   49.3 \\
    \cite{deng2022deep}             &  (0.138)&   (0.301)&  (0.502)&  0.562 (0.584)&  0.694 (0.673)&   32.6 (31.6)\\
    \cite{mohlin2020probabilistic}  &  (0.164)& (0.389) &  (0.615) &  0.693 (0.684)&  0.757 (0.751)&   17.1 (17.9) \\
    \cite{murphy2021implicit}       &  (0.294)&  (0.534)  & (0.680) &  0.719 (0.714)&  0.735 (0.730)&   21.5 (20.3)\\
    \midrule
    Rotation Laplace                &  \textbf{0.447}  & \textbf{0.611}  &  \textbf{0.715}  &  \textbf{0.742}  &  \textbf{0.772} & \textbf{12.7}   \\
    \bottomrule
    \end{tabular}%
  \label{tab:modelnet}%
\end{table}%

\begin{table}[t]
  \centering
  \fontsize{7.5}{9}\selectfont
  \caption{\small Numerical comparisons with probabilistic baselines on Pascal3D+ dataset averaged on all categories. Numbers in parentheses ($\cdot$) are our reproduced results. Please refer to supplementary for comparisons with each category.}
  \vspace{2mm}
    \begin{tabular}{lcccccc}
    \toprule
          & Acc@3$^\circ$$\uparrow$ & Acc@5$^\circ$$\uparrow$ & Acc@10$^\circ$$\uparrow$ & Acc@15$^\circ$$\uparrow$ & Acc@30$^\circ$$\uparrow$ & Med.($^\circ$)$\downarrow$   \\
    \midrule
    
    \cite{tulsiani2015viewpoints}   &      -&      -&       -&      -&  0.808&  13.6 \\
    \cite{mahendran2018mixed}       &      -&      -&      -&       -&  0.859&  10.1 \\
    \cite{liao2019spherical}        &      -&      -&      -&       -&  0.819&  13.0 \\
    \cite{prokudin2018deep}         &      -&      -&      -&       -&  0.838&  12.2 \\
    \cite{mohlin2020probabilistic}  &  (0.089)&  (0.215)&    (0.484)& (0.650)&  0.825 (0.827)&  11.5 (11.9) \\
    \cite{murphy2021implicit}       &  (0.102)&   (0.242)&  (0.524)& (0.672)&  0.837 (0.838)&  10.3 (10.2) \\
    \midrule
    Rotation Laplace                 &  \textbf{0.134} & \textbf{0.292} & \textbf{0.574}  &  \textbf{0.714}  & \textbf{0.874} &  \textbf{9.3}  \\
    \bottomrule
    \end{tabular}%
  \label{tab:pascal}%
\end{table}%

The experiments on Pascal3D+ dataset are shown in Table \ref{tab:pascal}, where our Rotation Laplace distribution outperforms all the baselines. While our method gets reasonably good performance on the median error and coarser-scale accuracy, we do not find a similar impressive improvement on fine-scale metrics as in ModelNet10-SO3 dataset. 
We suspect it is because the imperfect human annotations of real-world images may lead to comparatively noisy ground truths, increasing the difficulty for networks to get rather close predictions with GT labels.
Nevertheless, our method still manages to obtain superior performance, which illustrates the robustness of our Rotation Laplace distribution.

\subsection{Comparisons with Non-probabilistic Methods}
\label{exp:track2}

\subsubsection{Evaluation Setup}
\paragraph{Settings}
For comparisons with non-probabilistic methods, we follow the latest work of \cite{chen2022projective} to learn a network for each category. For Pascal3D+ dataset, we follow \cite{chen2022projective} to use ImageNet\_val as our test set. We use the same evaluation metrics as in \cite{chen2022projective} and quote its reported numbers for baselines.

\paragraph{Baselines}
We compare to multiple baselines that leverage different rotation representations to directly regress the prediction given input images, including 
\textbf{6D} \citep{zhou2019continuity}, \textbf{9D} / \textbf{9D-Inf} \citep{levinson2020analysis} and \textbf{10D} \citep{peretroukhin2020smooth}. We also include regularized projective manifold gradient (\textbf{RPMG}) series of methods \citep{chen2022projective}.

\subsubsection{Results}
We report the numerical results of our method and on-probabilistic baselines on ModelNet10-SO3 dataset in Table \ref{tab:rpmg_modelnet}. Our method obtains a clear superior performance to the best competitor under all the metrics among all the categories. Note that we train a model for each category (so do all the baselines), thus our performance in Table \ref{tab:rpmg_modelnet} is better than Table \ref{tab:modelnet} where one model is trained for the whole dataset.
The results on Pascal3D+ dataset are shown in Table \ref{tab:rmpg_pascal} where our method with Rotation Laplace distribution achieves state-of-the-art performance.

\setlength{\tabcolsep}{4pt}
\begin{table*}[t]
    \caption{\small Numerical comparisons with non-probabilistic baselines on ModelNet10-SO3 dataset. One model is trained for each category.}
    \vspace{2mm}
    \centering
    \scriptsize
    \begin{tabular}{l|ccc|ccc|ccc|ccc}
    \toprule
    \multirow{2}{*}{Methods}&& Chair &&&Sofa&&& Toilet&&&Bed&\\
    \cmidrule{2-13}  
        & Mean$\downarrow$ & Med.$\downarrow$ & Acc@5$\uparrow$ 
        & Mean$\downarrow$ & Med.$\downarrow$ & Acc@5$\uparrow$ 
        & Mean$\downarrow$ & Med.$\downarrow$ & Acc@5$\uparrow$ 
        & Mean$\downarrow$ & Med.$\downarrow$ & Acc@5$\uparrow$ 
        \\
        \midrule   
        
        6D & 19.6 & 9.1  & 0.19 & 17.5 & 7.3 & 0.27  &    10.9&6.2&0.37&32.3&11.7&0.11     \\
        9D & 17.5 & 8.3 & 0.23 & 19.8 & 7.6 & 0.25  &    11.8&6.5&0.34&30.4&11.1&0.13     \\
        9D-Inf & 12.1 & 5.1 & 0.49& 12.5 & 3.5 & 0.70 & 7.6&3.7&0.67&22.5&4.5&0.56     \\
        10D & 18.4 & 9.0 & 0.20 & 20.9 & 8.7 & 0.20 &  11.5& 5.9&0.39& 29.9&11.5&0.11      \\
        \midrule
        RPMG-6D & 12.9 & 4.7 & 0.53 & 11.5 & 2.8 & 0.77 &      7.8&3.4&0.71&20.3&3.6&0.67     \\
        RPMG-9D & {11.9} & {4.4} & {0.58} & {10.5} & {2.4} & {0.82}  &     7.5&3.2&0.75&20.0&{2.9}&{0.76}    \\
         RPMG-10D & 12.8& 4.5& 0.55 & 11.2&{2.4}&{0.82}& {7.2}&{3.0}&{0.76}&{19.2}&{2.9}&0.75\\
        \midrule
        Rot. Laplace & \textbf{9.7}  & \textbf{3.5}  &  \textbf{0.68} & \textbf{8.8} & \textbf{2.1} & \textbf{0.84}  & \textbf{5.3} & \textbf{2.6} & \textbf{0.83} & \textbf{15.5} & \textbf{2.3} & \textbf{0.82}  \\
        \bottomrule
    \end{tabular}
  \label{tab:rpmg_modelnet}
\end{table*}
\setlength{\tabcolsep}{7.5pt}
\begin{table}[t]
  \centering
  \scriptsize
      \caption{\small Numerical comparisons with non-probabilistic baselines on Pascal3D+ dataset. One model is trained for each category.}
      \vspace{2mm}
    \begin{tabular}{l|cccc|cccc}
    \toprule
    \multicolumn{1}{c|}{\multirow{2}[4]{*}{Methods}} & \multicolumn{4}{c|}{Bicycle}  & \multicolumn{4}{c}{Sofa} \\
\cmidrule{2-9}          & Acc@10$\uparrow$ & Acc@15$\uparrow$ & Acc@20$\uparrow$ & Med.$\downarrow$  & Acc@10$\uparrow$ & Acc@15$\uparrow$ & Acc@20$\uparrow$ & Med.$\downarrow$ \\
    \midrule
    6D                  &      0.218& 0.390 & 0.553 & 18.1  &  0.508& 0.767 & 0.890 & 9.9                      \\
    9D                  &      0.206 & 0.376 & 0.569&18.0   &  0.524 & 0.796 & 0.903&9.2                       \\
    9D-Inf              &      0.380 & 0.533 & 0.699&13.4   &  0.709 & {0.880} & 0.935&{6.7}     \\
    10D                 &      0.239&0.423&0.567&17.9       &  0.502&0.770&0.896&9.8                           \\
    \midrule
    RPMG-6D             &      0.354 & 0.572 &0.706&13.5                       &    0.696 & 0.861 &0.922&{6.7}                           \\
    RPMG-9D             &      0.368&0.574 & {0.718}&{12.5}      &    {0.725} &{0.880} &{0.958}&{6.7} \\
    RPMG-10D            &      {0.400} & {0.577} & 0.713 & 12.9  &    0.693 & 0.871 & 0.939 & 7.0                                 \\
    \midrule
    Rot. Laplace        &  \textbf{0.435}   &   \textbf{0.641}&  \textbf{0.744}&   \textbf{11.2}&   \textbf{0.735}&    \textbf{0.900}&     \textbf{0.964}&  \textbf{6.3} \\
    
    \bottomrule
    \end{tabular}%
  \label{tab:rmpg_pascal}%
\end{table}%

\ree{
\subsection{Qualitative Results}

We visualize the predicted distributions in Figure \ref{fig:visual} with the visualization method in \cite{mohlin2020probabilistic}.
As shown in the figure, the predicted distributions can exhibit high uncertainty when the object has rotational symmetry, leading to near 180$^\circ$ errors (a-c), or the input image is with low resolution (d). Subfigure (e-f) show cases with high certainty and reasonably low errors.
Please refer to the supplementary for more visual results.
\begin{figure}[t]
    \centering
    \begin{tabular}{cccccc}
    \includegraphics[width=0.12\linewidth]{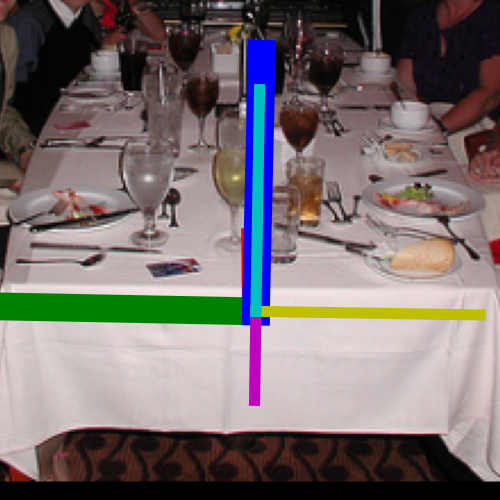}
    &\includegraphics[width=0.12\linewidth]{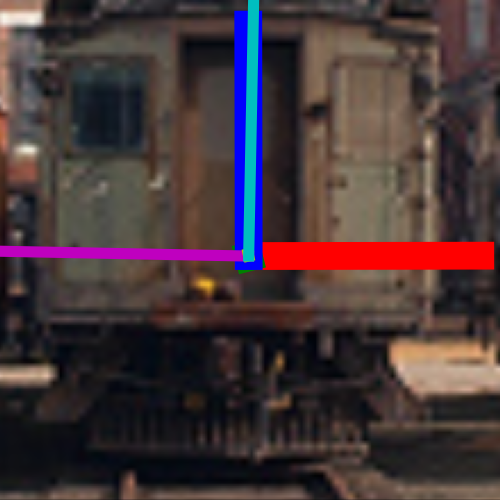}
    &\includegraphics[width=0.12\linewidth]{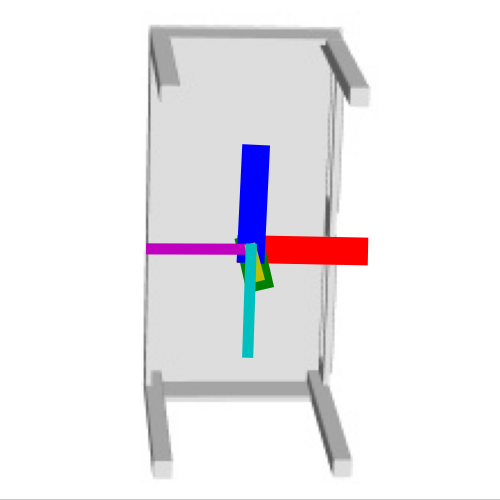}
    &\includegraphics[width=0.12\linewidth]{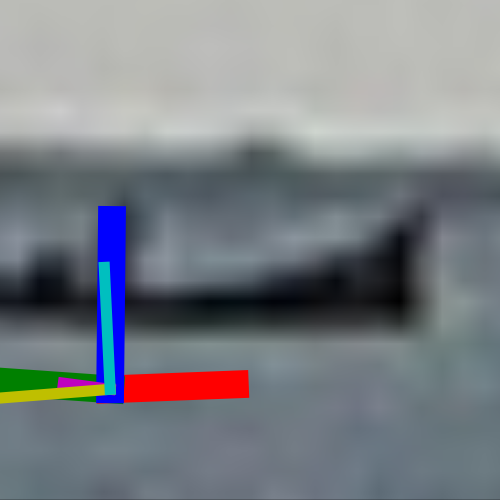}
    &\includegraphics[width=0.12\linewidth]{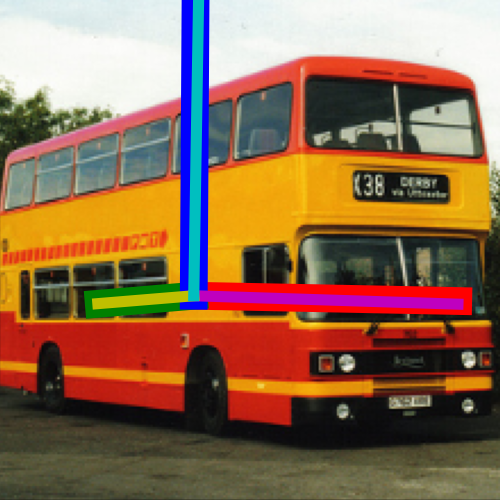}
    &\includegraphics[width=0.12\linewidth]{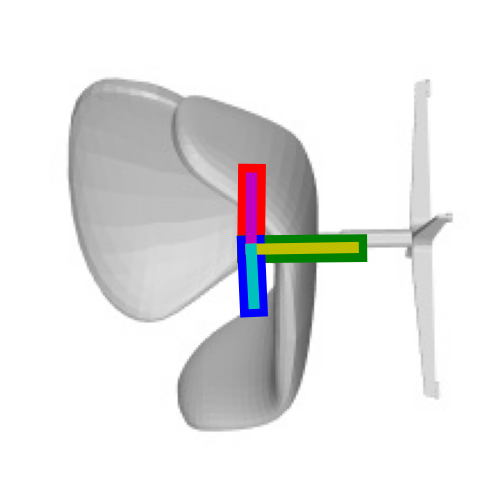}
    \\
    \includegraphics[clip,trim=4.5cm 4cm 4.5cm 1.5cm, width=0.12\linewidth]{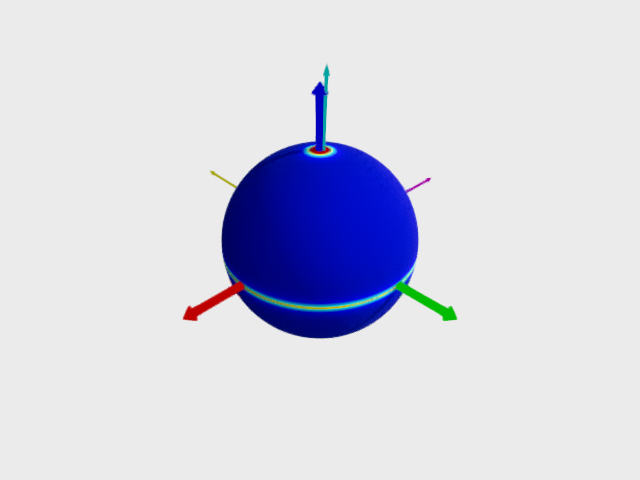}
    &\includegraphics[clip,trim=4.5cm 4cm 4.5cm 1.5cm, width=0.12\linewidth]{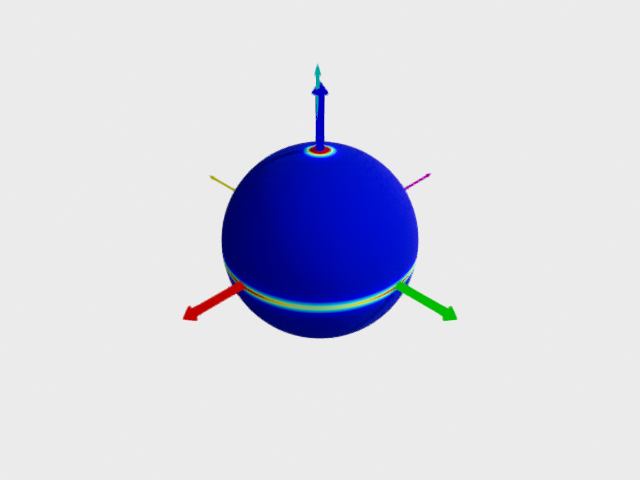}
    &\includegraphics[clip,trim=4.5cm 4cm 4.5cm 1.5cm, width=0.12\linewidth]{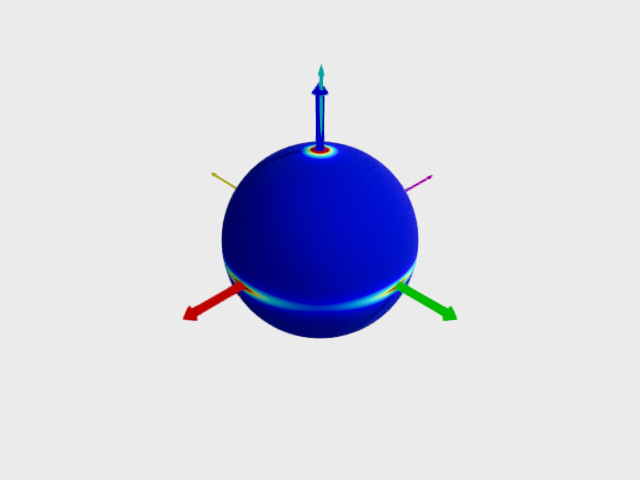}
    &\includegraphics[clip,trim=4.5cm 4cm 4.5cm 1.5cm, width=0.12\linewidth]{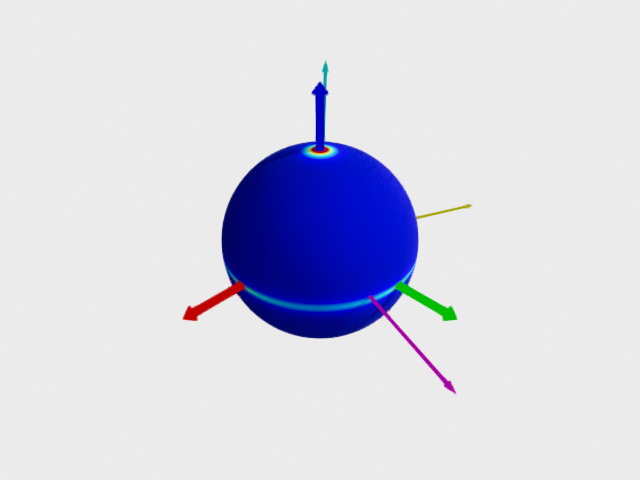}
    &\includegraphics[clip,trim=4.5cm 4cm 4.5cm 1.5cm, width=0.12\linewidth]{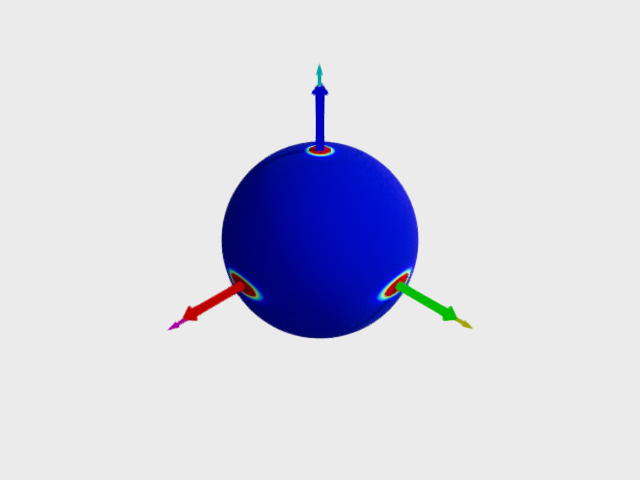}
    &\includegraphics[clip,trim=4.5cm 4cm 4.5cm 1.5cm, width=0.12\linewidth]{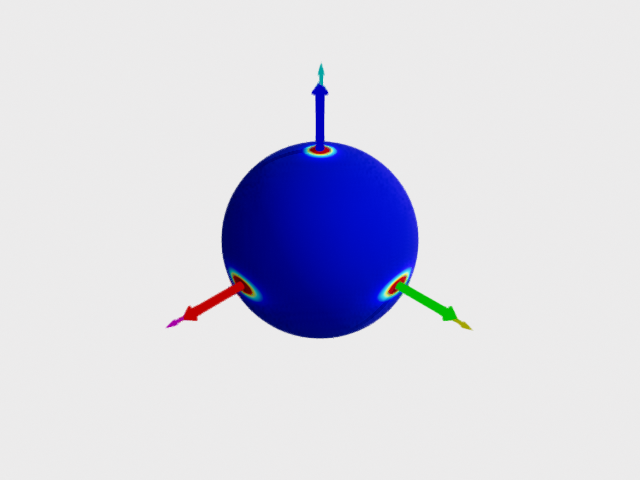}
    \\
    \small{(a)} &\small{(b)} &\small{(c)} &\small{(d)} &\small{(e)} & \small{(f)} 
    \end{tabular}
    \vspace{-3mm}
    \caption{\small \ree{\textbf{Visualizations of the predicted distributions.} The top row displays example images with the projected axes of predictions (thick lines) and ground truths (thin lines) of the object. The bottom row shows the visualization of the corresponding predicted distributions of the image.
    For clarity we have aligned the predicted poses with the standard axes.
    }}
    \vspace{-2mm}
	\label{fig:visual}
\end{figure}

}

\subsection{Implementation Details}

For fair comparisons, we follow the implementation designs of \cite{mohlin2020probabilistic} and merely change the distribution from matrix Fisher distribution to our Rotation Laplace distribution.
For numerical stability, we clip $\operatorname{tr}(\mathbf{S} - \mathbf{A}^T\mathbf{R})$ by $\max(1\mathrm{e}-8, \operatorname{tr}(\mathbf{S} - \mathbf{A}^T\mathbf{R}))$ for Eq.\ref{eq:rl}.
\ree{Please refer to supplementary for more details.}

\subsection{\ree{Comparisons of Rotation Laplace Distribution and Quaternion Laplace Distribution}}

\setlength{\tabcolsep}{3pt}
\begin{table*}[t]
    \caption{\small Numerical comparisons with our proposed Quaternion \& Rotation Laplace distribution and baselines on ModelNet10-SO3 dataset. One model is trained for each category. Quaternion Laplace distribution clearly outperforms Bingham distribution \citep{deng2022deep}.}
    \vspace{2mm}
    \centering
    \scriptsize
    \begin{tabular}{l|ccc|ccc|ccc|ccc}
    \toprule
    \multirow{2}{*}{}&& Chair &&&Sofa&&& Toilet&&&Bed&\\
    \cmidrule{2-13}  
        & Mean$\downarrow$ & Med.$\downarrow$ & Acc@5$\uparrow$ 
        & Mean$\downarrow$ & Med.$\downarrow$ & Acc@5$\uparrow$ 
        & Mean$\downarrow$ & Med.$\downarrow$ & Acc@5$\uparrow$ 
        & Mean$\downarrow$ & Med.$\downarrow$ & Acc@5$\uparrow$ 
        \\
        \midrule   
        \cite{deng2022deep} & 16.5 & 7.2 & 0.31 & 16.5 & 4.9 & 0.52 & 9.6 & 4.2 & 0.59 & 22.0 & 5.1 & 0.49\\
        \cite{mohlin2020probabilistic} & 10.8 & 4.6 & 0.55 & 11.1 & 3.5 & 0.70 & 6.4 & 3.5 & 0.70 & 16.0 & 3.8 & 0.66 \\
        \midrule
        Quat. Laplace & 12.6 & 5.2 & 0.49 & 13.1 & 3.7 & 0.67 & 5.9 & 3.4 & 0.69 & 17.7 & 3.4 & 0.69\\
        Rot. Laplace & \textbf{9.7}  & \textbf{3.5}  &  \textbf{0.68} & \textbf{8.8} & \textbf{2.1} & \textbf{0.84}  & \textbf{5.3} & \textbf{2.6} & \textbf{0.83} & \textbf{15.5} & \textbf{2.3} & \textbf{0.82}  \\
        \bottomrule
    \end{tabular}
  \label{tab:ablation}
\end{table*}

For the completeness of experiments, we also compare our proposed Quaternion Laplace distribution and Bingham distribution and report the performance in Table \ref{tab:ablation}. 
As shown in the table, Quaternion Laplace distribution consistently achieves superior performance than its competitor, which validates the effectiveness of our Laplace-inspired derivations.
However, its rotation error is in general larger than Rotation Laplace distribution, since its rotation representation, quaternion, is not a continuous representation, as pointed in \cite{zhou2019continuity}, thus leading to inferior performance.

\section{Conclusion}
In this paper, we draw inspiration from multivariant Laplace distribution and derive two novel distributions for probabilistic rotation regression, namely, Rotation Laplace distribution for rotation matrices on $\SO$ and Quaternion Laplace distribution for quaternions on $\mathcal{S}^3$. 
Extensive comparisons with both probabilistic and non-probabilistic baselines on ModelNet10-SO3 and Pascal3D+ datasets demonstrate the effectiveness and advantages of our proposed distributions.

\section*{Acknowledgement}
We thank Haoran Liu from Peking University for the help in experiments.
This work is supported in part by National Key R\&D Program of China 2022ZD0160801.

\bibliography{main}
\bibliographystyle{iclr2023_conference}

\appendix

\section{Notations and Definitions}
\label{sec:notation}

\subsection{Notations for Lie Algebra and Exponential \& Logarithm Map}

\ree{This paper follows the common notations for Lie algebra and exponential \& logarithm map \citep{lee2018bayesian,teed2021tangent,sola2018micro}}.

The three-dimensional special orthogonal group $\SO$ is defined as 
{\footnotesize \begin{equation*}
    \SO = \{ \mathbf{R} \in \mathbb{R}^{3\times3} | \mathbf{RR}^T = \mathbf{I}, \det{(\mathbf{R})} = 1 \}.
\end{equation*}
}The Lie algebra of $\SO$, denoted by $\mathfrak{so}(3)$, is the tangent space of $\SO$ at $\mathbf{I}$, given by
{\footnotesize\begin{equation*}
    \mathfrak{so}(3) = \{ \boldsymbol{\Phi} \in \mathbb{R}^{3\times 3}| \boldsymbol{\Phi} = -\boldsymbol{\Phi}^T\}.
\end{equation*}
}$\mathfrak{so}(3)$ is identified with $(\mathbb{R}^3, \times)$ by the \textit{hat} $\wedge$ map and the \textit{vee} $\vee$ map defined as
{\small\begin{equation*}
    \mathfrak{s o}(3) \ni\left[\begin{array}{ccc}
0 & -\phi_z & \phi_y \\
\phi_z & 0 & -\phi_x \\
-\phi_y & \phi_x & 0
\end{array}\right] 
\stackrel{\text { vee } \vee} {\underset{\text { hat }\wedge}{\rightleftarrows}}
\left[\begin{array}{l}\phi_x \\\phi_y \\\phi_z\end{array}\right] \in \mathbb{R}^3
\end{equation*}}

The exponential map, taking skew symmetric matrices to rotation matrices is given by
{\footnotesize\begin{equation*}
    \exp(\hat{\boldsymbol{\phi}})=\sum_{k=0}^{\infty}{\frac{\hat{\boldsymbol{\phi}}^k}{k!}}=\mathbf{I}+\frac{\sin{\theta}}{\theta}{\hat{\boldsymbol{\phi}}}+\frac{1-\cos{\theta}}{\theta^2}{\hat{\boldsymbol{\phi}}^2},
\end{equation*}
}where $\theta=\left\lVert {\boldsymbol{\phi}} \right\rVert$.
The exponential map can be inverted by the logarithm map, going from $\SO$ to $\mathfrak{so}(3)$ as
{\footnotesize\begin{equation*}
    \log(\mathbf{R})=\frac{\theta}{2\sin{\theta}}(\mathbf{R}-\mathbf{R}^T),
\end{equation*}
}where $\theta=\arccos{\frac{\operatorname{tr}(\mathbf{R})-1}{2}}$.

\ree{
\subsection{Haar Measure}
\label{sec:haar}
To evaluate the normalization factors and therefore the probability density functions, the measure $\mathrm{d}\mathbf{R}$ on $\SO$ needs to be defined. For the Lie group $\SO$, the commonly used bi-invariant measure is referred to as Haar measure \citep{haar1933massbegriff,james1999history}. Haar measure is unique up to scalar multiples \citep{chirikjian2000engineering} and we follow the common practice \citep{mohlin2020probabilistic,lee2018bayesian} that the Haar measure $\mathrm{d}\mathbf{R}$ is scaled such that $\int_{\SO} \mathrm{d} \mathbf{R}=1$.
}

\ree{\section{More Analysis on Gradient w.r.t. Outliers}

\begin{figure}[t]
    \centering
    \includegraphics[width=0.48\linewidth]{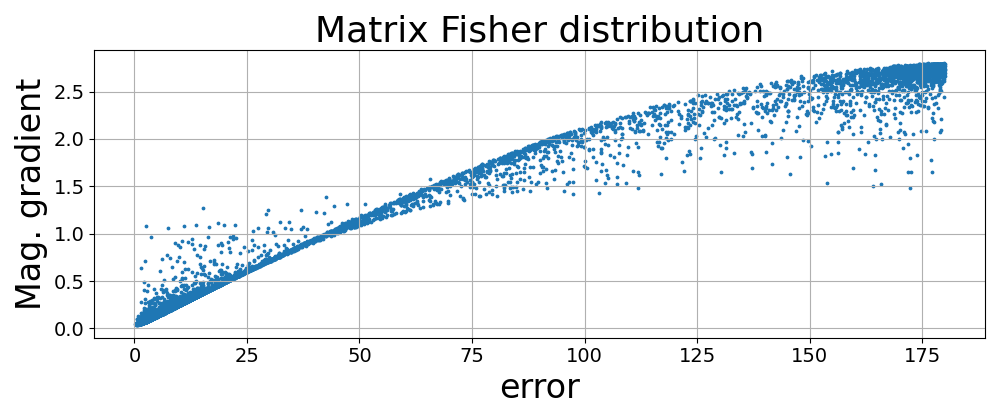}
    \includegraphics[width=0.48\linewidth]{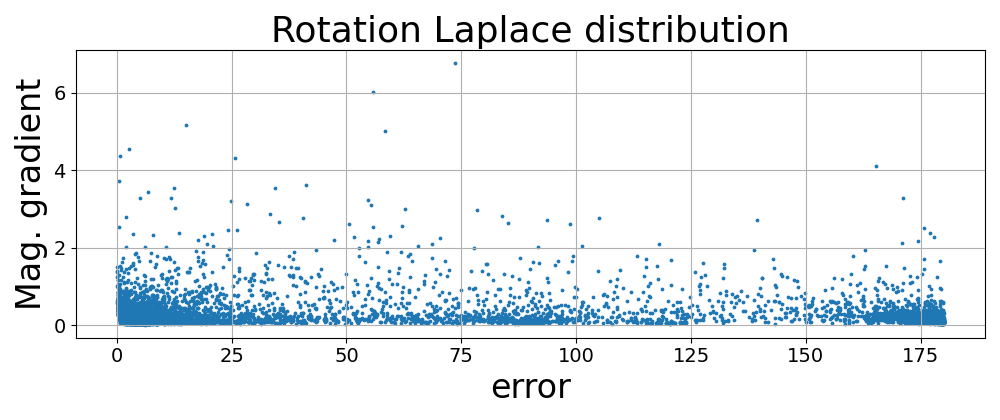}
    \vspace{-4mm}
    \caption{\ree{Visualization of the gradient magnituide $\|\partial\mathcal{L}/\partial\text{(distribution param.)}\|$ w.r.t. the prediction errors  on ModelNet10-SO3 dataset after convergence.}}
	\label{fig:grad_scatter}
\end{figure}

In the task of rotation regression, predictions with really large errors (e.g., 180$^\circ$ error) are fairly observed due to rotational ambiguity or lack of discriminate visual features. Properly handling these outliers during training is one of the keys to success in probabilistic modeling of rotations. 

In Figure \ref{fig:grad_scatter}, for matrix Fisher distribution and Rotation Laplace distribution, we visualize the gradient magnitudes $\|\partial\mathcal{L}/\partial\text{(distribution param.)}\|$ w.r.t. the prediction errors on ModelNet10-SO3 dataset after convergence, where each point is a data point in the test set.
As shown in the figure, for matrix Fisher distribution, predictions with larger errors clearly yield larger gradient magnitudes, and those with near 180$^\circ$ errors (the outliers) have the biggest impact. 
Given that outliers may be inevitable and hard to be fixed, they may severely disturb the training process and the sensitivity to outliers can result in a poor fit \citep{murphy2012machine,nair2022maximum}. In contrast, for our Rotation Laplace distribution, the gradient magnitudes are not affected by the prediction errors much, leading to a stable learning process.

Consistent results can also be seen in Figure \ref{fig:teaser} of the main paper, where the red dots illustrate the \textit{sum} of the gradient magnitude over the population within an interval of prediction errors. We argue that, at convergence, the gradient should focus more on the large population with low errors rather than fixing the unavoidable large errors.
}

\ree{\section{Uncertainty quantification measured by distribution entropy}}
\label{sec:entropy}

\begin{figure}[t]
    \centering
    \includegraphics[width=0.48\linewidth]{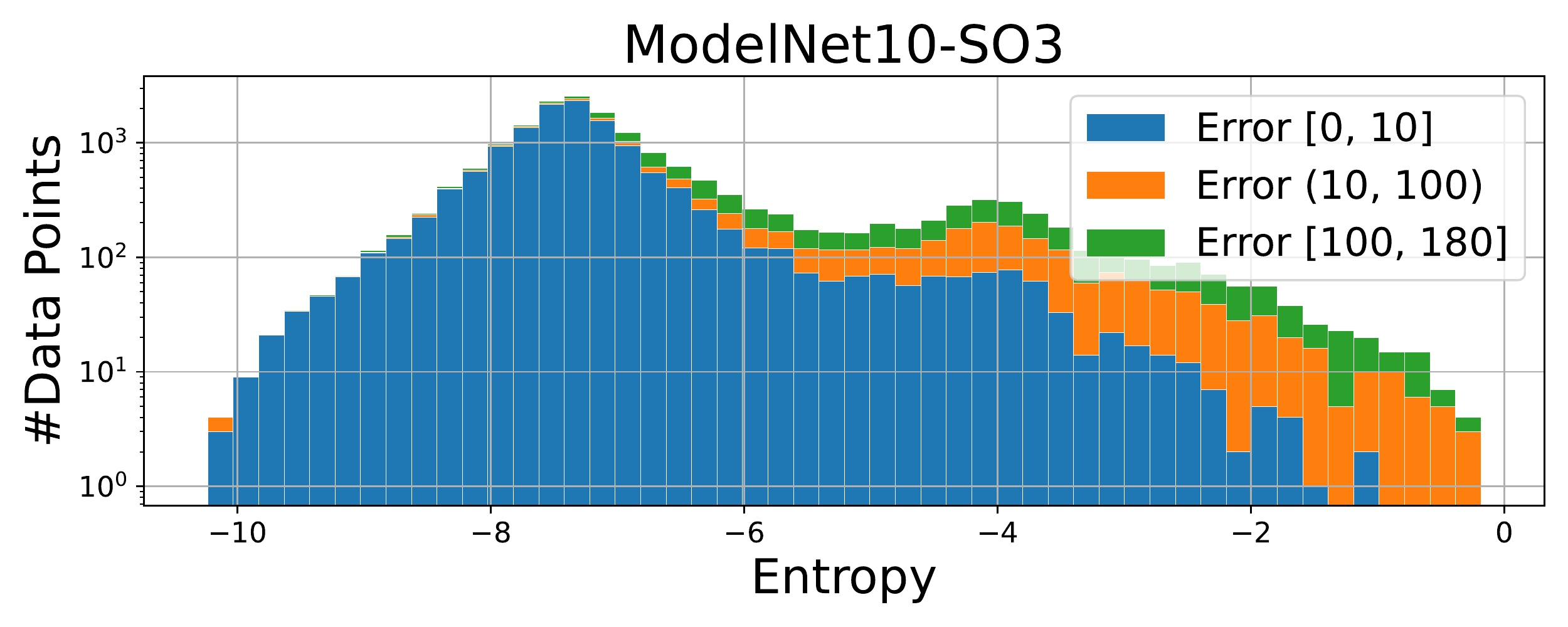}
    \hspace{3mm}
    \includegraphics[width=0.48\linewidth]{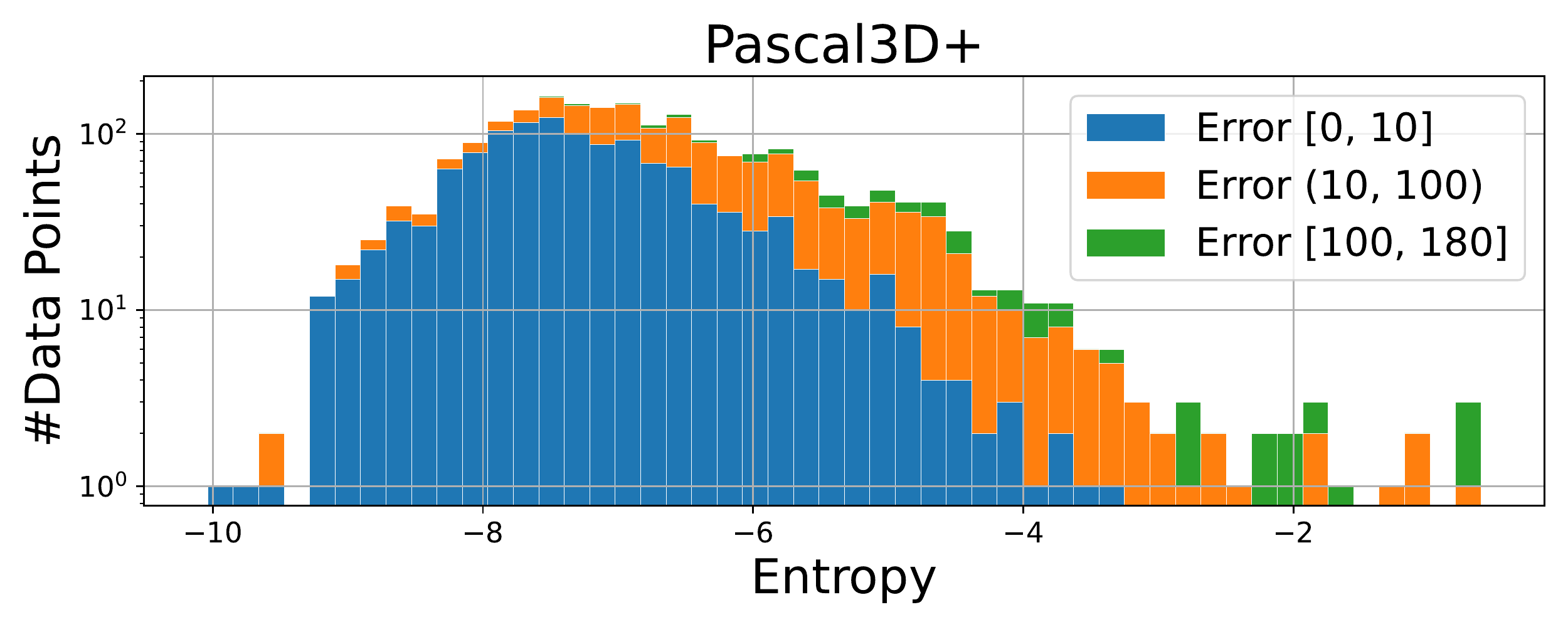}
    \vspace{-3mm}
    \caption{\textbf{Visualization of the indication ability of the distribution entropy w.r.t. the performance.} The horizontal axis is the distribution entropy and the vertical axis is the number of data points (in log scale), color coded by the errors (in degrees). The experiments are done on the test set of ModelNet10-SO3 dataset (left) and Pascal3D+ dataset (right).}
    \vspace{-2mm}
	\label{fig:uncertain}
\end{figure}

Probabilistic modeling of rotation naturally models the uncertainty information of rotation regression. \cite{yin2022fishermatch} proposes to use the \textit{entropy} of the distribution as an uncertainty measure. We adopt it as the uncertainty indicator of Rotation Laplace distribution and plot the relationship between the error of the prediction and the corresponding distribution entropy on the testset of ModelNet10-SO3 and Pascal3D+ datasets in Figure \ref{fig:uncertain}.
As shown in the figure, 
predictions with lower entropies (i.e., lower uncertainty) clearly achieve higher accuracy than predictions with large entropies, demonstrating the ability of uncertainty estimation of our Rotation Laplace distribution.
We compute the entropy via discretization, \ree{
where $\SO$ space is quantized into a finite set of equivolumetric girds $\mathcal{G}=\{\mathbf{R}|\mathbf{R}\in \SO\}$, and
}
\begin{equation*}
\scriptsize
    H\left(p\right)=-\int_{\SO} p \log p \mathrm{d} \mathbf{R}
    \approx 
    \ree{-\sum_{\mathbf{R}_i \in \mathcal{G}} p_i \log p_i \Delta \mathbf{R}_i}
\end{equation*}
\ree{We use about 0.3M grids to discretize $\SO$ space.}

\ree{
\section{Effect of Different Numbers of Discretization Samples}

To compute the normalization factor of our distribution, we discretize $\SO$ space into a finite set of equivolumetric grids using Hopf fibration. Here we show the comparison on different numbers of samples. We experiment with ModelNet10-SO3 toilet dataset on a single 3090 GPU.

As stated in Table \ref{tab:discrete}, the approximation with too few samples leads to inferior performance, and increasing the number of samples yields  a better performance at the cost of a longer runtime. The performance improvement saturates when the number of samples is sufficient. We choose to use 37k samples in our experiments.
}

{\setlength{\tabcolsep}{8pt}
\begin{table}[t]
  \centering
  \caption{Comparison on different numbers of discretization samples. The experiment is done on ModelNet10-SO3 toilet dataset on a single 3090 GPU.}
  \fontsize{8}{9.6}\selectfont
    \begin{tabular}{ccccc}
    \toprule
    Number of samples & Training time (min)$\downarrow$ & Mean($^\circ$)$\downarrow$ & Med.($^\circ$)$\downarrow$ & Acc@5$^\circ$$\uparrow$  \\
    \midrule
    0.6k  & 122   & 5.8   & 2.8   & 0.80 \\
    4.6k  & 122   & 5.3   & 2.6   & 0.82 \\
    37k   & 136   & 5.3   & 2.6   & 0.83 \\
    295k  & 168   & 5.3   & 2.5   & 0.82 \\
    \bottomrule
    \end{tabular}%
  \label{tab:discrete}%
\end{table}%
}

\section{Additional Results}

\subsection{Additional Numerical Results}
\label{sec:supp_results}

Table \ref{tab:supp_modelnet} and \ref{tab:supp_pascal} extend the results on ModelNet10-SO3 dataset and Pascal3D+ dataset in the main paper and show the per-category results.  
Our prediction with Rotation Laplace distribution is at or near state-of-the-art on many categories. The numbers for  baselines are quoted from \cite{murphy2021implicit}.

\begin{table}[t]
\centering
\fontsize{7.5}{10}\selectfont
\caption{Per-category results ModelNet10-SO3 dataset.}
\resizebox{0.99\textwidth}{!}{
\begin{tabular}{@{}ll@{\hskip 0.3in}c@{\hskip 16pt}cccccccccc}
\toprule
       &                                      & {avg.}     & {bathtub}  & {bed}      & {chair}    & {desk}     & {dresser}  & {tv}       & {n. stand} & {sofa}     & {table}    & {toilet}   \\
\midrule
\multirow{4}{*}{Acc@15\textdegree$\uparrow$}                   & \citet{deng2022deep}            & 0.562      & 0.140      & 0.788      & 0.800      & 0.345      & 0.563      & 0.708      & 0.279      & 0.733      & 0.440      & 0.832      \\      
       & \citet{prokudin2018deep}        & 0.456      & 0.114      & 0.822      & 0.662      & 0.023      & 0.406      & 0.704      & 0.187      & 0.590      & 0.108      &  0.946 \\
       & \citet{mohlin2020probabilistic} &  0.693 &  0.322 & \bgl 0.882 & \bgl 0.881 &  0.536 &  0.682 &  0.790 &  0.516 & \bgl 0.919 &  0.446 & \bgl 0.957 \\      
       & \citet{murphy2021implicit}     & \bgl 0.719 & \bgd 0.392 &  0.877 &  0.874 & \bgl 0.615 & \bgl 0.687 & \bgl 0.799 & \bgl 0.567 &  0.914 & \bgd 0.523 & 0.945      \\
       & Rotation Laplace               & \bgd 0.741  & \bgl 0.390  & \bgd 0.902  & \bgd 0.909  & \bgd 0.644  & \bgd 0.722  & \bgd 0.815  & \bgd 0.590  & \bgd 0.934  & \bgl 0.521  & \bgd 0.977  \\
\midrule
\multirow{4}{*}{Acc@30\textdegree$\uparrow$}                   & \citet{deng2022deep}            & 0.694      & 0.325      & 0.880      & 0.908      & 0.556      & 0.649      & 0.807      & 0.466      & 0.902      & 0.485      & 0.958      \\      
       & \citet{prokudin2018deep}        & 0.528      & 0.175      & 0.847      & 0.777      & 0.061      & 0.500      & 0.788      & 0.306      & 0.673      & 0.183      & 0.972 \\
       & \citet{mohlin2020probabilistic} & \bgl 0.757 &  0.403 & \bgl 0.908 & \bgl 0.935 & \bgl 0.674 & \bgl 0.739 & \bgl 0.863 & \bgl 0.614 & \bgl 0.944 & 0.511 & \bgl 0.981 \\      
       & \citet{murphy2021implicit}     & 0.735 & \bgl 0.410 &  0.883 &  0.917 & 0.629 & 0.688 &  0.832 &  0.570 &  0.921 & \bgl 0.531 & 0.967      \\
       & Rotation Laplace               & \bgd  0.770  & \bgd 0.430 & \bgd 0.911  & \bgd 0.940 & \bgd 0.698  & \bgd 0.751  & \bgd 0.869 & \bgd 0.625  & \bgd 0.946  & \bgd 0.541  & \bgd 0.986  \\
\midrule
\multirow{4}{*}{\shortstack{Median \\ Error ($^\circ$)$\downarrow$}} & \citet{deng2022deep}            & 32.6       & 147.8      & 9.2        & 8.3        & 25.0       & 11.9       & 9.8        & 36.9       & 10.0       & 58.6       & 8.5        \\
       & \citet{prokudin2018deep}        & 49.3       &  122.8 & \bgl 3.6   & 9.6        & 117.2      & 29.9       & 6.7        & 73.0       & 10.4       & 115.5      &  4.1   \\
       & \citet{mohlin2020probabilistic} & \bgl 17.1  & \bgl 89.1  &  4.4   & \bgl 5.2   &  13.0  &  6.3   &  5.8   &  13.5  & \bgl 4.0   &  25.8  & \bgl 4.0   \\      
       & \citet{murphy2021implicit}     &  21.5  & 161.0      &  4.4   &  5.5   & \bgl 7.1   & \bgl 5.5   & \bgl 5.7   & \bgl 7.5   &  4.1   & \bgd 9.0   & 4.8        \\
       & Rotation Laplace               & \bgd 12.2  & \bgd 85.1  & \bgd 2.3  & \bgd 3.4 & \bgd 5.4 & \bgd 2.7  & \bgd 3.7 & \bgd 4.8  & \bgd 2.1  & \bgl 9.6  & \bgd 2.5  \\
\bottomrule
\end{tabular}
}
\label{tab:supp_modelnet}
\end{table}

\begin{table*}[t]
\caption{Per-category results on Pascal3D+ dataset.}
  \centering
  \fontsize{7.5}{10}\selectfont
  \resizebox{0.99\textwidth}{!}{
    \begin{tabular}{@{}ll@{\hskip 0.3in}c@{\hskip 16pt}cccccccccccc}
      &                                      & {avg.}     & {aero}  & {bike}      & {boat}    & {bottle}     & {bus}  & {car}       & {chair} & {table}     & {mbike}    & {sofa} & {train}  & {tv} \\
    \midrule
    \multirow{6}{*}{Acc@30\textdegree  $\uparrow$} 
        & \cite{tulsiani2015viewpoints}         & 0.808 & 0.81  & 0.77  &  0.59 & 0.93 &  \bgd {0.98} & 0.89 & 0.80 & 0.62 & \bgl 0.88 & 0.82 & 0.80 & 0.80 \\
        & \cite{mahendran2018mixed}             & \bgl 0.859 & 0.87 & 0.81 &  \bgd {0.64} &  \bgd {0.96} &  0.97 &  0.95 &  \bgd {0.92} & 0.67 & 0.85 & \bgd 0.97 & \bgl 0.82 & \bgl 0.88 \\
        & \cite{liao2019spherical}              & 0.819 & 0.82 & 0.77 & 0.55 & 0.93 & 0.95 & 0.94 & 0.85 & 0.61 & 0.80 &  0.95 &  \bgd {0.83} & 0.82     \\   
        & \cite{prokudin2018deep}               &  0.838 &   0.89 &  0.83 & 0.46 &  \bgd {0.96} & 0.93 & 0.90 & 0.80 & \bgl 0.76 & \bgd 0.90 &  {0.90} & \bgl 0.82 &  \bgd {0.91} \\
        & \cite{mohlin2020probabilistic}        & 0.825 &  \bgd {0.90} & \bgl 0.85 & 0.57 &  0.94 & 0.95 &  \bgd {0.96} & 0.78 & 0.62 & 0.87 & 0.85 & 0.77 & 0.84 \\
        & \cite{murphy2021implicit}             & 0.837 & 0.81 & \bgl 0.85 & 0.56 & 0.93 & 0.95 & 0.94 &  0.87 &  \bgd {0.78} & 0.85 & 0.88 & 0.78 & 0.86  \\
        & Rot. Laplace (Ours)                   & \bgd {0.876} & \bgd {0.90}  & \bgd {0.90}  & \bgl 0.60 & \bgd {0.96} & \bgd {0.98} & \bgd {0.96} & \bgl 0.91  & \bgl 0.76  & \bgl 0.88 & \bgd 0.97 &  0.81 & \bgl 0.88  \\
       
    \midrule
    \multirow{6}{*}{\shortstack{Median\\error ($^\circ$)  $\downarrow$}}
        & \cite{tulsiani2015viewpoints} & 13.6 & 13.8 & 17.7 & \bgl 21.3 & 12.9 & 5.8 & 9.1 & 14.8 & 15.2 & 14.7 & 13.7 & 8.7 & 15.4 \\
        & \cite{mahendran2018mixed} &  10.1 &  \bgd {8.5} &  14.8 &  \bgd {20.5} &  7.0 &  3.1 &  5.1 & \bgl 9.3 &  11.3 & 14.2 & 10.2 & \bgl 5.6 & \bgl 11.7 \\
        & \cite{liao2019spherical}            & 13.0 & 13.0 & 16.4 & 29.1 & 10.3 & 4.8 & 6.8 & 11.6 & 12.0 & 17.1 & 12.3 & 8.6 & 14.3       \\
        & \cite{prokudin2018deep} & 12.2 &  9.7 & 15.5 & 45.6 &  \bgd {5.4} & \bgl 2.9 &  \bgd {4.5} & 13.1 & 12.6 &  \bgd {11.8} & \bgl 9.1 &  \bgd {4.3} & 12.0 \\
        & \cite{mohlin2020probabilistic} & 11.5 & 10.1 & 15.6 & 24.3 & 7.8 & 3.3 & 5.3 & 13.5 & 12.5 &  12.9 & 13.8 & 7.4 & \bgl 11.7\\
        & \cite{murphy2021implicit}                        &  10.3 & 10.8 & \bgl 12.9 & 23.4 & 8.8 & 3.4 & 5.3 &  10.0 &  \bgd {7.3} & 13.6 &  9.5  & 6.4  & 12.3 \\
        & Rot. Laplace (Ours)             &  \bgd {9.4}  & \bgl 8.6 & \bgd {11.7} &  21.8 & \bgl 6.9 & \bgd {2.8} & \bgl 4.8 & \bgd {7.9} & \bgl 9.1 & \bgl 12.2 & \bgd {8.1} & 6.9 & \bgd {11.6}    \\
      
      \bottomrule
    \end{tabular}
  }
  \label{tab:supp_pascal}
  \end{table*}

\ree{
\subsection{Additional Visual Results}

We show additional visual results on ModelNet10-SO3 dataset in Figure \ref{fig:vis_modelnet} and on Pascal3D+ dataset in Figure \ref{fig:vis_pascal}. As shown in the figures, our distribution provides rich information about the rotation estimations.

To visualize the predicted distributions, we adopt two popular visualization methods used in \cite{mohlin2020probabilistic} and \cite{murphy2021implicit}. 
The visualization in \cite{mohlin2020probabilistic} is achieved by summing the three marginal distributions over the standard basis of $\mathbb{R}^3$ and displaying them on the sphere with color coding. \cite{murphy2021implicit} introduces a new visualization method based on discretization over $\SO$. It projects a great circle of points on $\SO$ to each point on the 2-sphere, and then uses the color wheel to indicate the location on the great circle. The probability density is shown by the size of the points on the plot. See the corresponding papers for more details.

\begin{figure}[t]
    \centering
    \begin{tabular}{cccccc}
    \includegraphics[height=1.6cm]{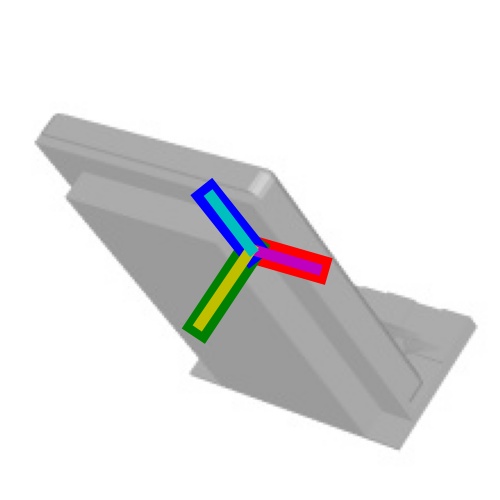}
    &\includegraphics[clip,trim=4.5cm 4cm 4.5cm 1.5cm, height=1.6cm]{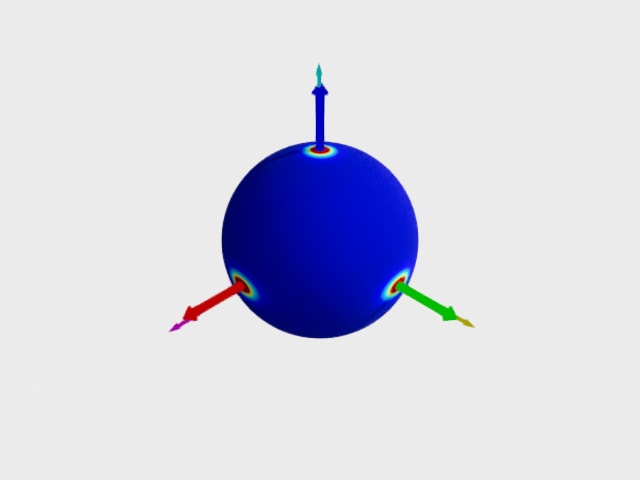}
    &\includegraphics[height=1.2cm]{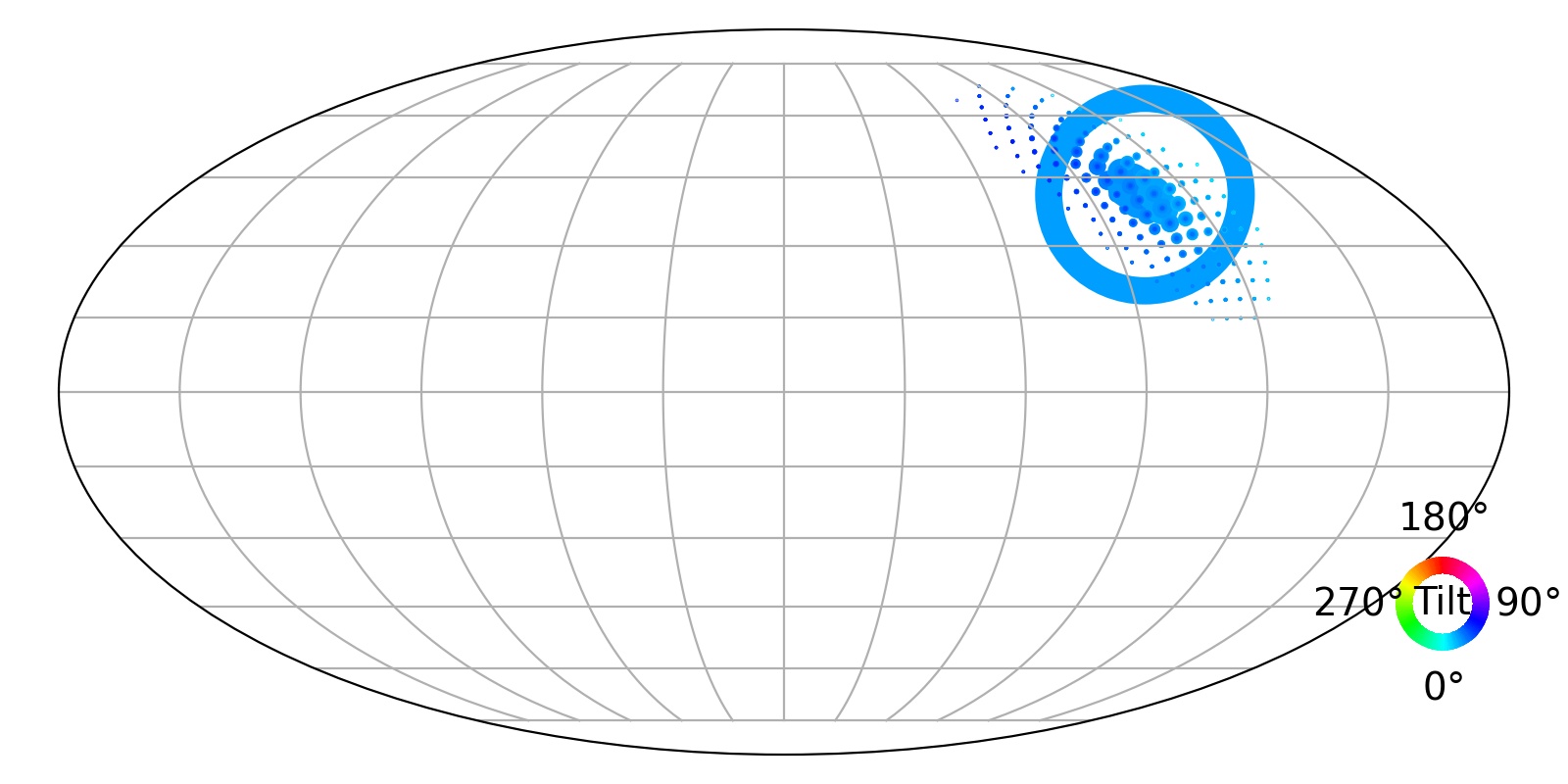}
    \hspace{4mm}
    &\includegraphics[height=1.6cm]{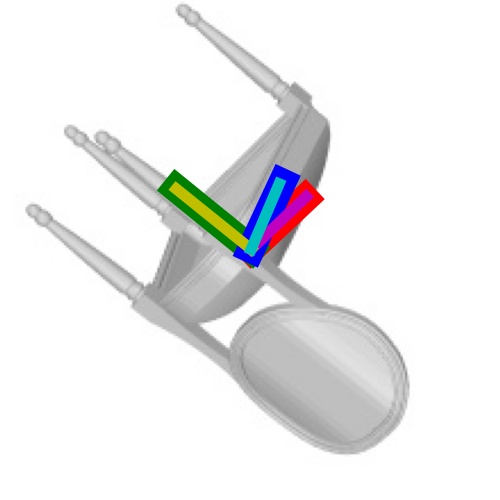}
    &\includegraphics[clip,trim=4.5cm 4cm 4.5cm 1.5cm, height=1.6cm]{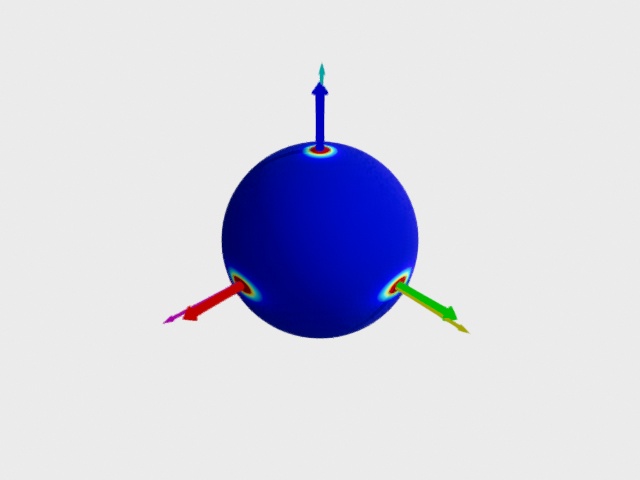}
    &\includegraphics[height=1.2cm]{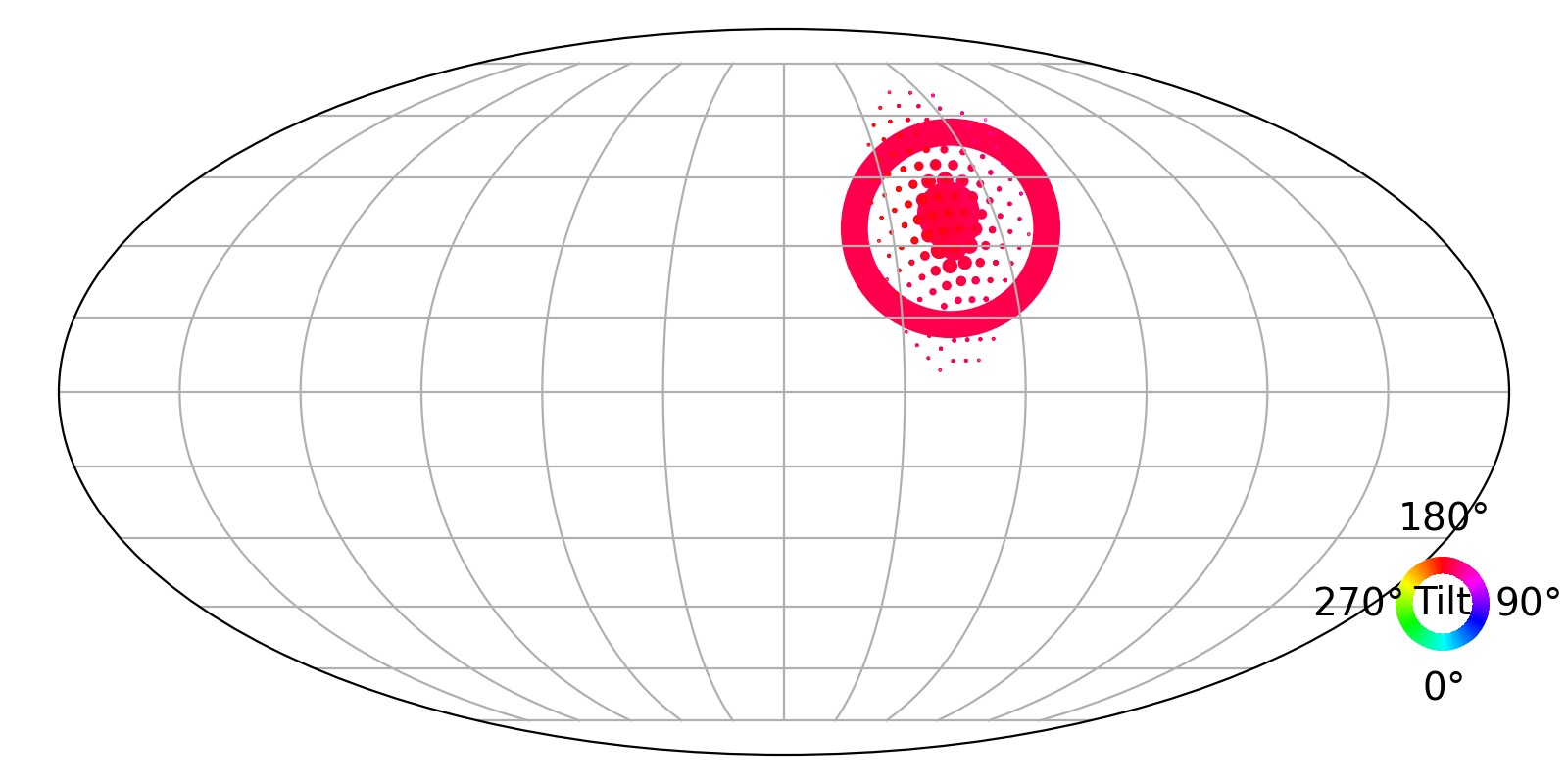}\\
    
    \includegraphics[height=1.6cm]{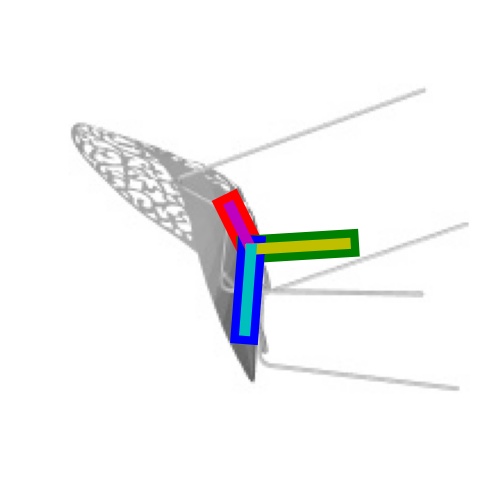}
    &\includegraphics[clip,trim=4.5cm 4cm 4.5cm 1.5cm, height=1.6cm]{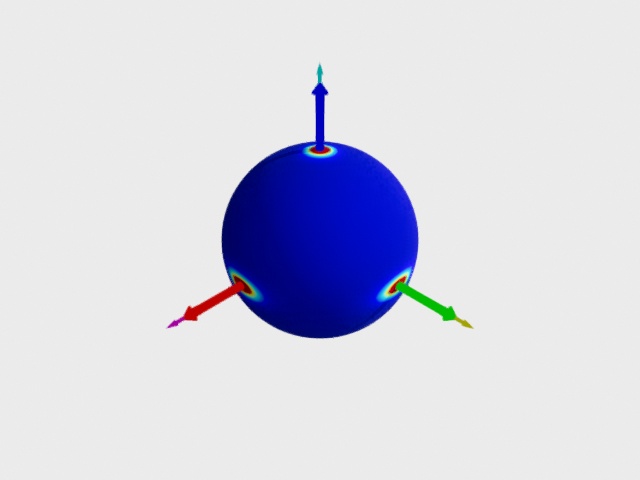}
    &\includegraphics[height=1.2cm]{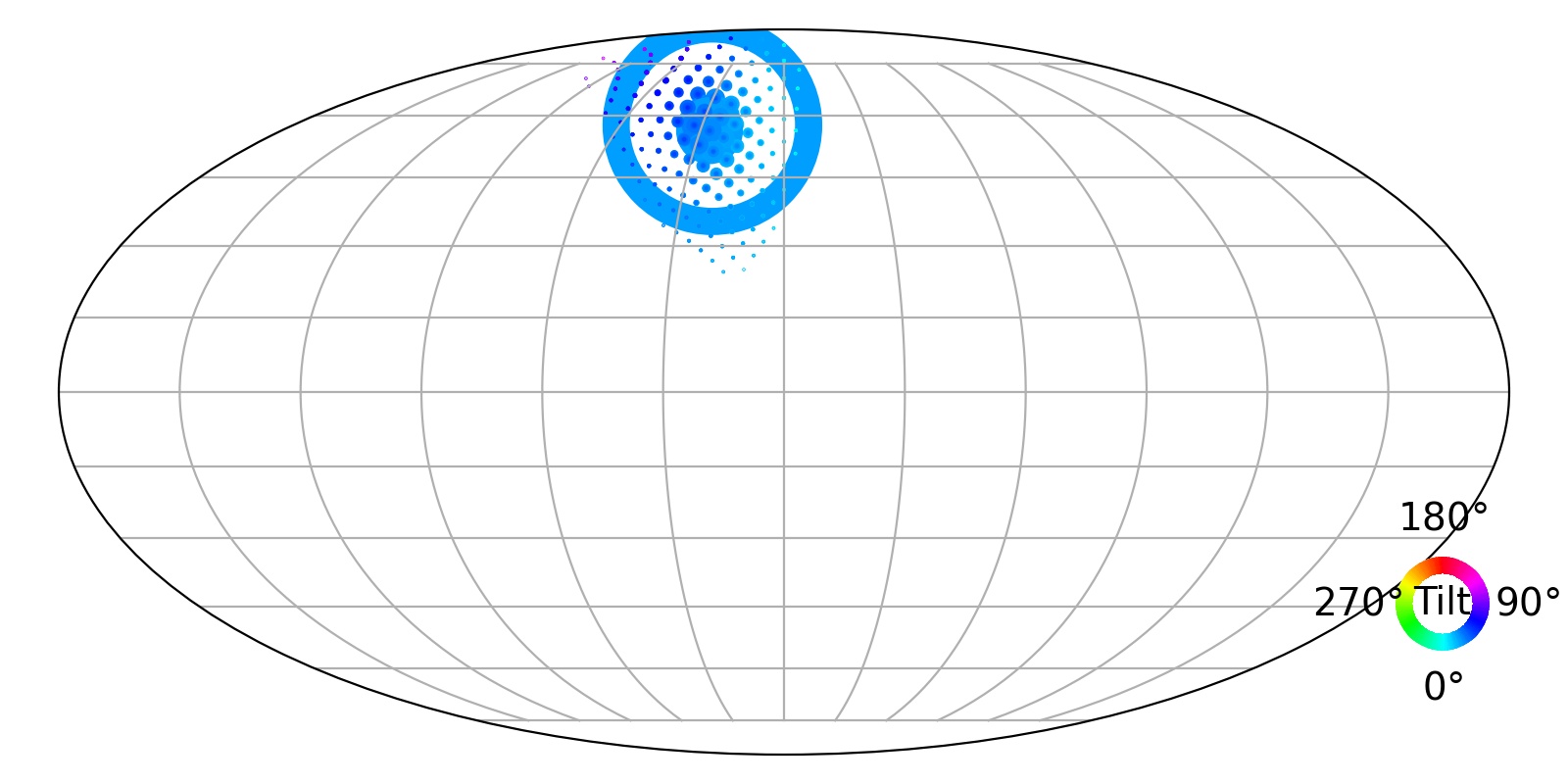}
    \hspace{4mm}
    &\includegraphics[height=1.6cm]{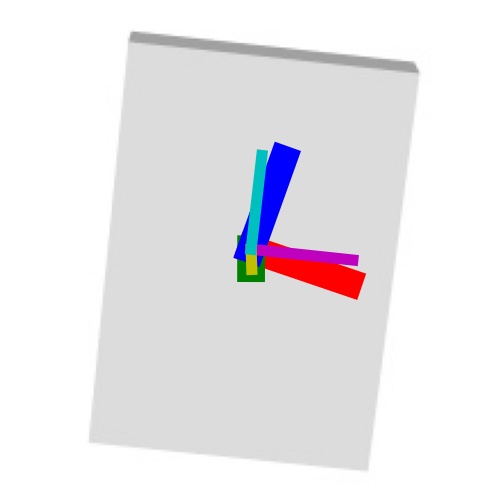}
    &\includegraphics[clip,trim=4.5cm 4cm 4.5cm 1.5cm, height=1.6cm]{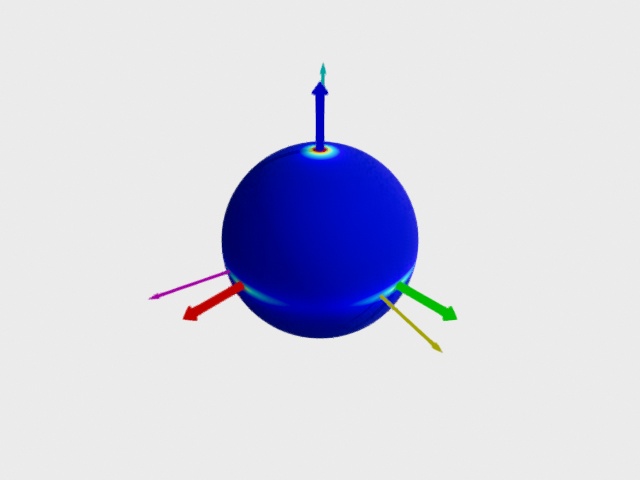}
    &\includegraphics[height=1.2cm]{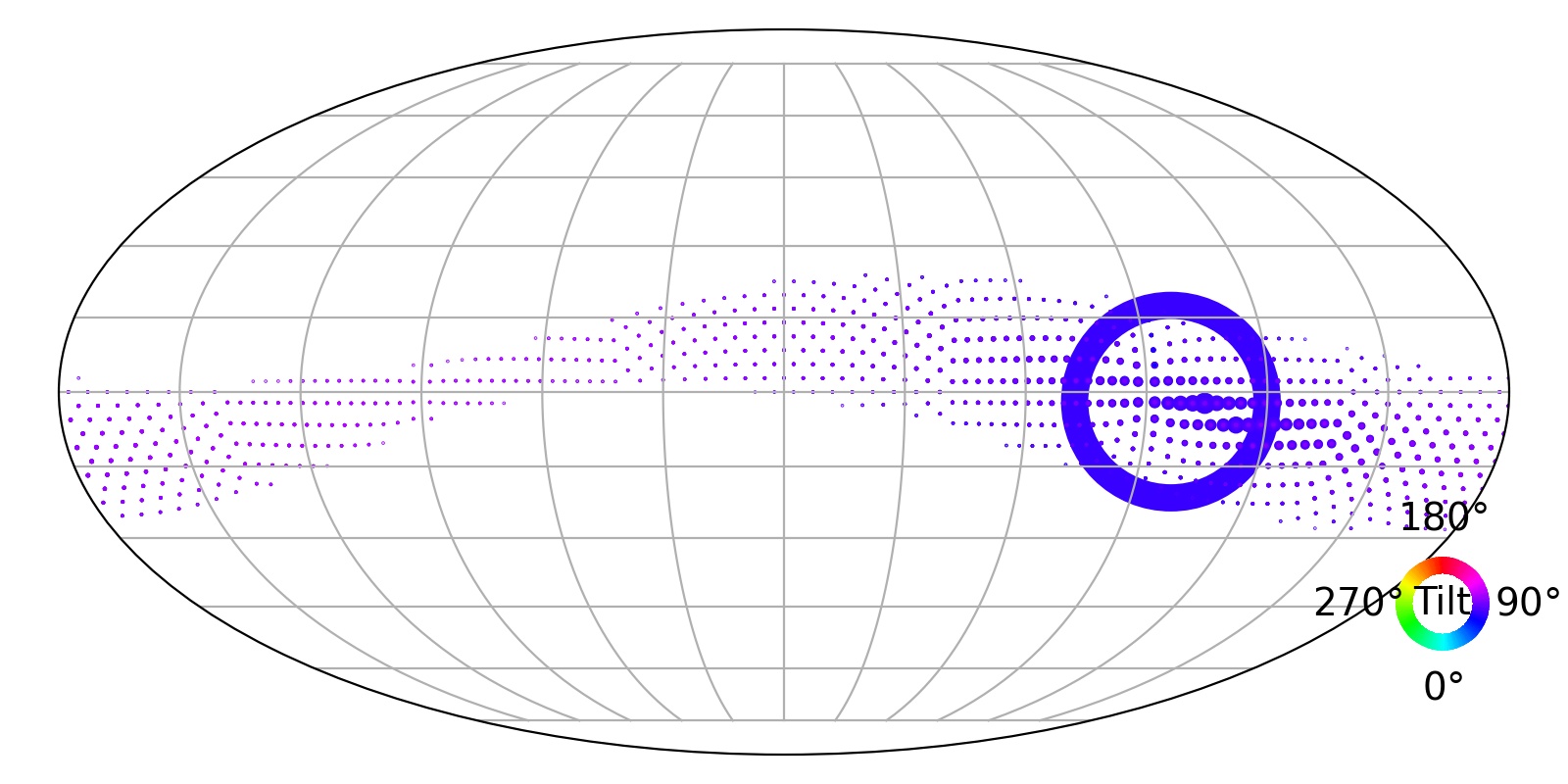}\\
    
    \includegraphics[height=1.6cm]{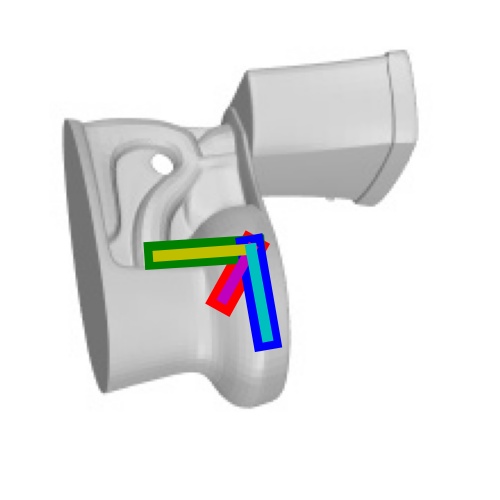}
    &\includegraphics[clip,trim=4.5cm 4cm 4.5cm 1.5cm, height=1.6cm]{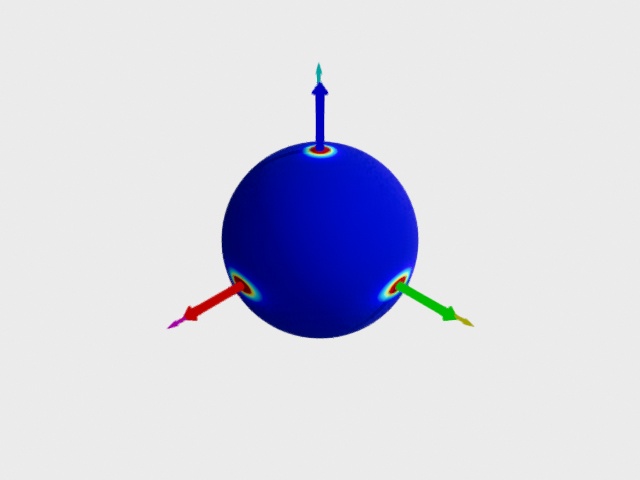}
    &\includegraphics[height=1.2cm]{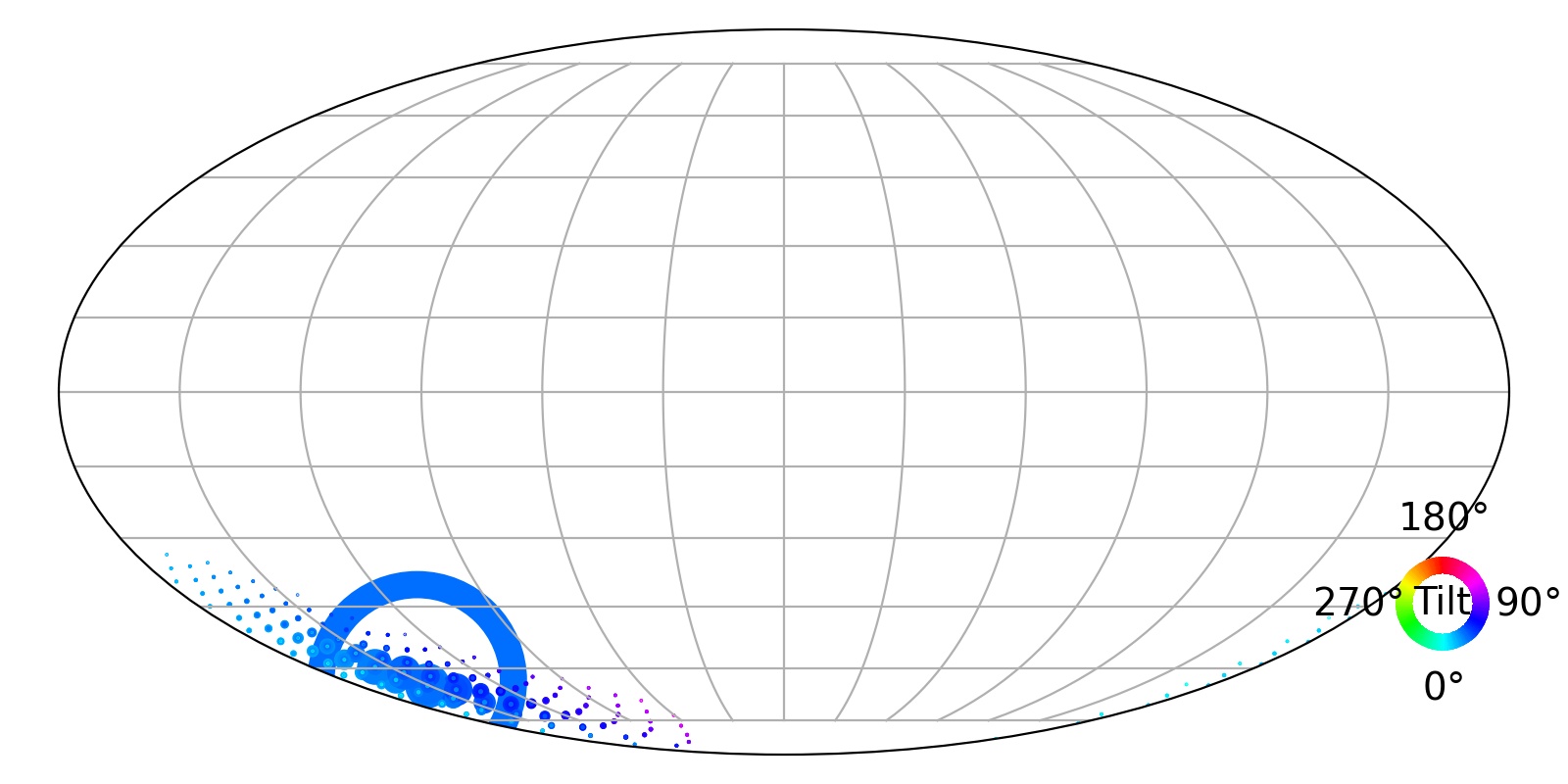}
    \hspace{4mm}
    &\includegraphics[height=1.6cm]{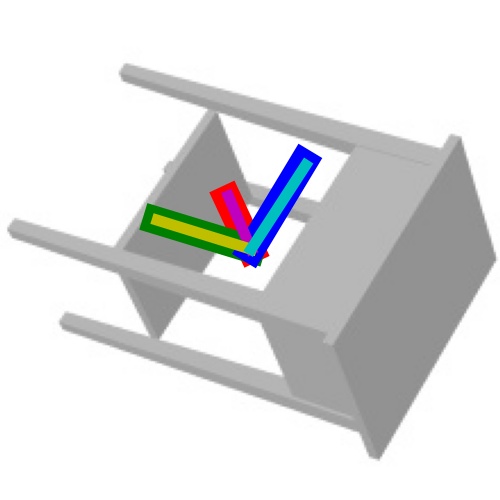}
    &\includegraphics[clip,trim=4.5cm 4cm 4.5cm 1.5cm, height=1.6cm]{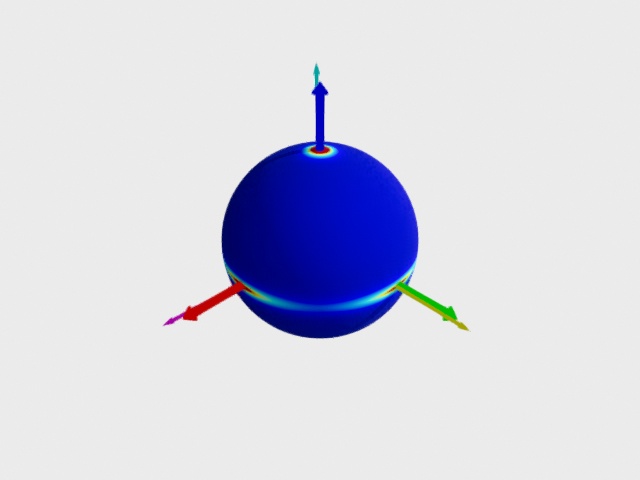}
    &\includegraphics[height=1.2cm]{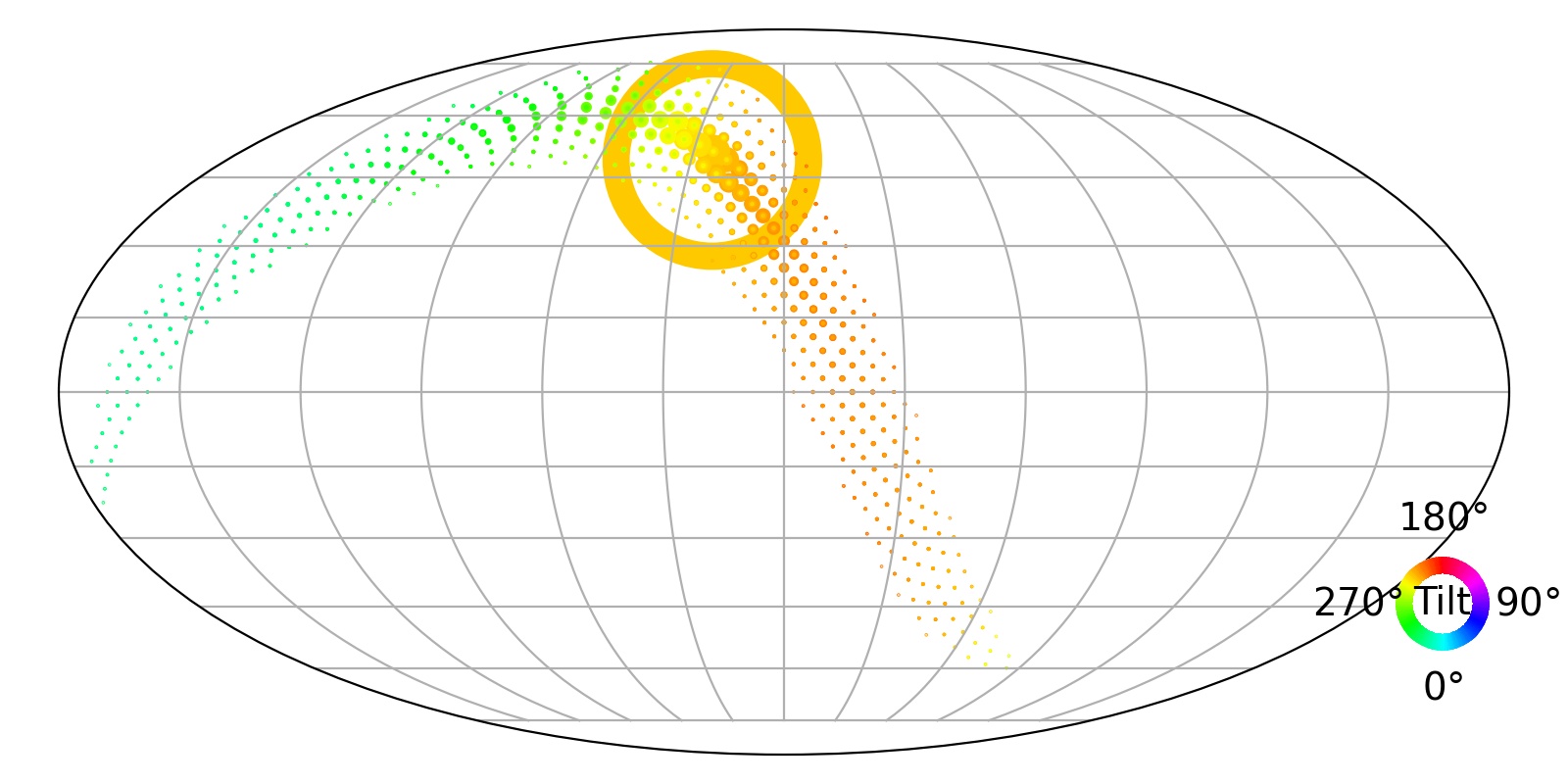}\\
    
    \includegraphics[height=1.6cm]{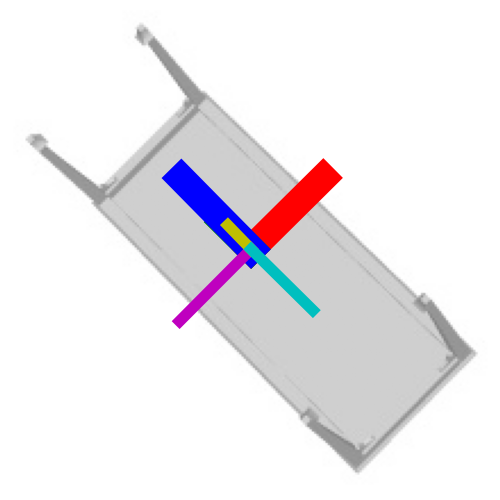}
    &\includegraphics[clip,trim=4.5cm 4cm 4.5cm 1.5cm, height=1.6cm]{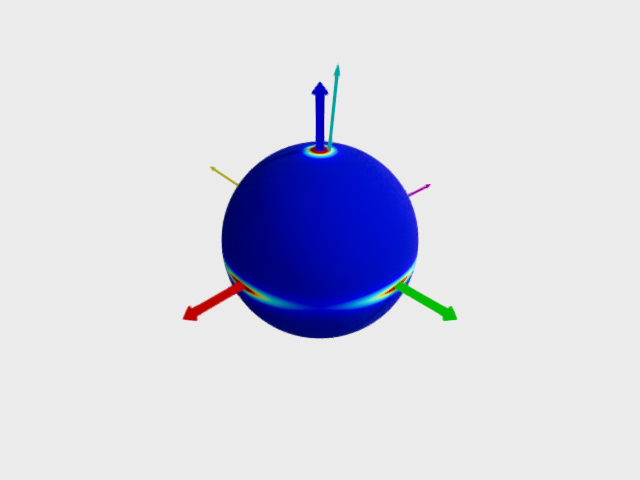}
    &\includegraphics[height=1.2cm]{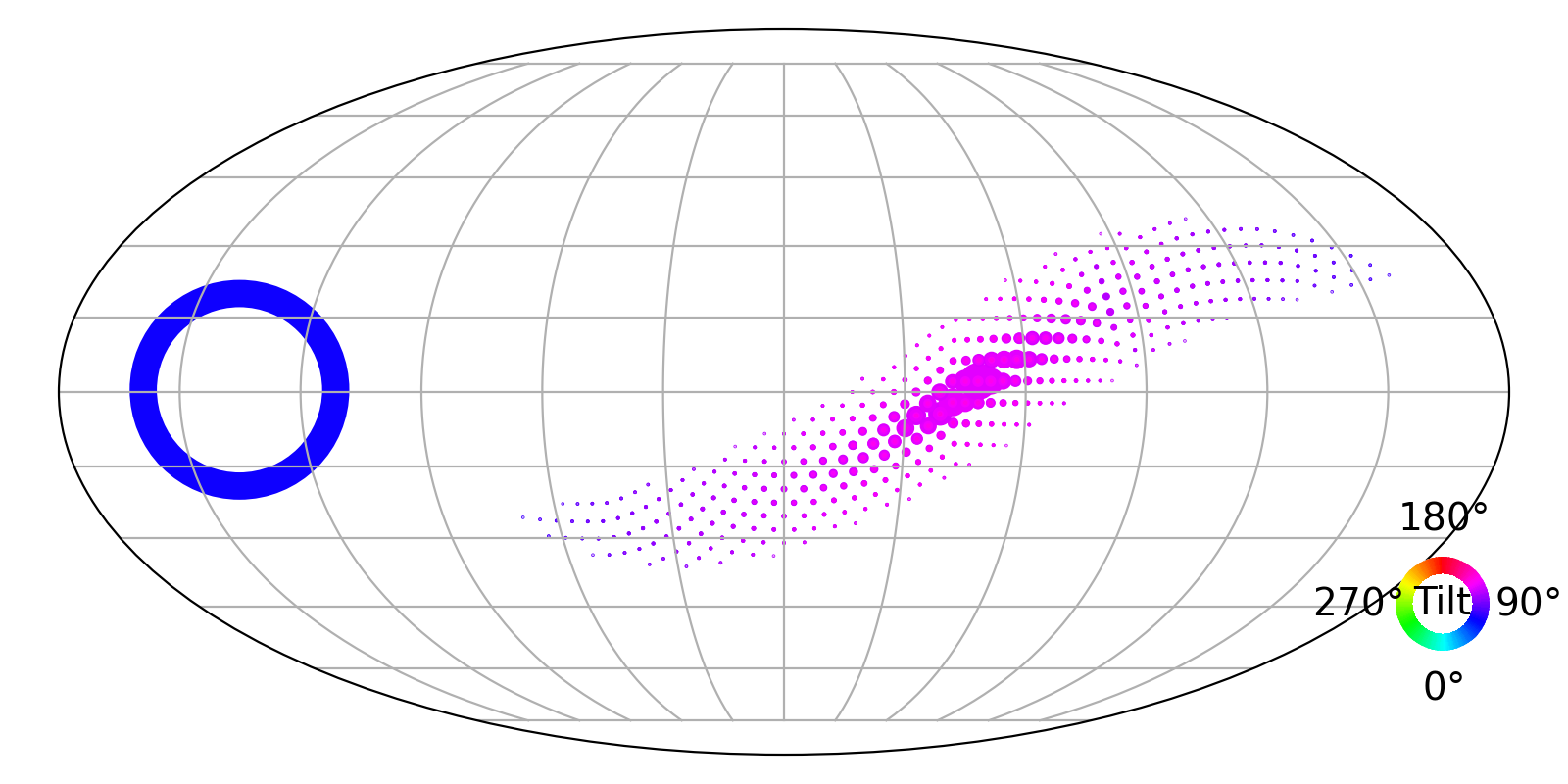}
    \hspace{4mm}
    &\includegraphics[height=1.6cm]{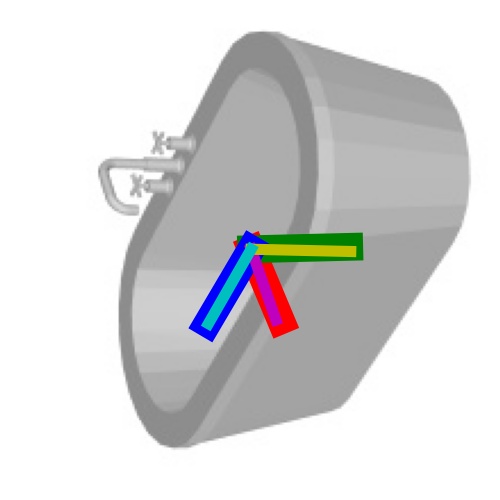}
    &\includegraphics[clip,trim=4.5cm 4cm 4.5cm 1.5cm, height=1.6cm]{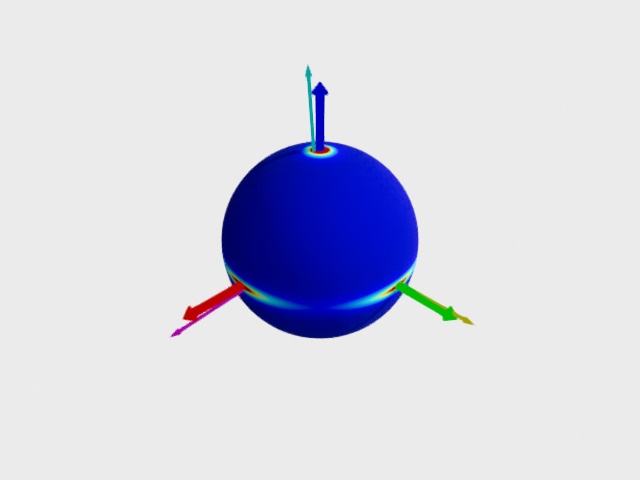}
    &\includegraphics[height=1.2cm]{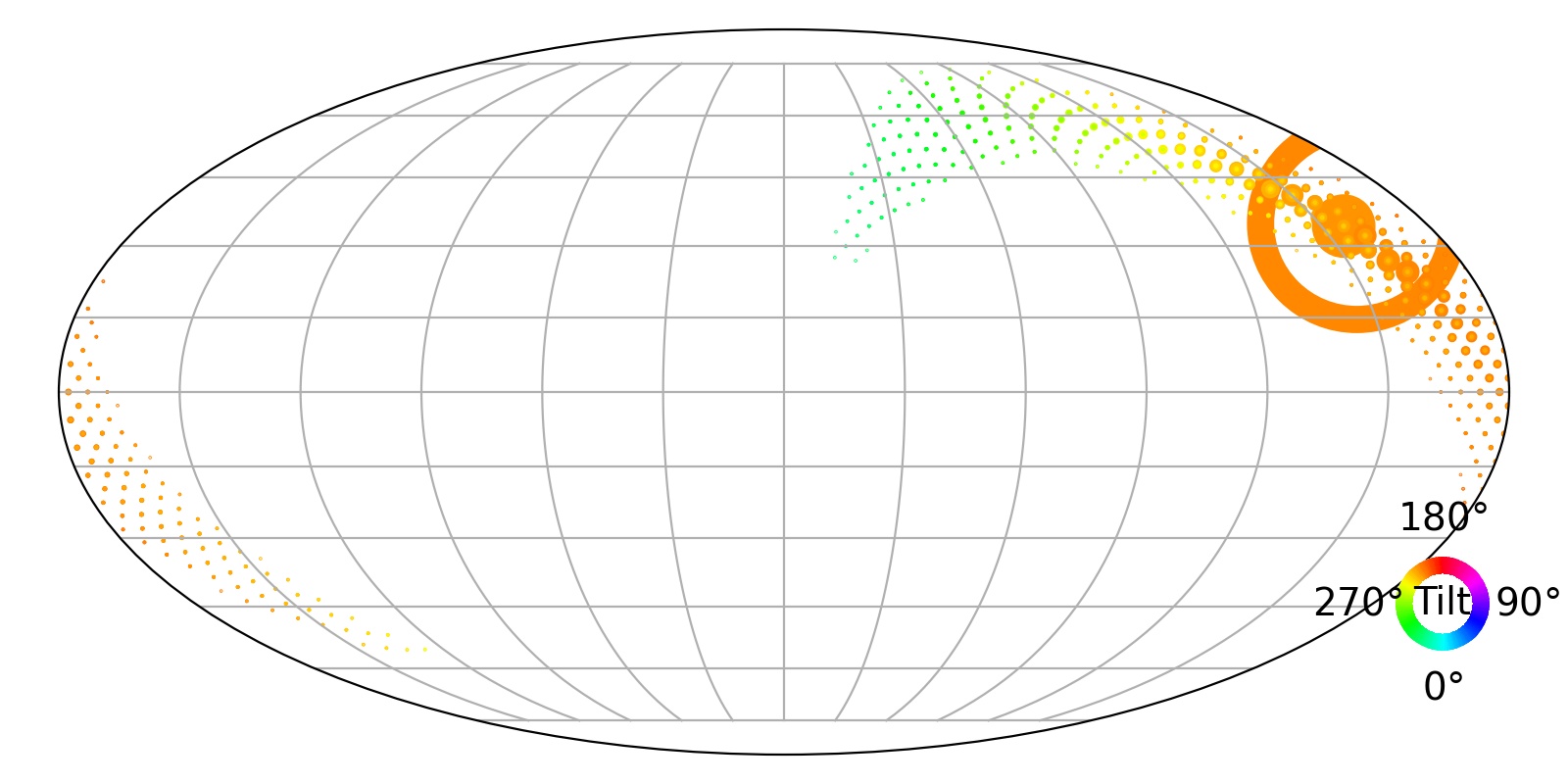}\\
    
    \includegraphics[height=1.6cm]{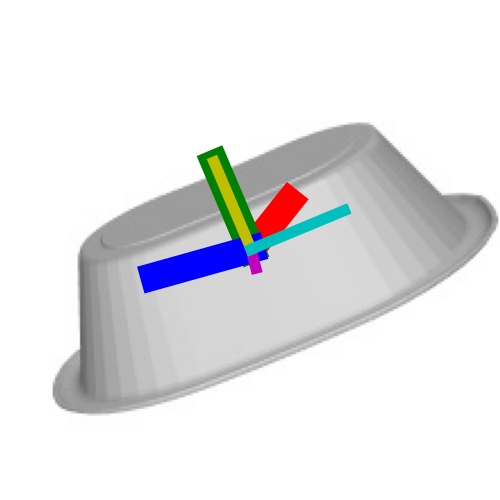}
    &\includegraphics[clip,trim=4.5cm 4cm 4.5cm 1.5cm, height=1.6cm]{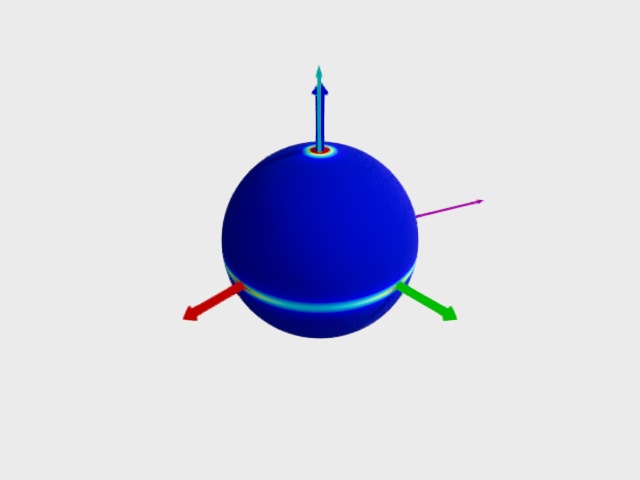}
    &\includegraphics[height=1.2cm]{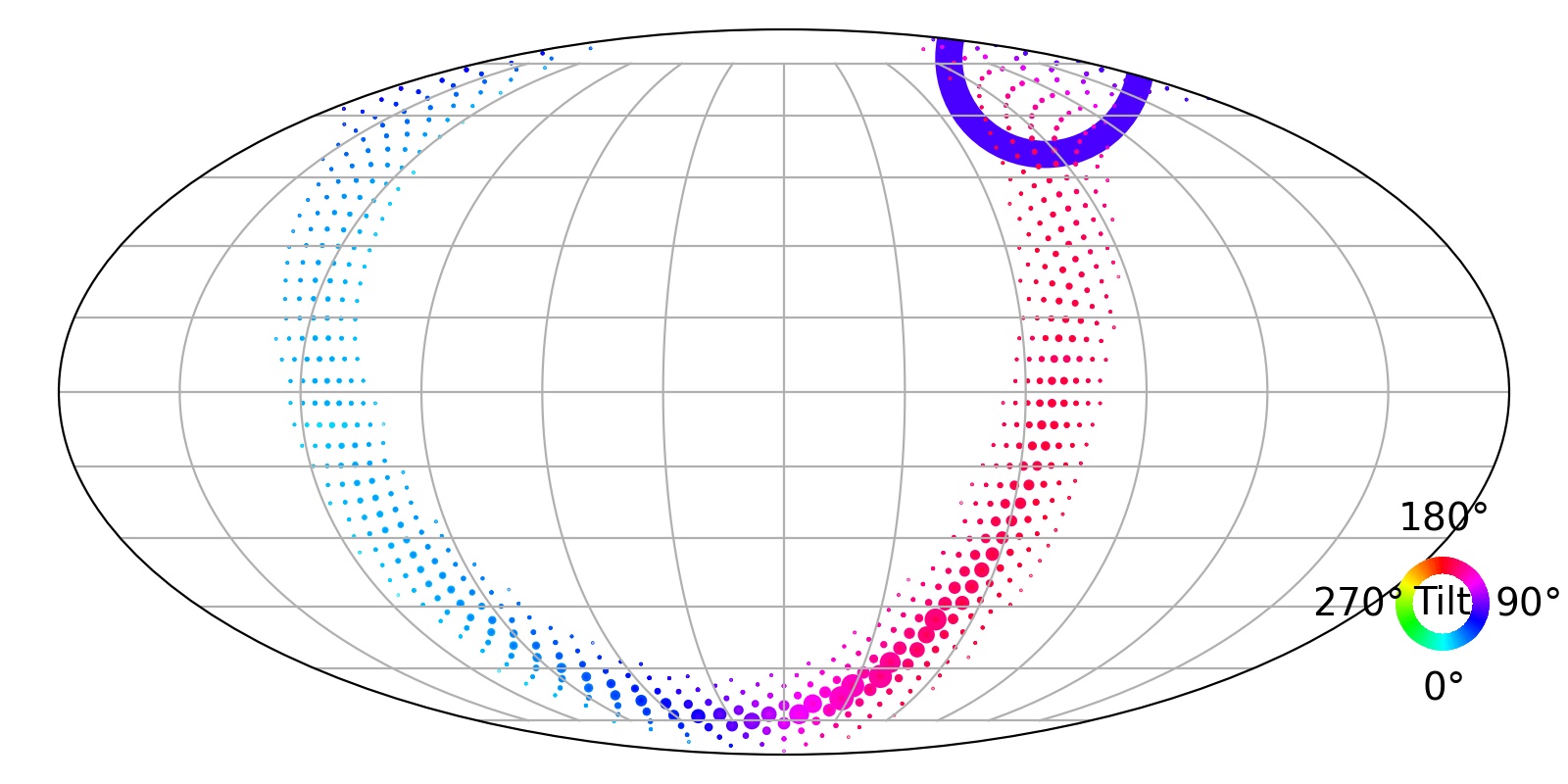}
    \hspace{4mm}
    &\includegraphics[height=1.6cm]{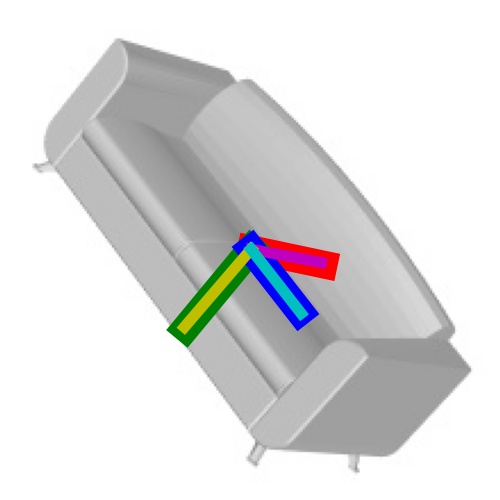}
    &\includegraphics[clip,trim=4.5cm 4cm 4.5cm 1.5cm, height=1.6cm]{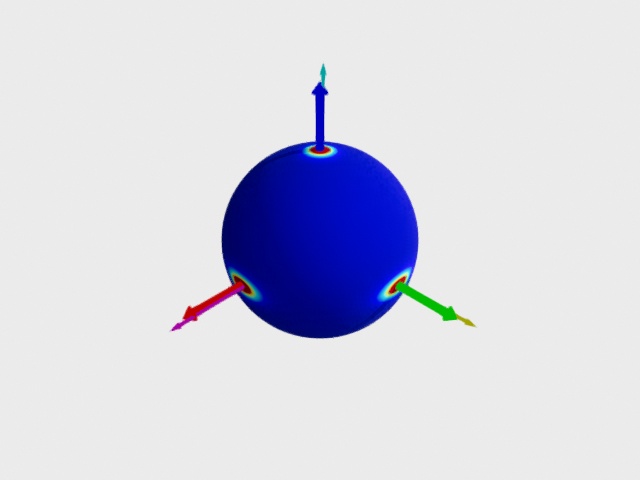}
    &\includegraphics[height=1.2cm]{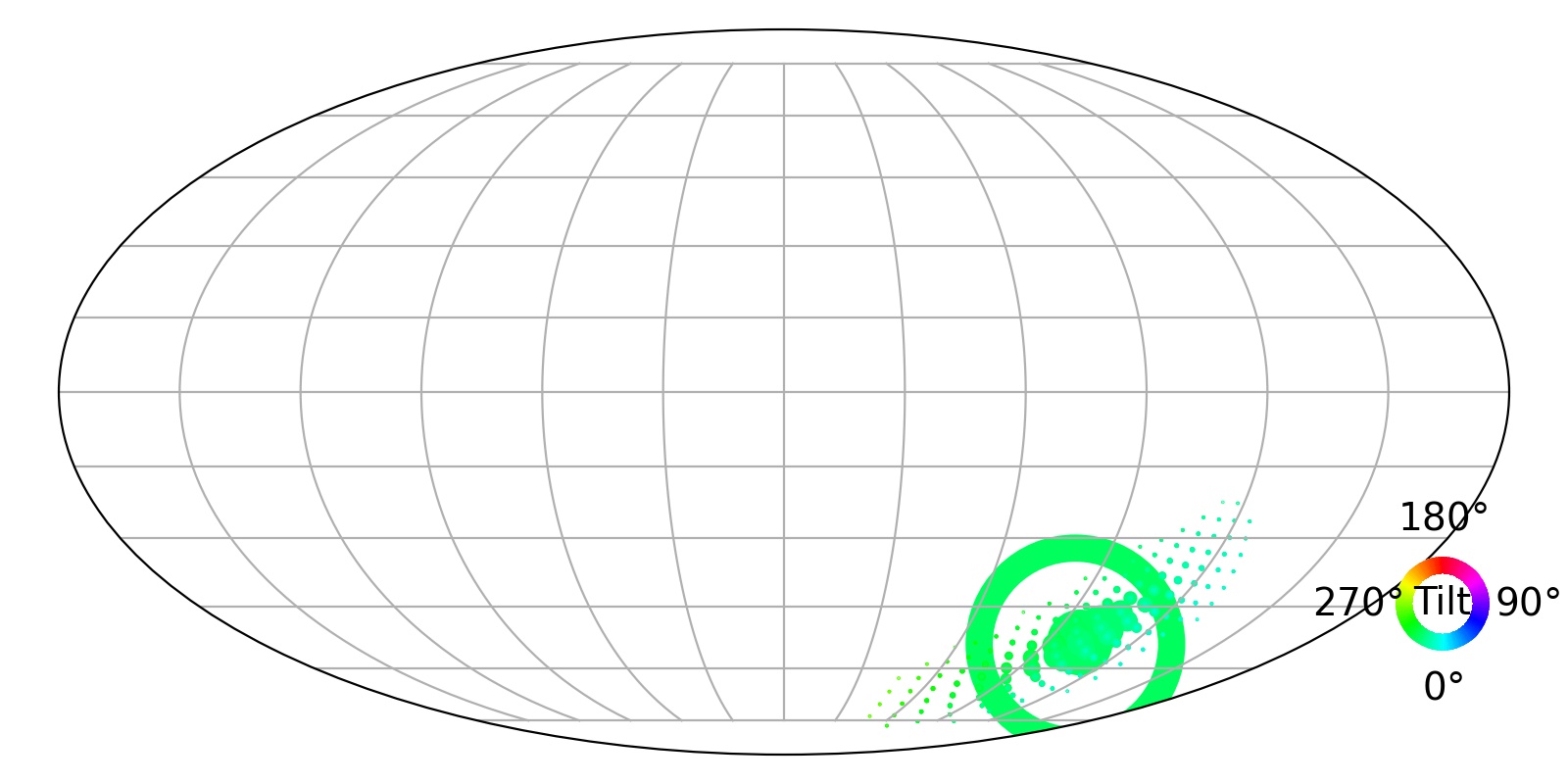}\\
    
    \includegraphics[height=1.6cm]{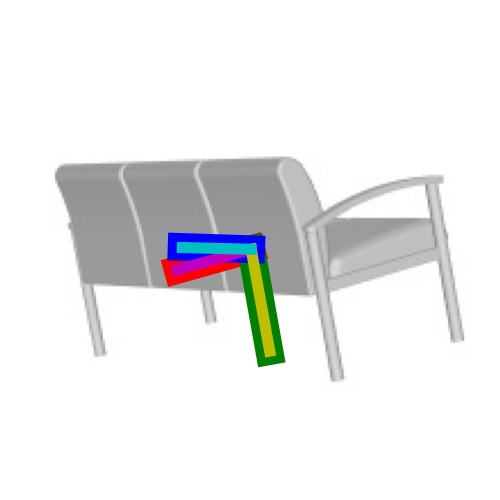}
    &\includegraphics[clip,trim=4.5cm 4cm 4.5cm 1.5cm, height=1.6cm]{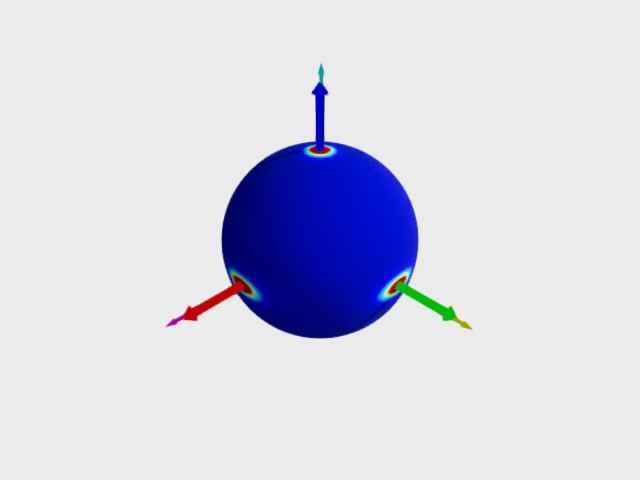}
    &\includegraphics[height=1.2cm]{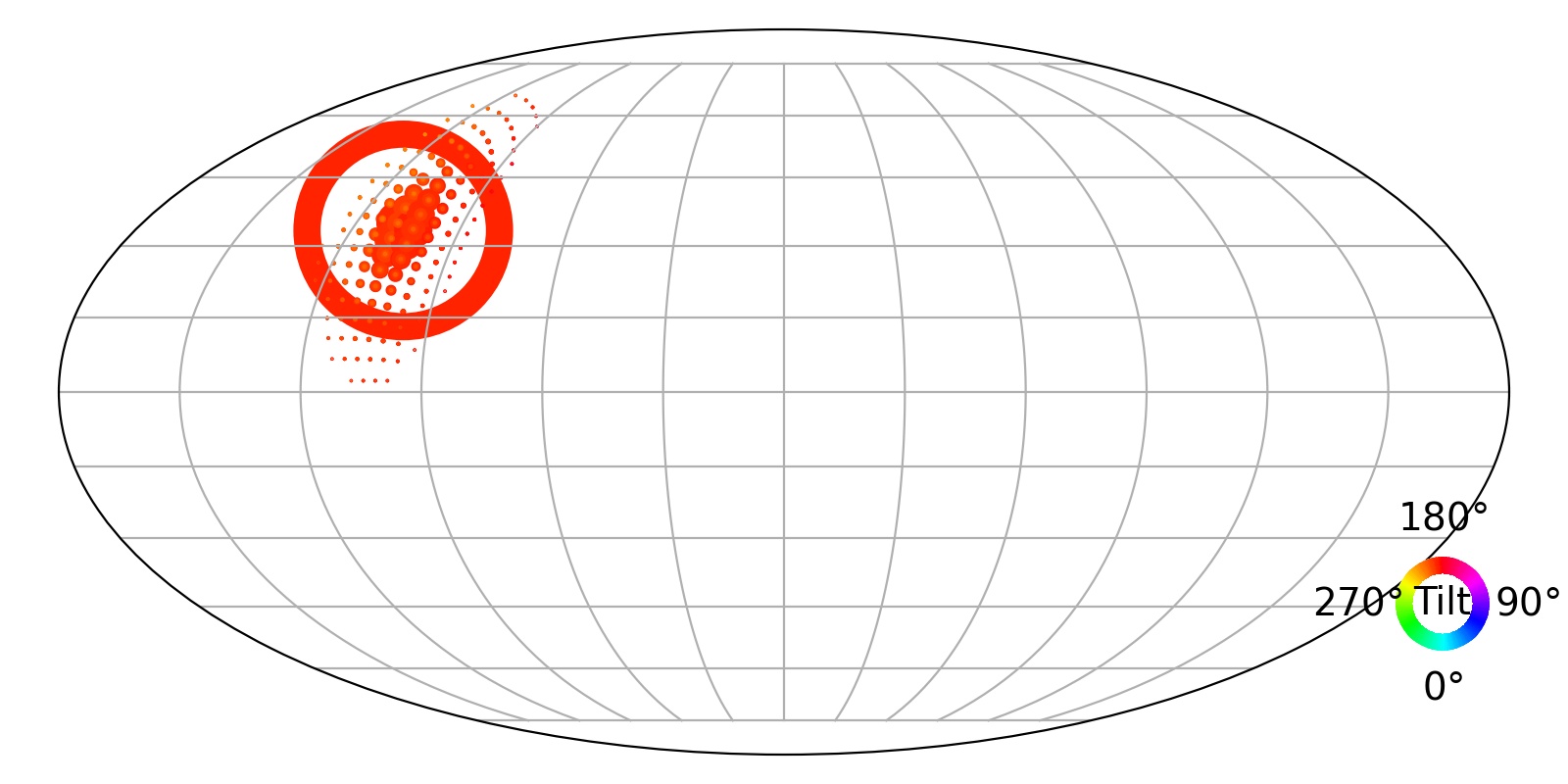}
    \hspace{4mm}
    &\includegraphics[height=1.6cm]{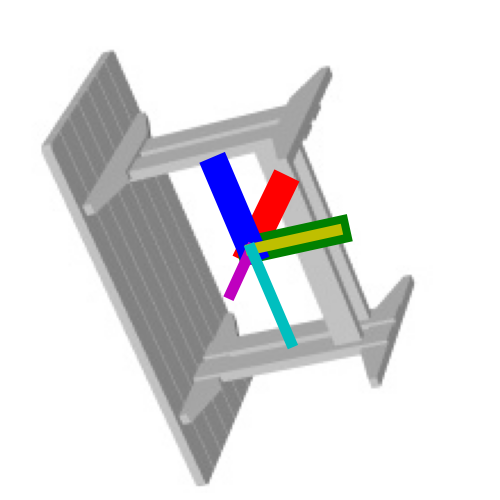}
    &\includegraphics[clip,trim=4.5cm 4cm 4.5cm 1.5cm, height=1.6cm]{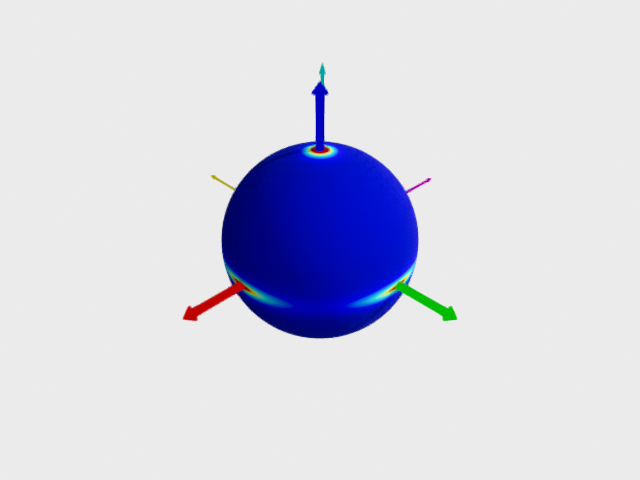}
    &\includegraphics[height=1.2cm]{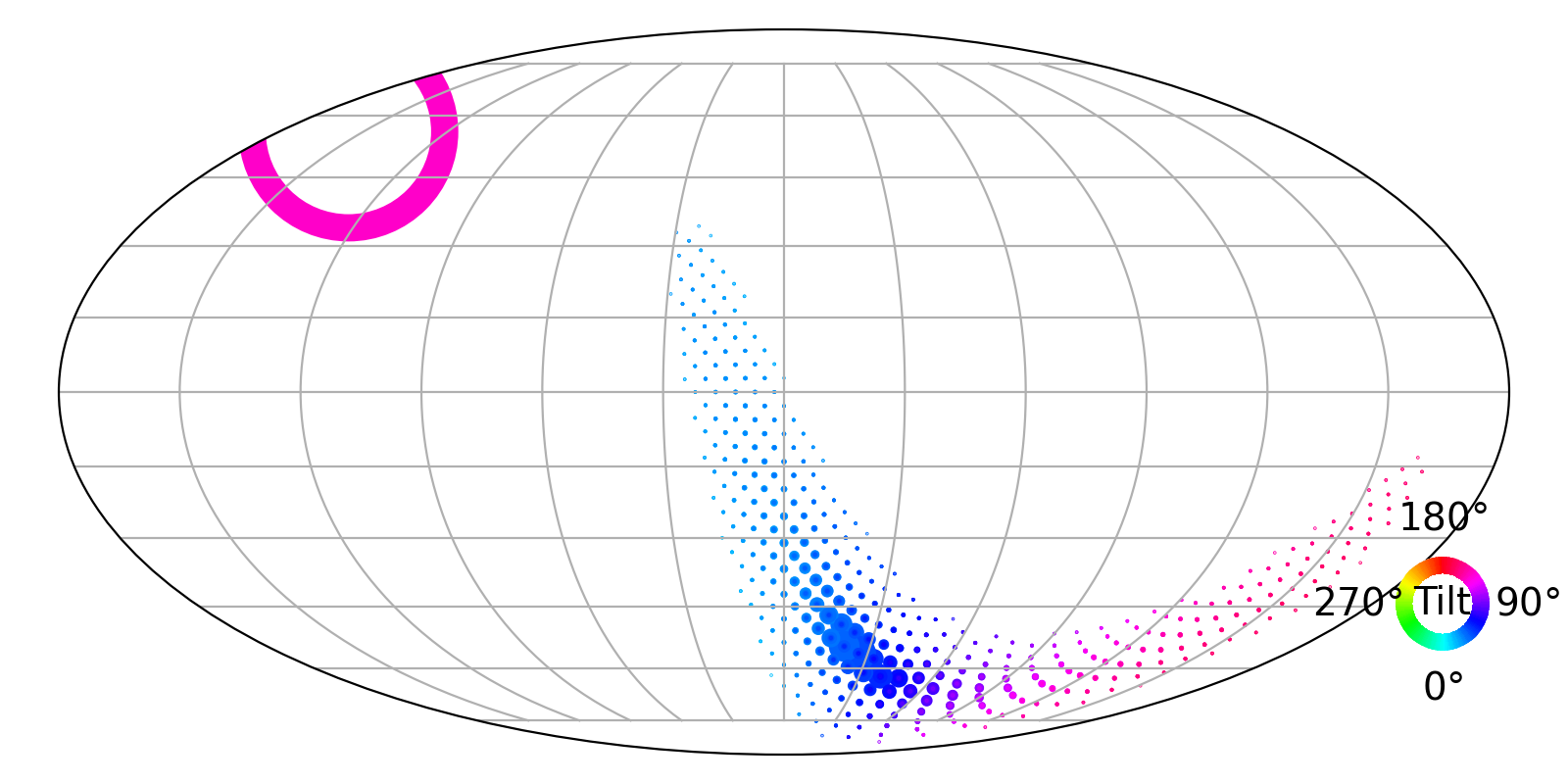}\\
    
    \small{Input image} & {\tiny\makecell{Distribution visual.\\\citep{mohlin2020probabilistic}}} & {\tiny\makecell{Distribution visual.\\\citep{murphy2021implicit}}}  \hspace{4mm} &
    \small{Input image} & {\tiny\makecell{Distribution visual.\\\citep{mohlin2020probabilistic}}} & {\tiny\makecell{Distribution visual.\\\citep{murphy2021implicit}}}
    \end{tabular}
    \vspace{-3mm}
    \caption{\small \ree{\textbf{Visual results on ModelNet10-SO3 dataset.} We adopt the distribution visualization methods in \cite{mohlin2020probabilistic} and \cite{murphy2021implicit}. For input images and visualizations with \cite{mohlin2020probabilistic}, predicted rotations are shown with thick lines and the ground truths are with thin lines. For visualizations with \cite{murphy2021implicit}, ground truths are shown by solid circles.}}
    \vspace{-2mm}
	\label{fig:vis_modelnet}
\end{figure}
\begin{figure}[t]
    \centering
    \begin{tabular}{cccccc}
    \includegraphics[height=1.6cm]{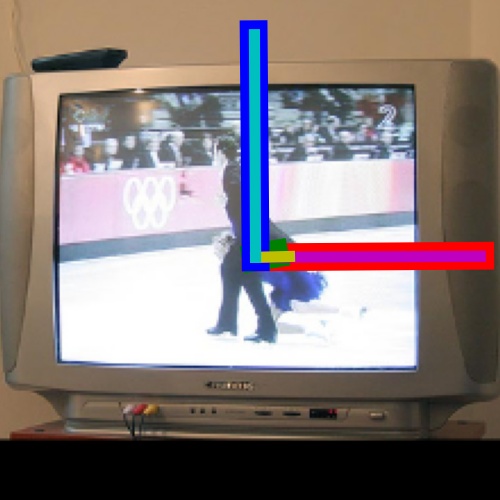}
    &\includegraphics[clip,trim=4.5cm 4cm 4.5cm 1.5cm, height=1.6cm]{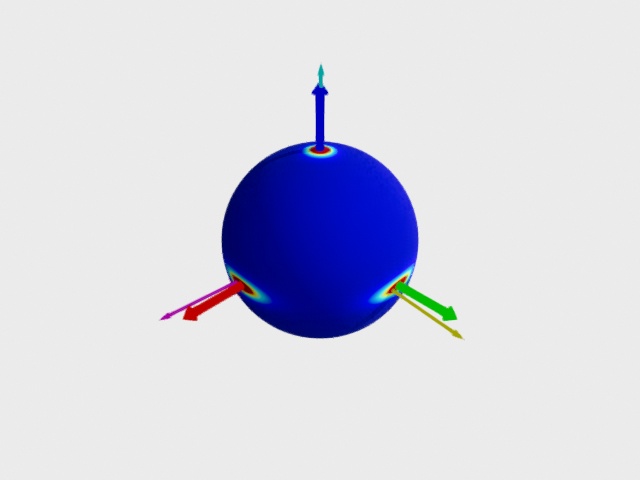}
    &\includegraphics[height=1.2cm]{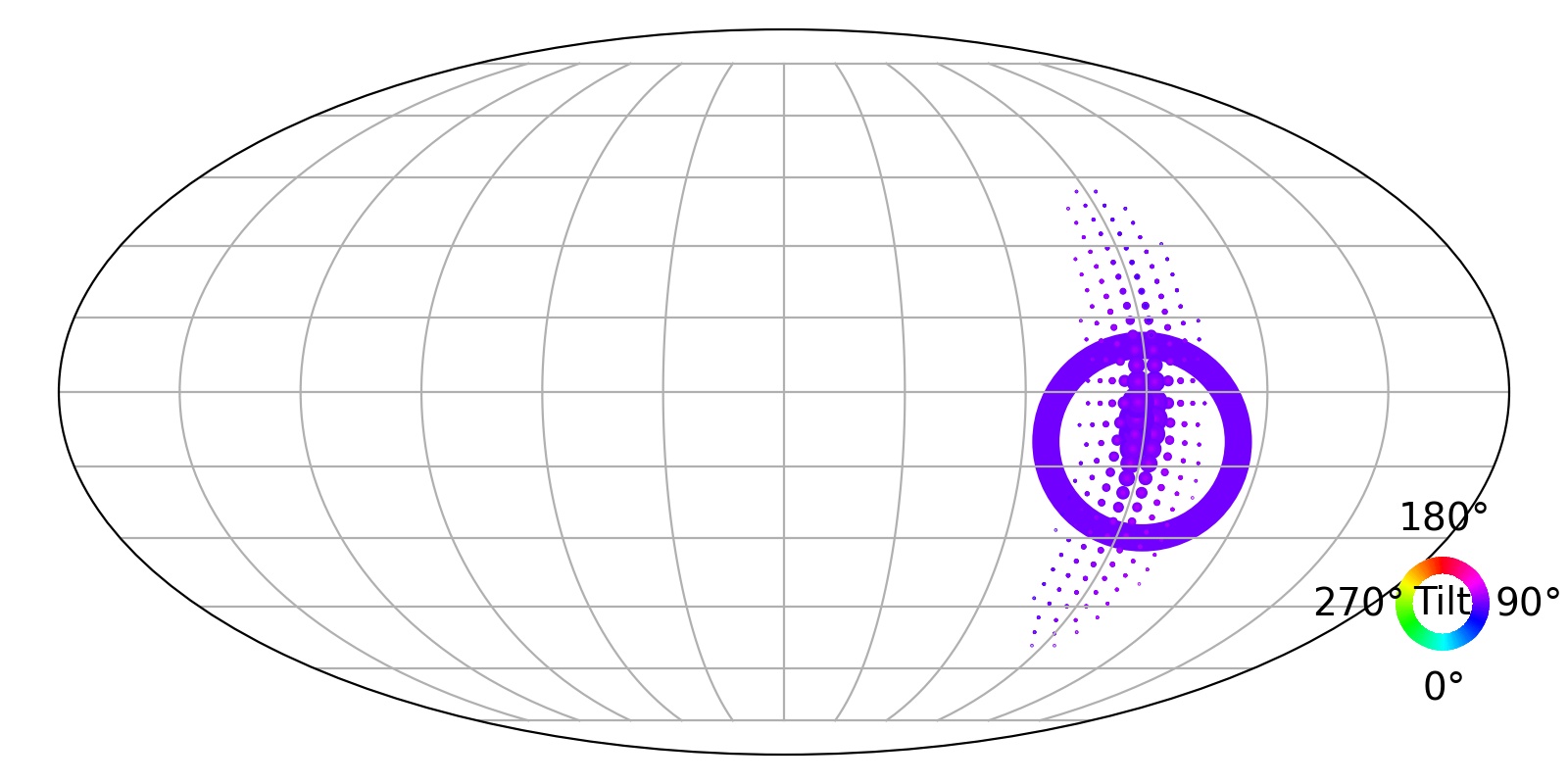}
    \hspace{4mm}
    &\includegraphics[height=1.6cm]{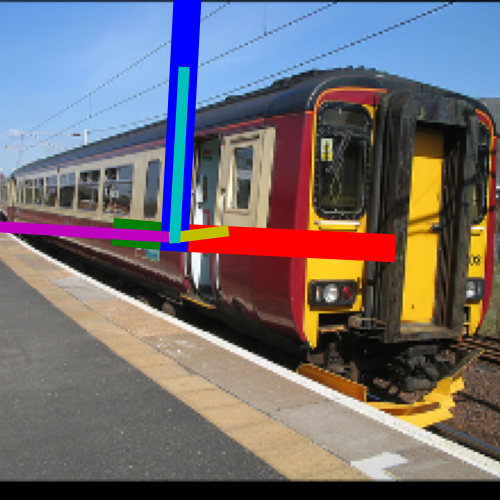}
    &\includegraphics[clip,trim=4.5cm 4cm 4.5cm 1.5cm, height=1.6cm]{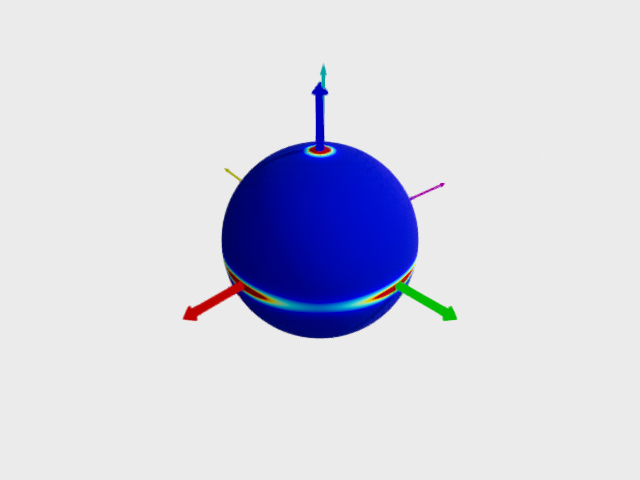}
    &\includegraphics[height=1.2cm]{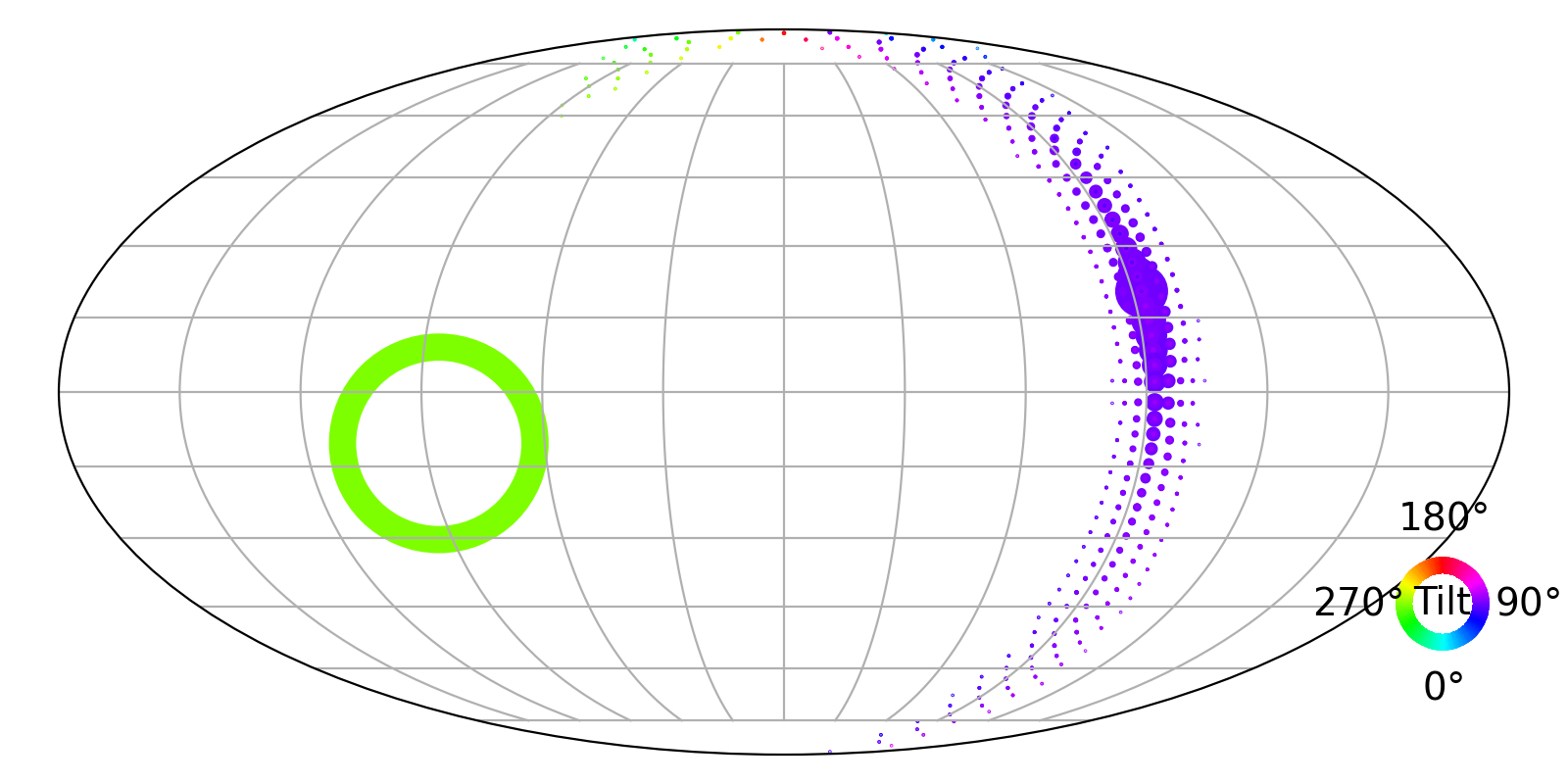}\\
    
    \includegraphics[height=1.6cm]{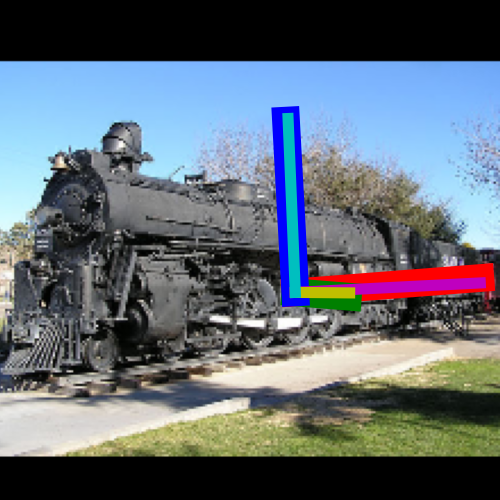}
    &\includegraphics[clip,trim=4.5cm 4cm 4.5cm 1.5cm, height=1.6cm]{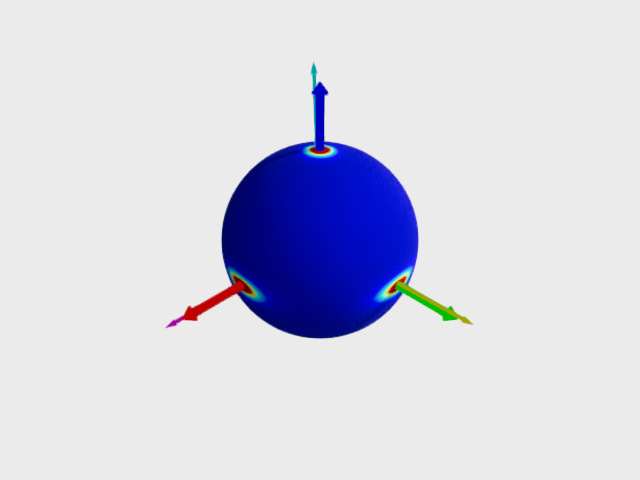}
    &\includegraphics[height=1.2cm]{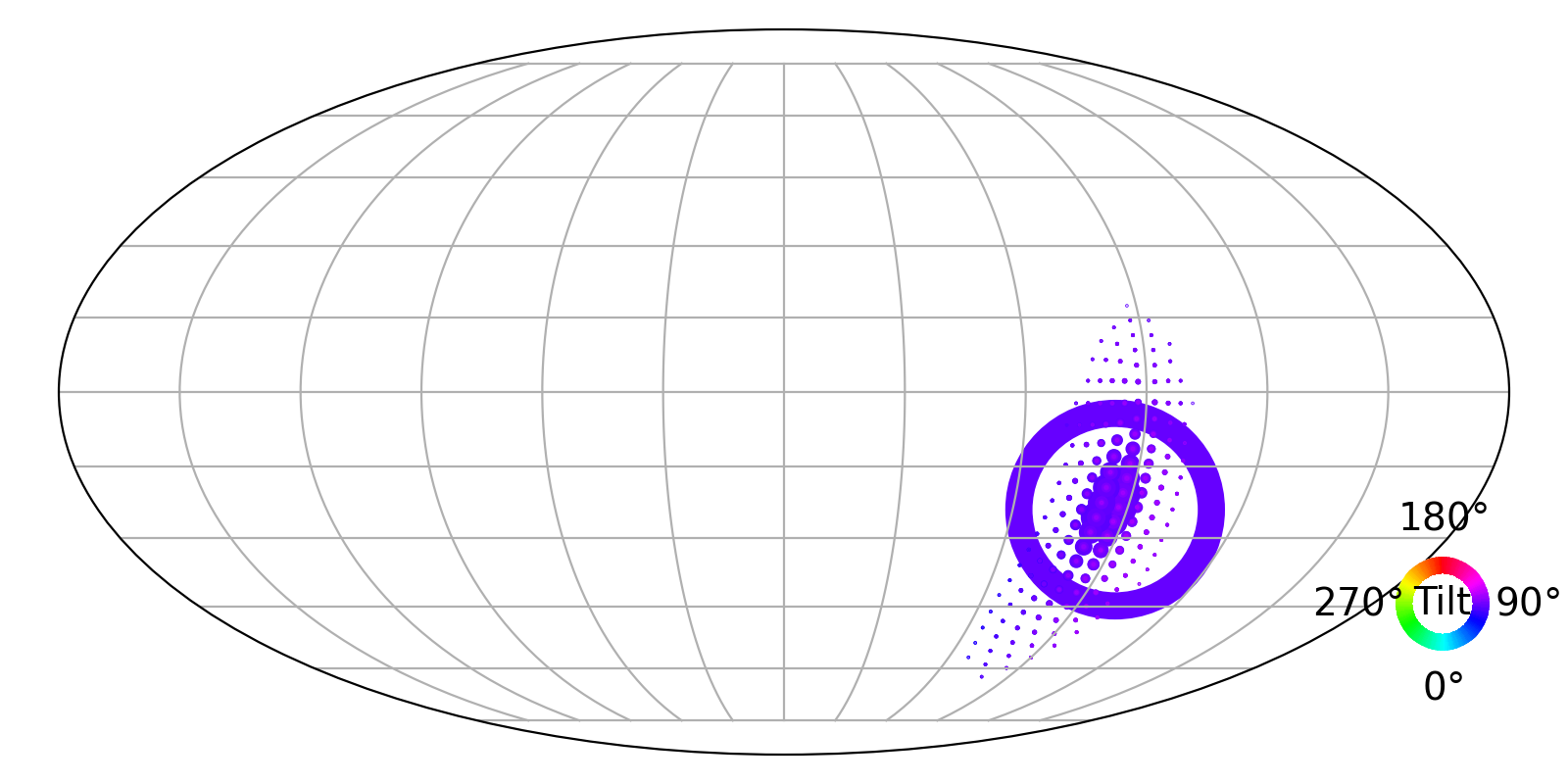}
    \hspace{4mm}
    &\includegraphics[height=1.6cm]{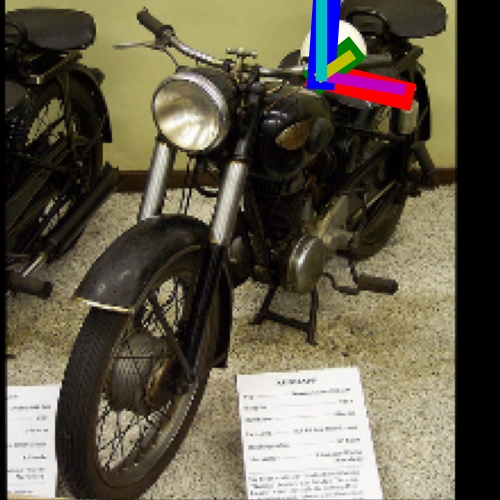}
    &\includegraphics[clip,trim=4.5cm 4cm 4.5cm 1.5cm, height=1.6cm]{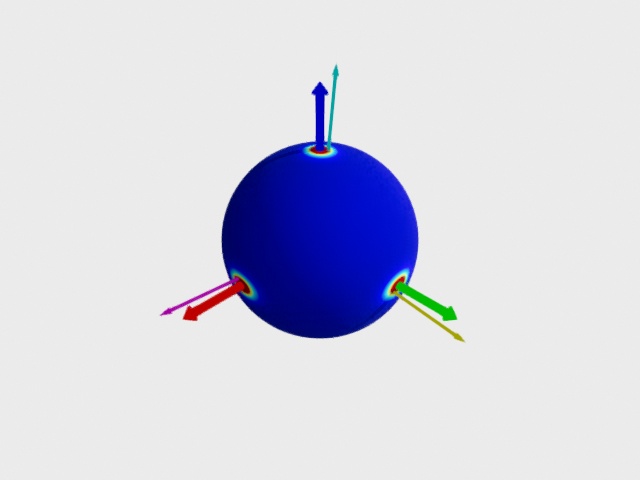}
    &\includegraphics[height=1.2cm]{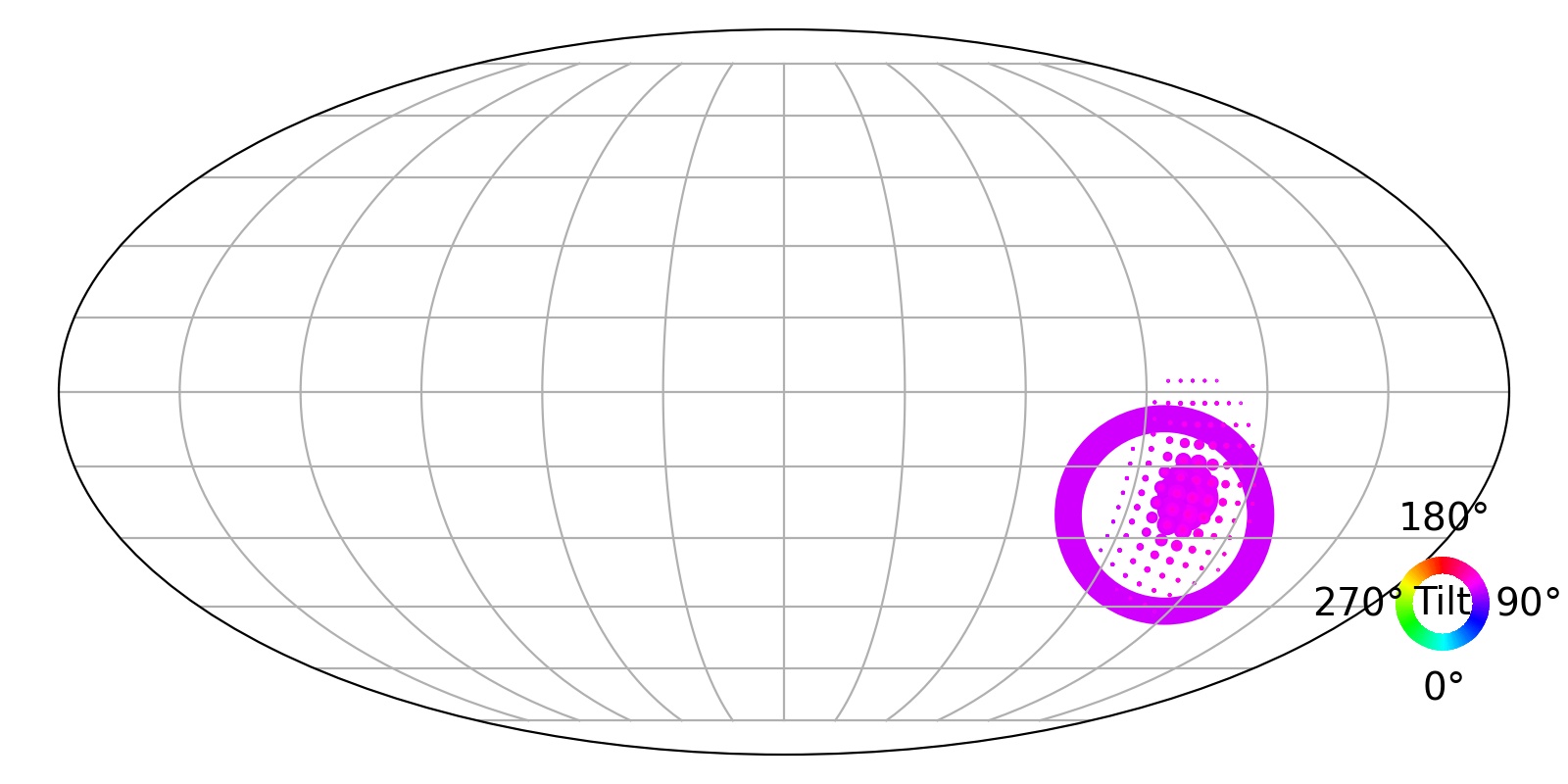}\\
    
    \includegraphics[height=1.6cm]{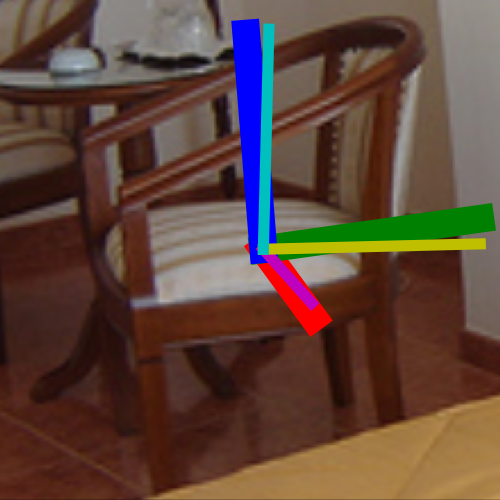}
    &\includegraphics[clip,trim=4.5cm 4cm 4.5cm 1.5cm, height=1.6cm]{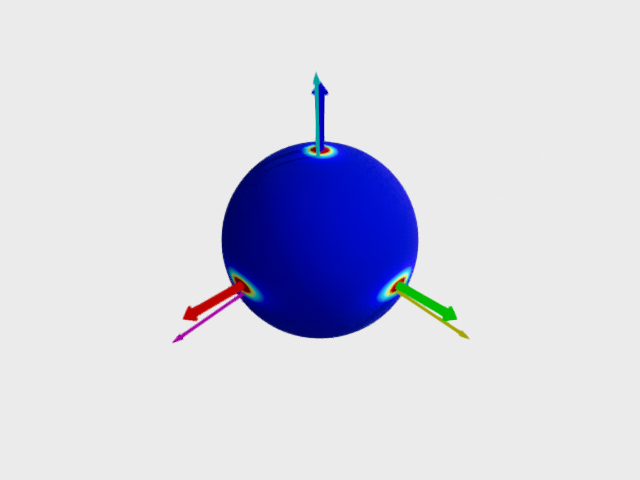}
    &\includegraphics[height=1.2cm]{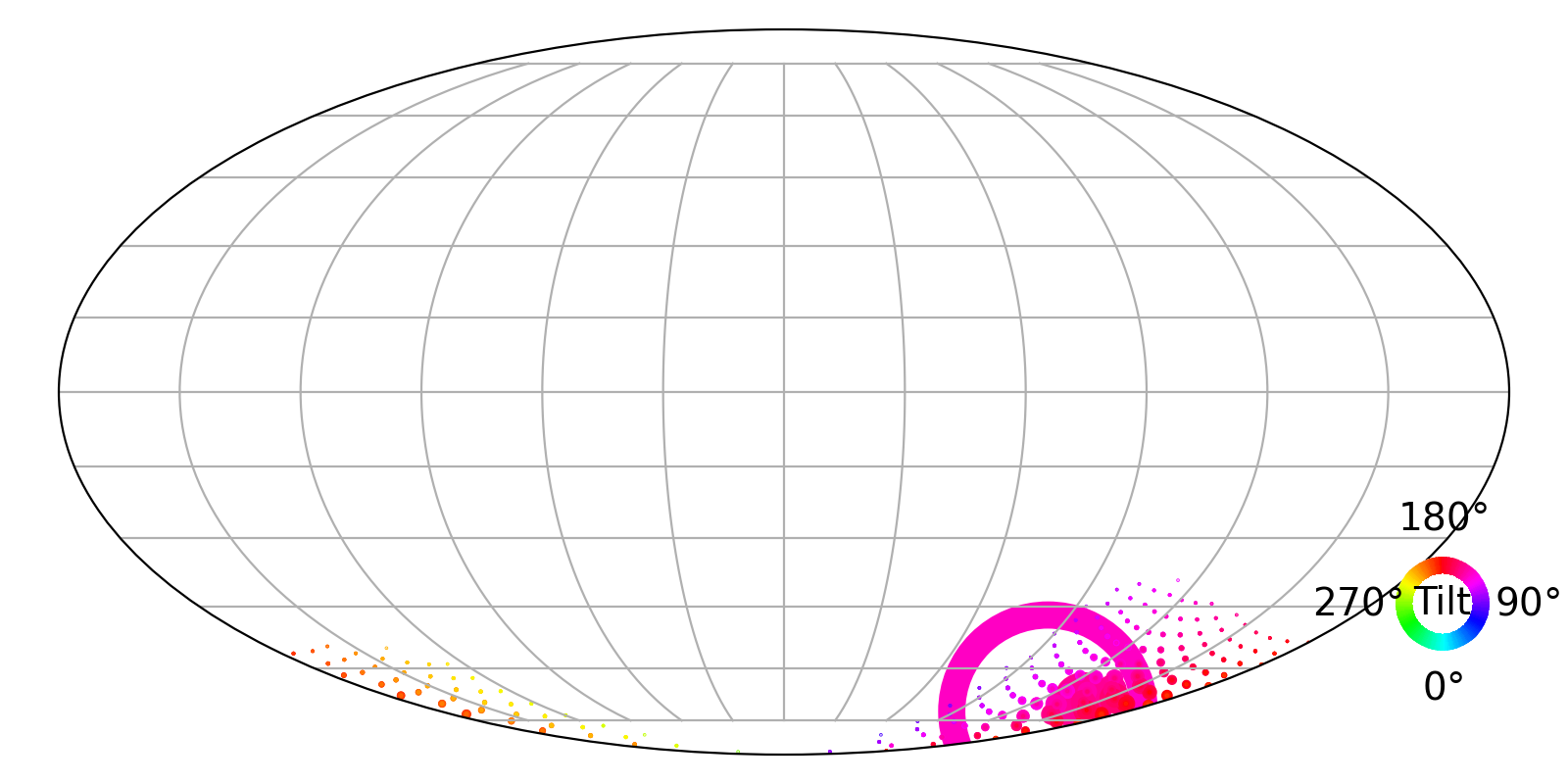}
    \hspace{4mm}
    &\includegraphics[height=1.6cm]{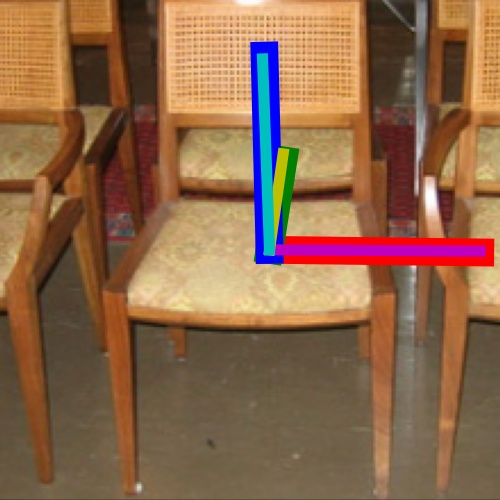}
    &\includegraphics[clip,trim=4.5cm 4cm 4.5cm 1.5cm, height=1.6cm]{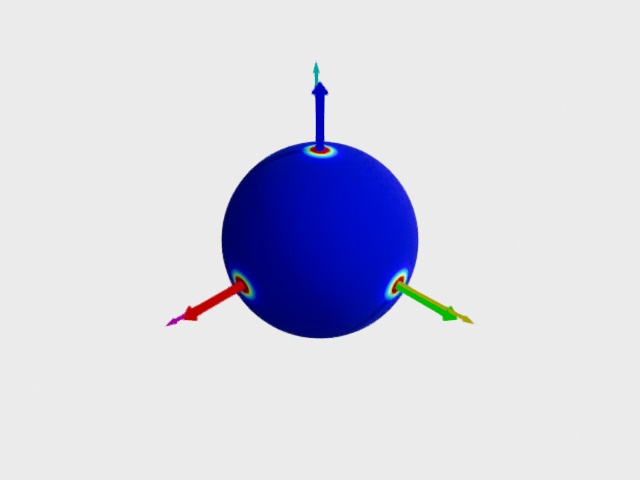}
    &\includegraphics[height=1.2cm]{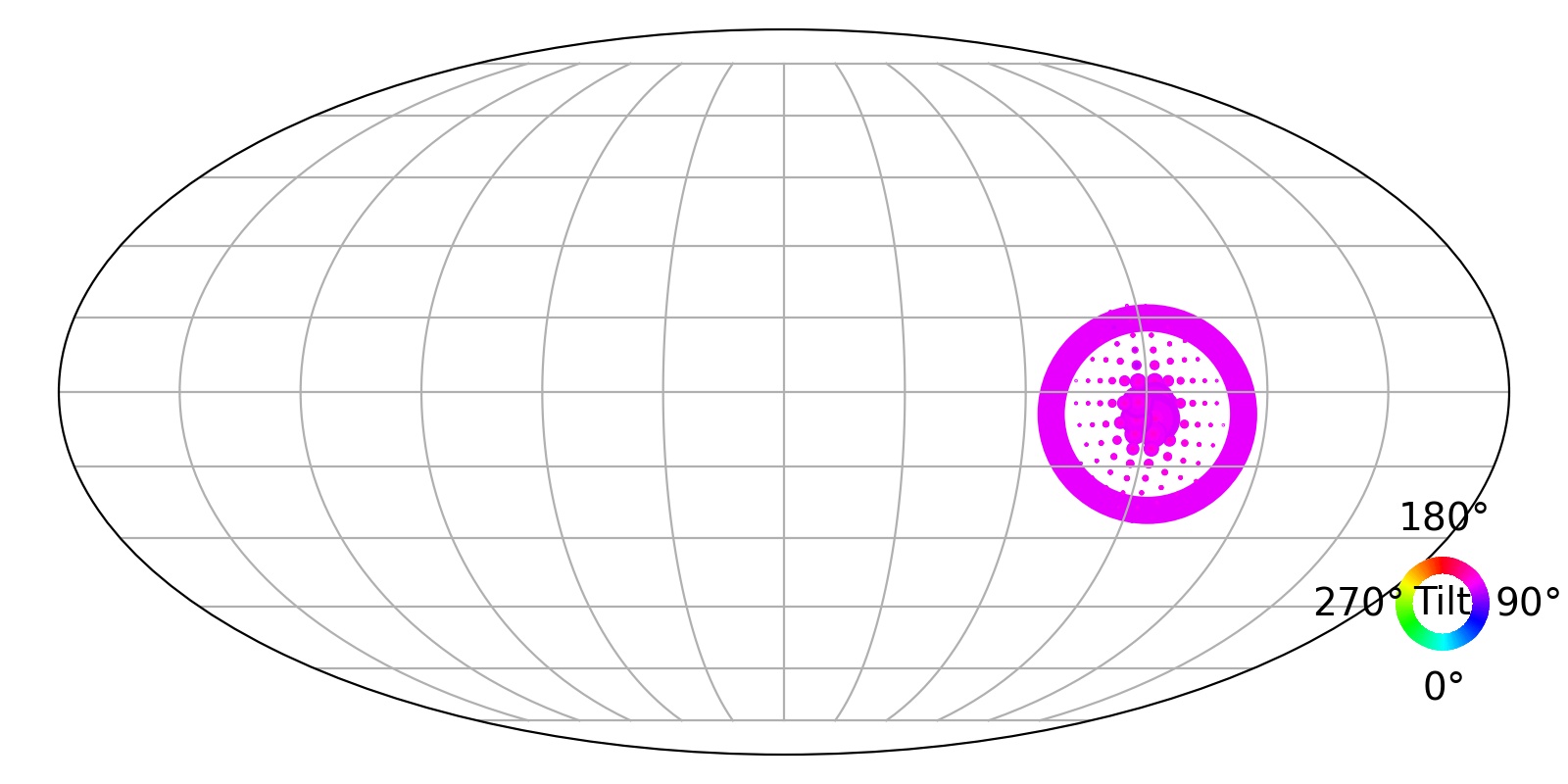}\\
    
    \includegraphics[height=1.6cm]{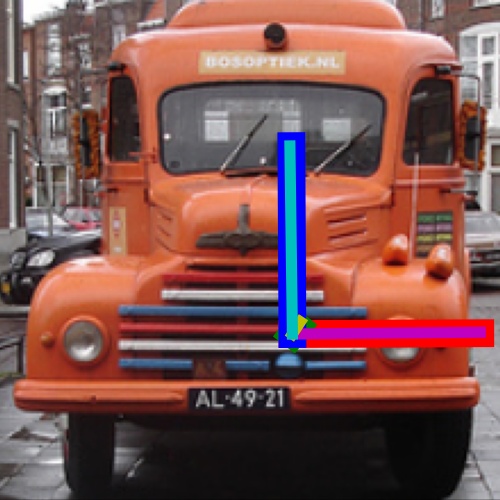}
    &\includegraphics[clip,trim=4.5cm 4cm 4.5cm 1.5cm, height=1.6cm]{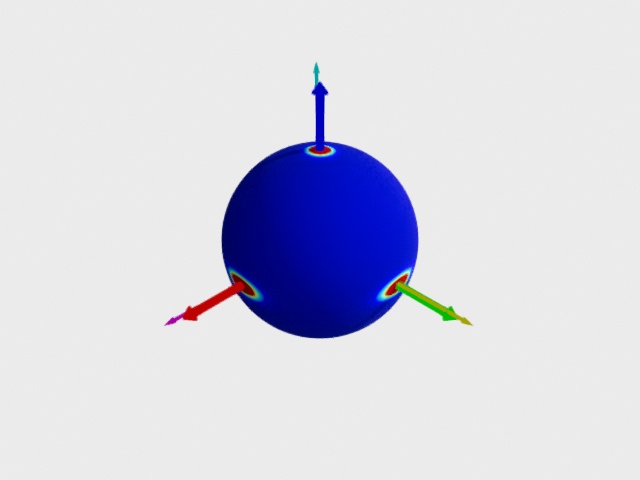}
    &\includegraphics[height=1.2cm]{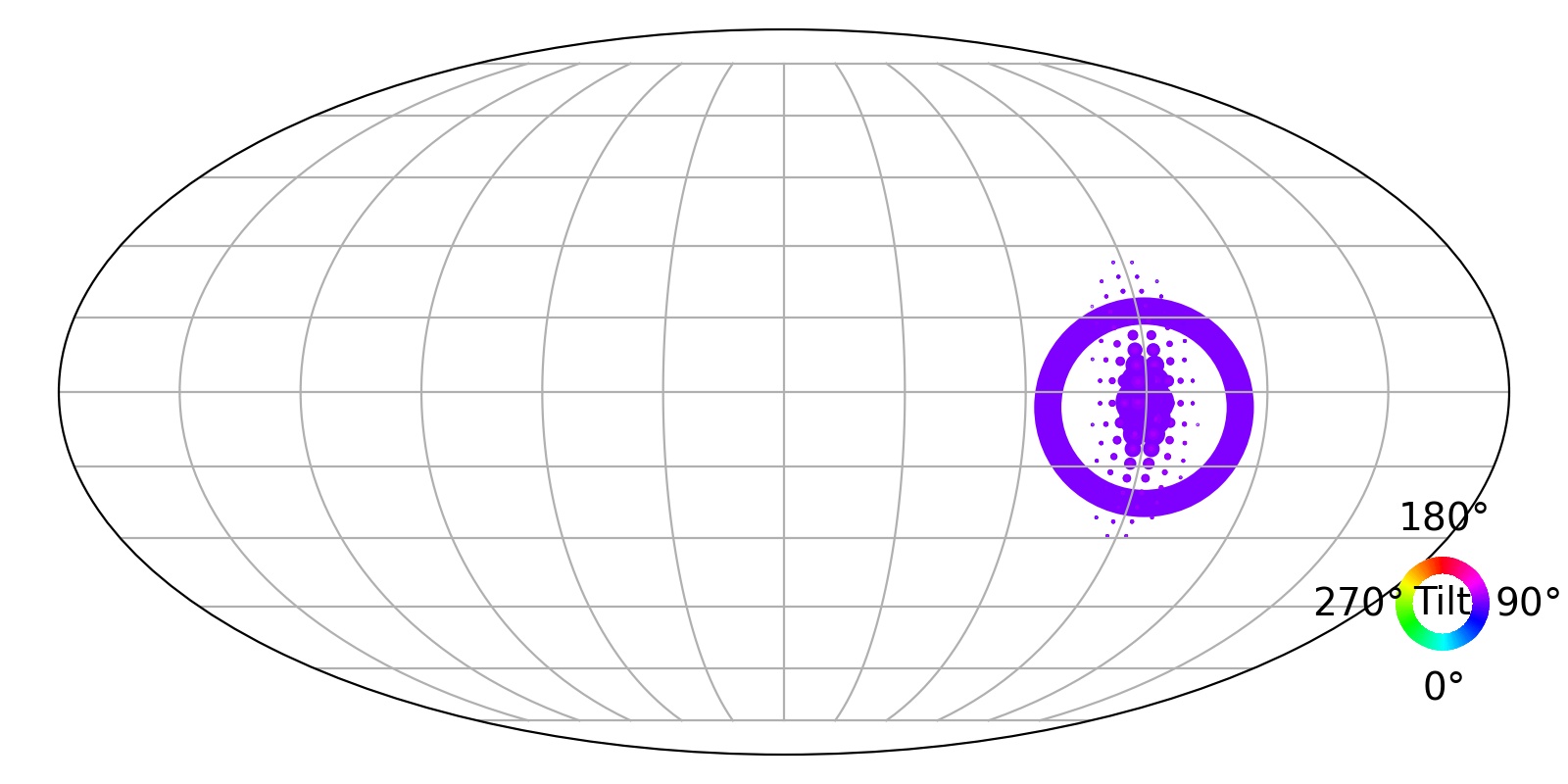}
    \hspace{4mm}
    &\includegraphics[height=1.6cm]{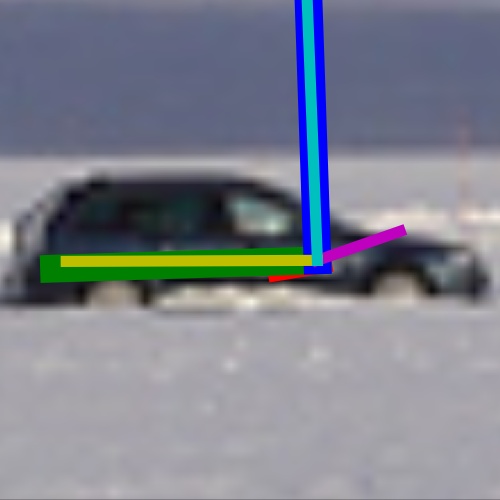}
    &\includegraphics[clip,trim=4.5cm 4cm 4.5cm 1.5cm, height=1.6cm]{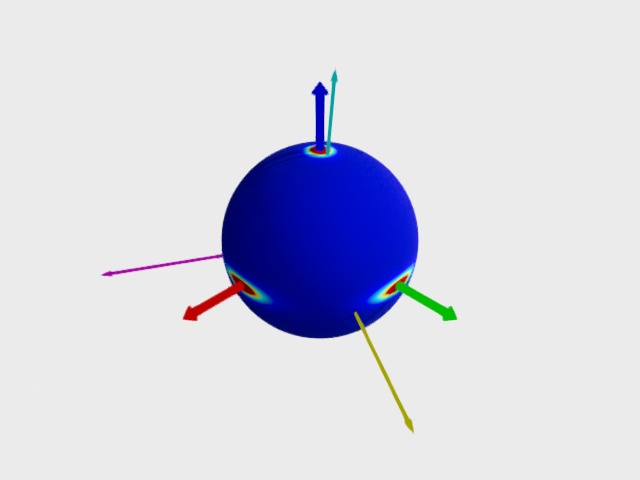}
    &\includegraphics[height=1.2cm]{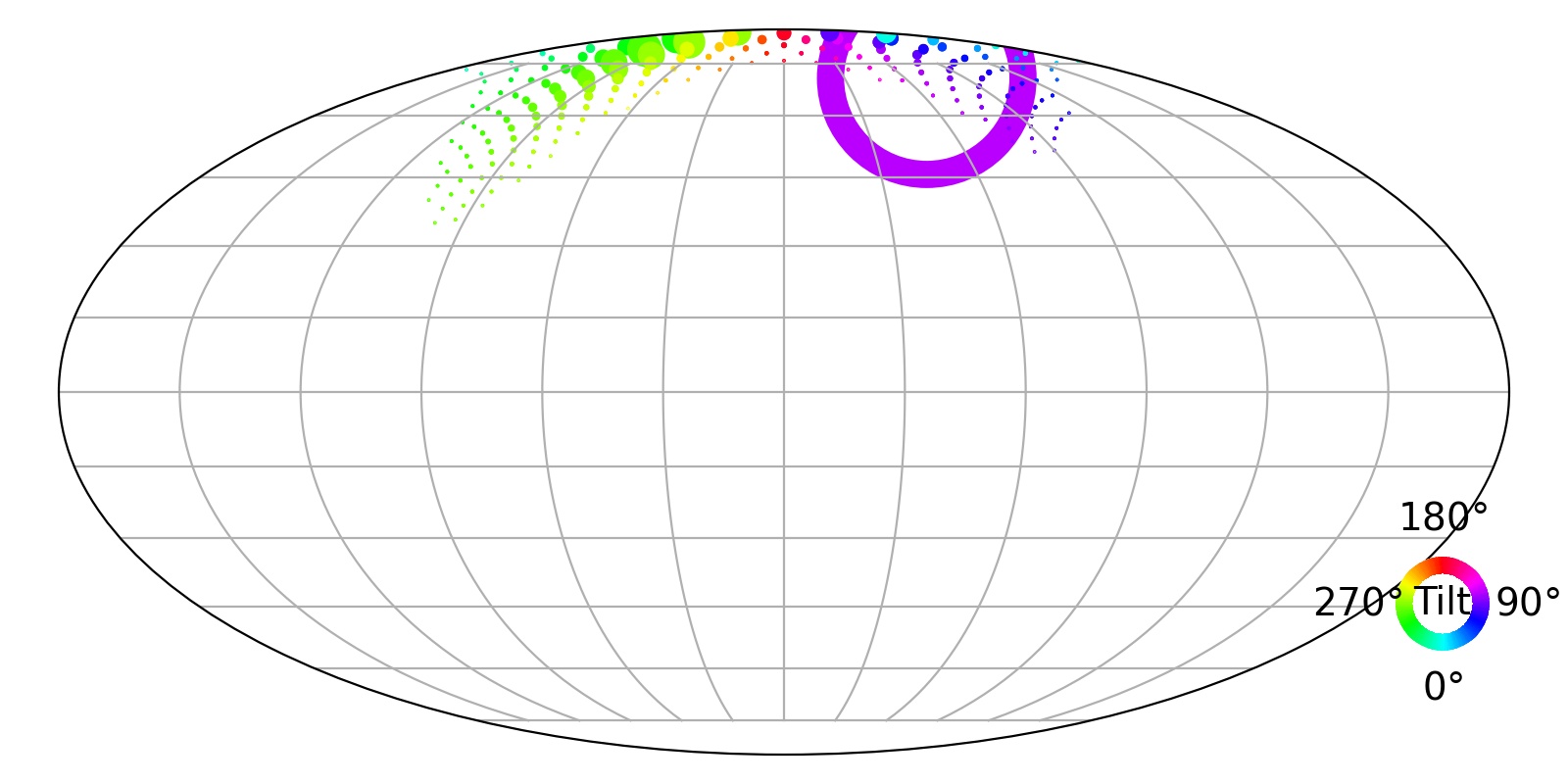}\\
    
    \includegraphics[height=1.6cm]{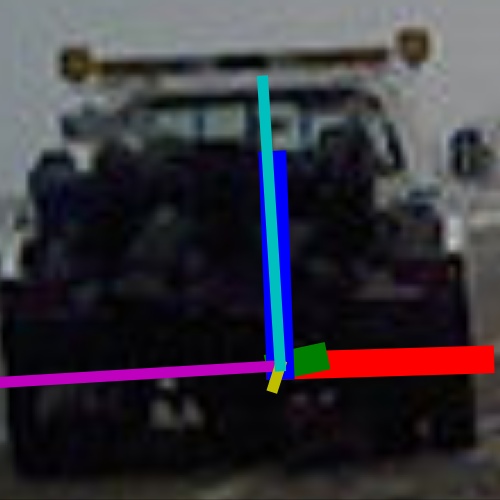}
    &\includegraphics[clip,trim=4.5cm 4cm 4.5cm 1.5cm, height=1.6cm]{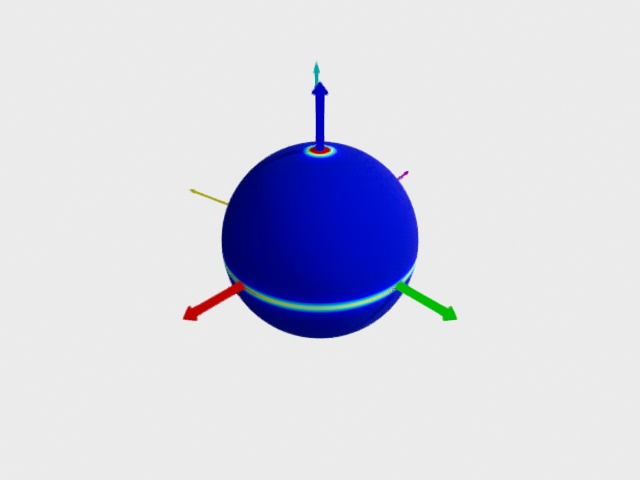}
    &\includegraphics[height=1.2cm]{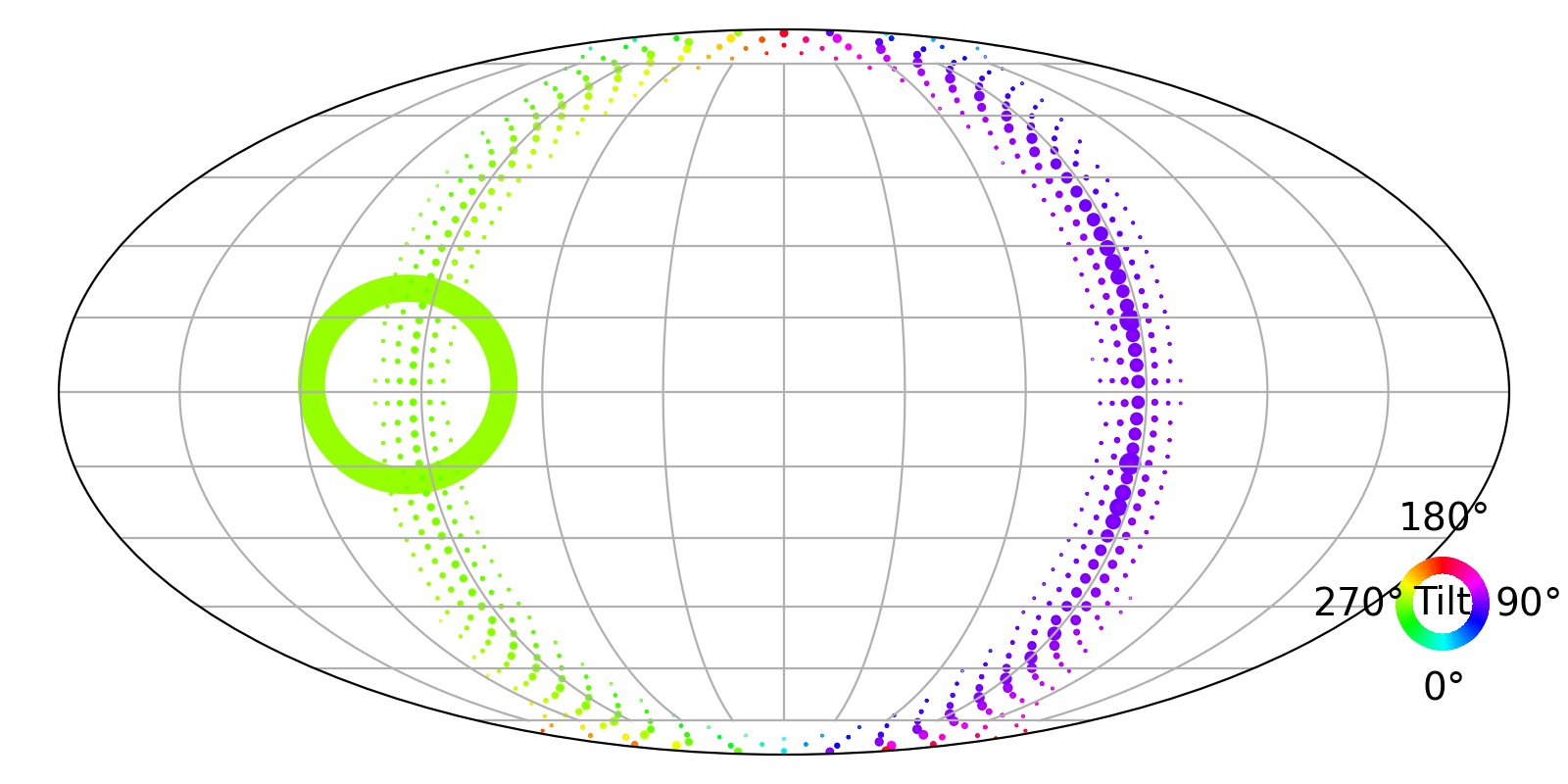}
    \hspace{4mm}
    &\includegraphics[height=1.6cm]{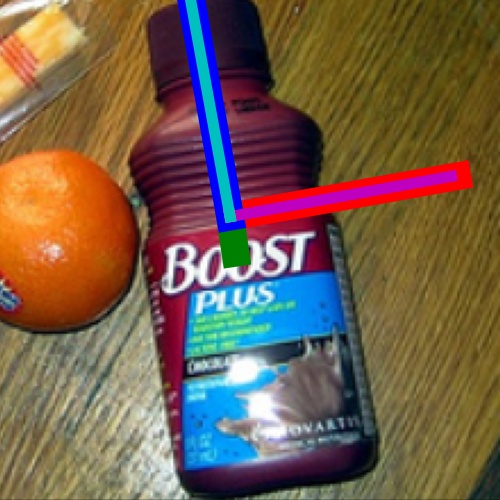}
    &\includegraphics[clip,trim=4.5cm 4cm 4.5cm 1.5cm, height=1.6cm]{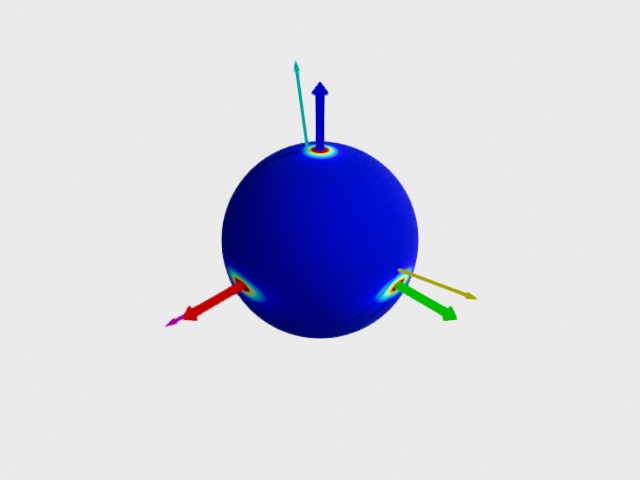}
    &\includegraphics[height=1.2cm]{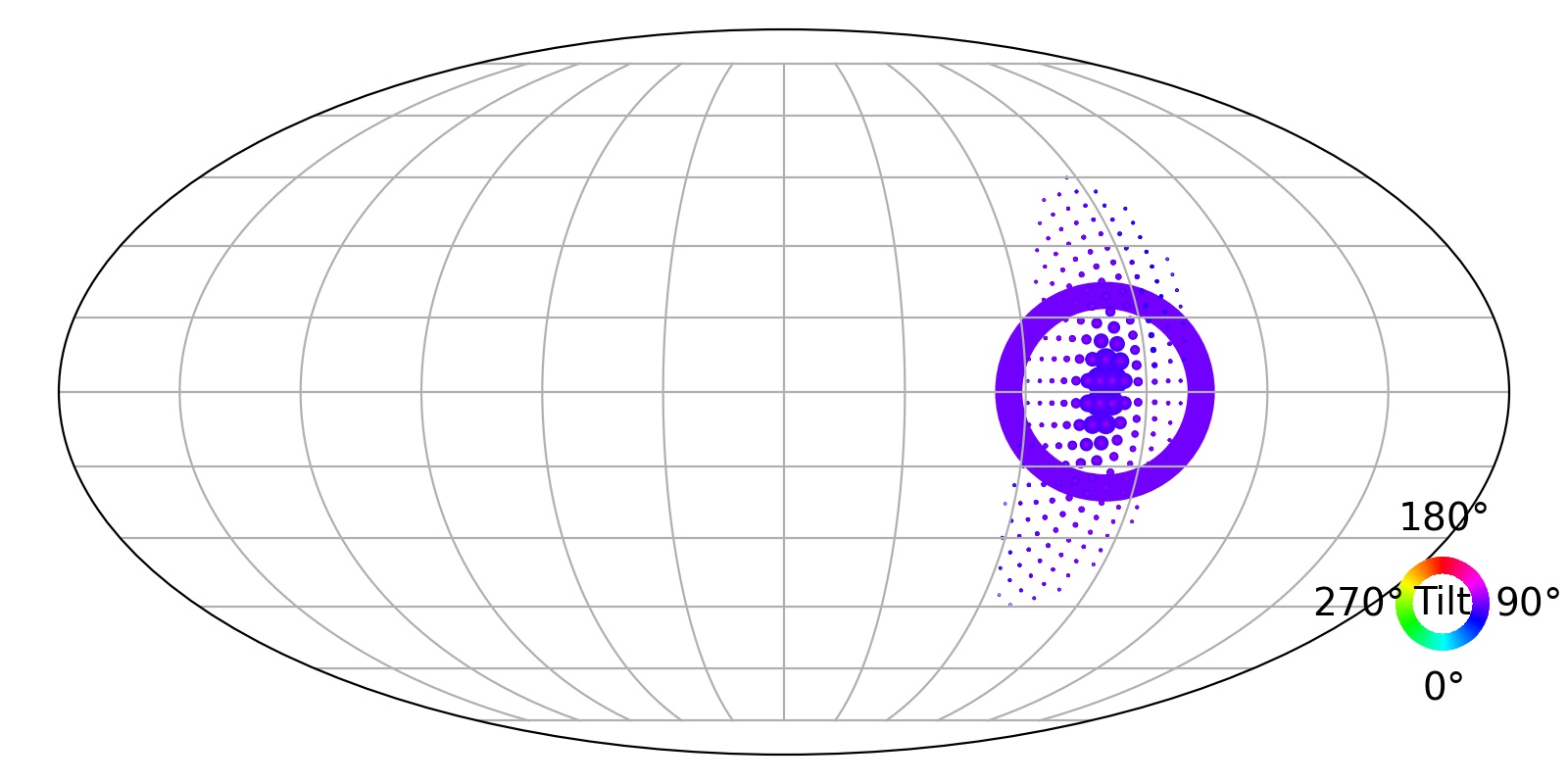}\\
    
    \includegraphics[height=1.6cm]{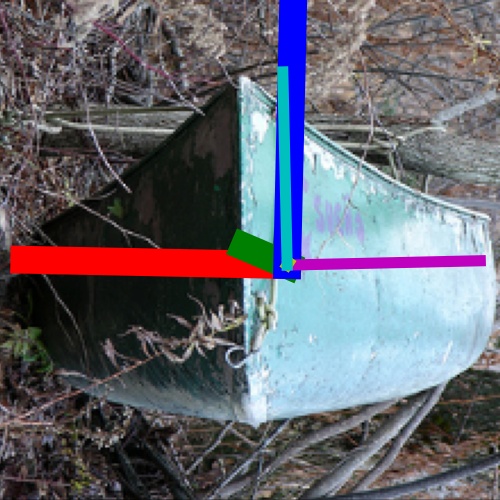}
    &\includegraphics[clip,trim=4.5cm 4cm 4.5cm 1.5cm, height=1.6cm]{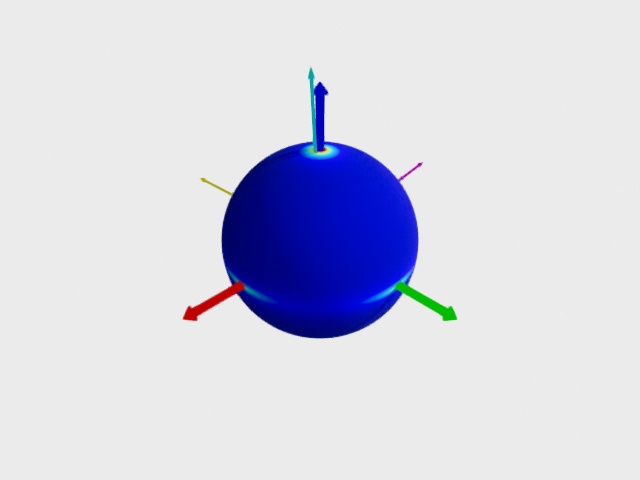}
    &\includegraphics[height=1.2cm]{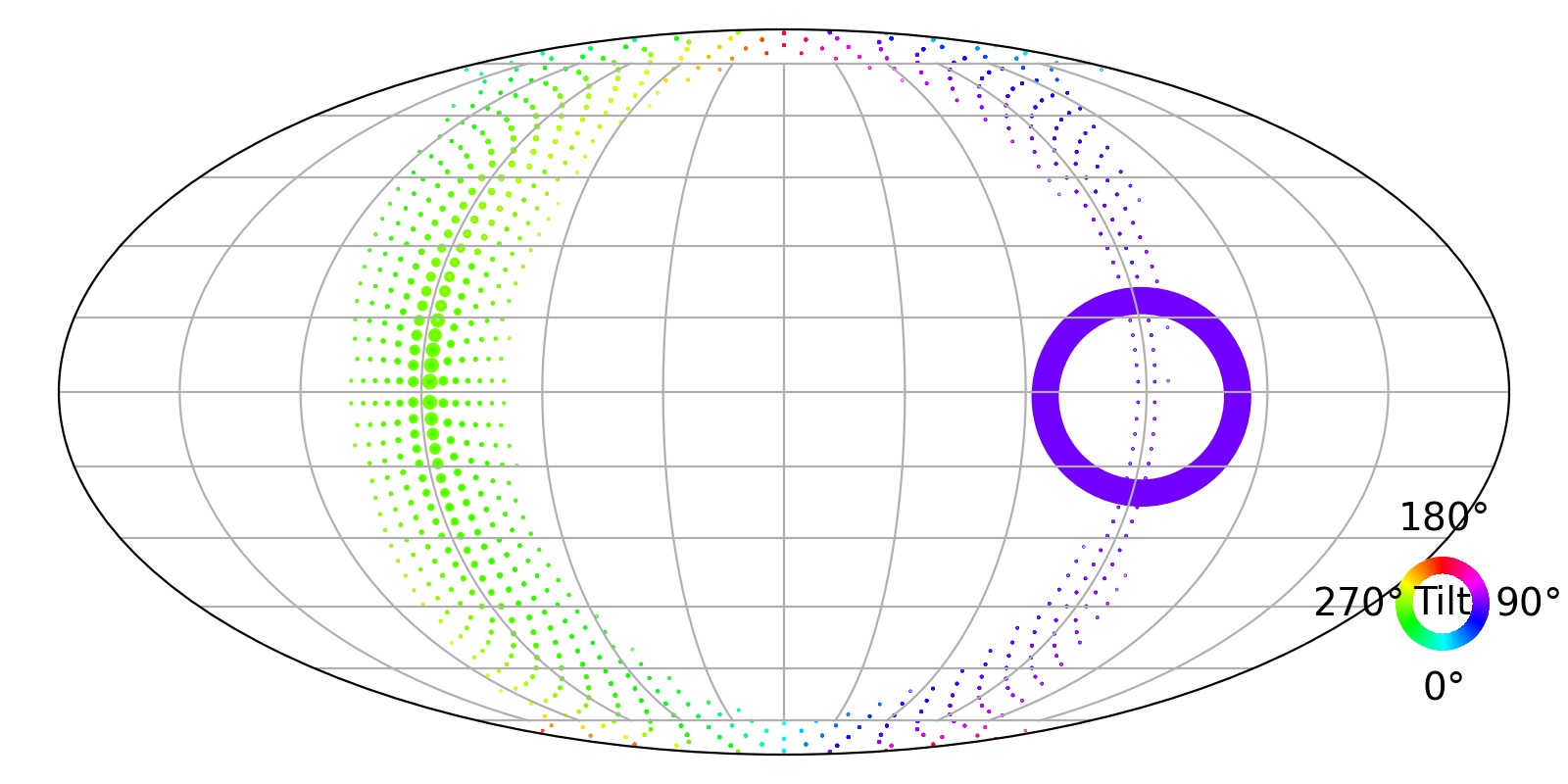}
    \hspace{4mm}
    &\includegraphics[height=1.6cm]{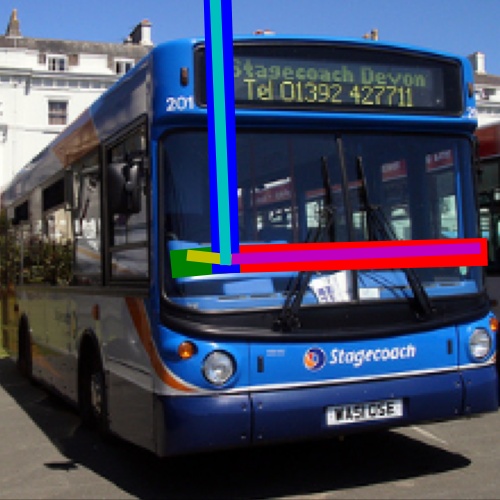}
    &\includegraphics[clip,trim=4.5cm 4cm 4.5cm 1.5cm, height=1.6cm]{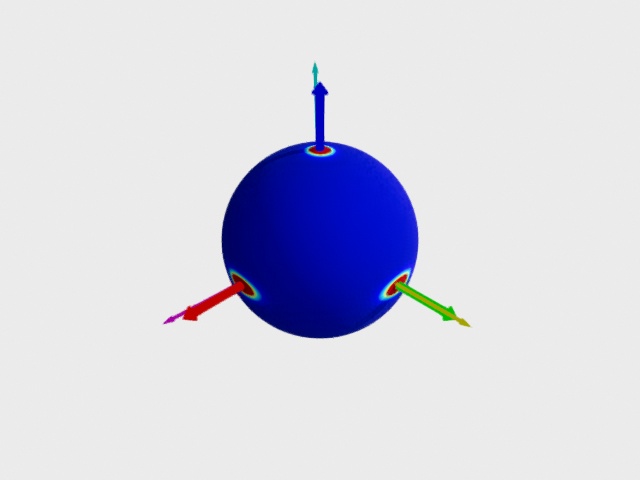}
    &\includegraphics[height=1.2cm]{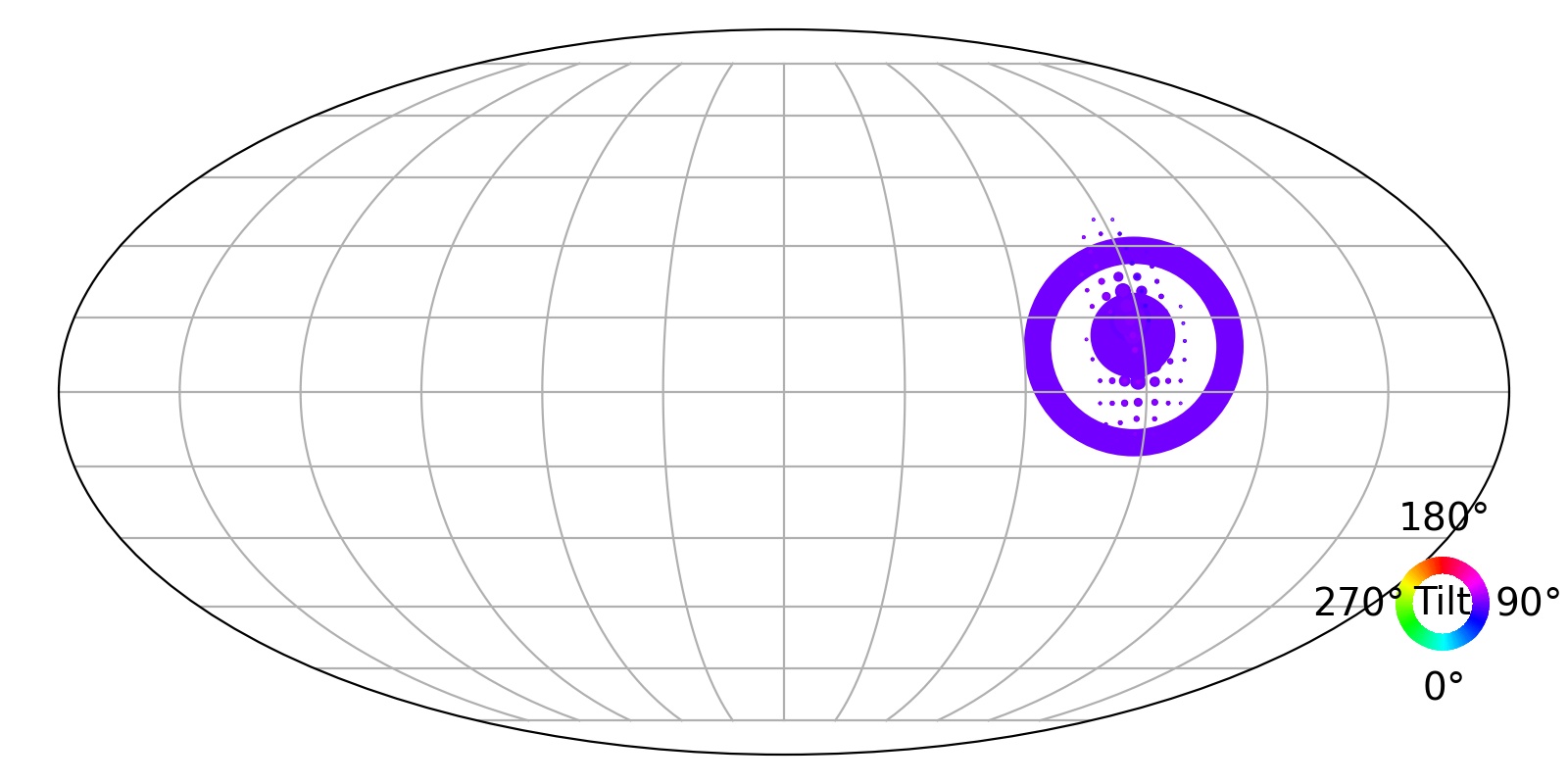}\\
    
    \small{Input image} & {\tiny\makecell{Distribution visual.\\\citep{mohlin2020probabilistic}}} & {\tiny\makecell{Distribution visual.\\\citep{murphy2021implicit}}}  \hspace{4mm} &
    \small{Input image} & {\tiny\makecell{Distribution visual.\\\citep{mohlin2020probabilistic}}} & {\tiny\makecell{Distribution visual.\\\citep{murphy2021implicit}}}
    \end{tabular}
    \vspace{-3mm}
    \caption{\small \ree{\textbf{Visual results on Pascal3D+ dataset.} We adopt the distribution visualization methods in \cite{mohlin2020probabilistic} and \cite{murphy2021implicit}. For input images and visualizations with \cite{mohlin2020probabilistic}, predicted rotations are shown with thick lines and the ground truths are with thin lines. For visualizations with \cite{murphy2021implicit}, ground truths are shown by solid circles.}}
    \vspace{-2mm}
	\label{fig:vis_pascal}
\end{figure}

}

\section{Derivations}
\label{sec:supp_proof}

\begin{prop1}
    Let $\boldsymbol{\Phi} = \log \mathbf{\widetilde{R}} \in \mathfrak{so}(3)$ and $\boldsymbol{\phi} = {\boldsymbol{\Phi}^\vee} \in \mathbb{R}^3$. For rotation matrix $\mathbf{R} \in \SO$ following \emph{matrix Fisher distribution}, when 
    \ree{$\|\mathbf{R} - \mathbf{R}_0 \| \rightarrow 0$}
    , $\boldsymbol{\phi}$ follows zero-mean \emph{multivariate Gaussian distribution}.
\end{prop1}
\begin{proof}
For $\mathbf{R}\sim \mathcal{MF}(\mathbf{A})$, we have
\begin{equation}
\footnotesize
\begin{aligned}
    p(\mathbf{R})\mathrm{d}\mathbf{R} &\propto \exp\left(\tr{\mathbf{A}^T\mathbf{R}}\right)\mathrm{d}\mathbf{R}   
    = \exp\left(\tr{\mathbf{S}\mathbf{V}^T\mathbf{\widetilde{R}}\mathbf{V}}\right)\mathrm{d}\mathbf{\widetilde{R}}
\end{aligned}
\end{equation}
Considering Eq. \ref{eq:vtrv} in the main paper, we have
\begin{equation}
\label{eq:tr}
\footnotesize
\begin{aligned}
    \tr{\mathbf{S}\mathbf{V}^T\mathbf{\widetilde{R}}\mathbf{V}} 
    &= \tr{\mathbf{S}} + \sum_{(i,j,k)\in I}-\frac{1}{2}(s_j+s_k)\mu_i^2+O({\len {\boldsymbol{\phi}}}^3) \\
    &=\tr{\mathbf{S}} -\frac{1}{2}\boldsymbol{\phi}^T\mathbf{V}
    \left[\begin{smallmatrix}
        s_2 + s_3 &  &  \\
        & s_1 + s_3 &  \\
        &  & s_1 + s_2
        \end{smallmatrix}\right]
    \mathbf{V}^T\boldsymbol{\phi}
\end{aligned}
\end{equation}
Thus
\begin{equation}
\footnotesize
\label{eq:gauss}
\begin{aligned}
    p(\mathbf{R})\mathrm{d}\mathbf{R} &\propto \exp\left(\tr{\mathbf{A}^T\mathbf{R}}\right)\mathrm{d}\mathbf{R}\\  
    &= \frac{\exp(\tr{\mathbf{S}})}{8\pi^2}\exp\left(-\frac{1}{2}\boldsymbol{\phi}^T\boldsymbol{\Sigma} ^{-1}\boldsymbol{\phi}\right) \left(1 + O({\len {\boldsymbol{\phi}}}^2)\right)\mathrm{d}\boldsymbol{\phi}
\end{aligned}
\end{equation}
When 
\ree{$\|\mathbf{R} - \mathbf{R}_0\| \rightarrow 0$}
, we have 
\ree{$\|\mathbf{\widetilde{R}} - \mathbf{I}\|  \rightarrow 0$ }
and $\boldsymbol{\phi} \rightarrow \mathbf{0}$, so Eq. \ref{eq:gauss} follows the multivariate Gaussian distribution 
with the covariance matrix as $\boldsymbol{\Sigma}$, where $\boldsymbol{\Sigma} = \mathbf{V}\operatorname{diag}(\frac{1}{s_2+s_3},\frac{1}{s_1+s_3},\frac{1}{s_1+s_2})\mathbf{V}^T$.
\end{proof}
\begin{prop3}
    Denote $\mathbf{q}_0$ as the mode of Quaternion Laplace distribution. Let $\pi$ be the tangent space of $\mathbb{S}^3$  at $\mathbf{q}_0$, and $\pi(\mathbf{x}) \in \mathbb{R}^4$ be the projection of $\mathbf{x} \in \mathbb{R}^4$ on $\pi$.
    For quaternion $\mathbf{q} \in \mathbb{S}^3$ following \emph{Bingham distribution} / \emph{Quaternion Laplace distribution}, when $\mathbf{q}\rightarrow\mathbf{q}_0$, $\pi(\mathbf{q})$ follows zero-mean \emph{multivariate Gaussian distribution} / zero-mean \emph{multivariate Laplace distribution}.
\end{prop3}
\begin{proof}
Denote $\mathbf{q_I}=(1,0,0,0)^T$ as the identity quaternion. 
Define $\mathbf{M}$ as an orthogonal matrix such that $\mathbf{M}^T\mathbf{q}_0=\mathbf{q_I}$.
Given $\pi(\mathbf{q}) = \mathbf{q}-(\mathbf{q}\cdot \mathbf{q}_0)\mathbf{q}_0$,
we have
\begin{equation}
\footnotesize
\begin{aligned}
    \mathbf{M}^T\pi(\mathbf{q}) &= \mathbf{M}^T\mathbf{q} - ((\mathbf{M}^T\mathbf{q})\cdot(\mathbf{M}^T\mathbf{q}_0))\mathbf{q_I} 
    = \mathbf{M}^T\mathbf{q} - w\mathbf{q_I},
\end{aligned}
\end{equation}
where $\mathbf{M}^T\mathbf{q}=(w,x,y,z)^T$.
Let $(\mathbf{e}_0,\mathbf{e}_1,\mathbf{e}_2,\mathbf{e}_3)$ be the column vectors of $\mathbf{I}_{4\times4}$,
we have 
\begin{equation}
\footnotesize
    (\mathbf{M}\mathbf{e}_i)\cdot \mathbf{q}_0 = \mathbf{e}_i\cdot \mathbf{q_I} = 0
\end{equation}
for $i=1,2,3$.
Therefore, $\mathbf{M}\mathbf{e}_i (i=1,2,3)$ form an orthogonal basis of $\pi$.

Given $\mathbf{M}^T\mathbf{q}=w\mathbf{e}_0+x\mathbf{e}_1+y\mathbf{e}_2+z\mathbf{e}_3,$ we have
\begin{equation}
\footnotesize
    \mathbf{q} = w(\mathbf{M}\mathbf{e}_0)+x(\mathbf{M}\mathbf{e}_1)+y(\mathbf{M}\mathbf{e}_2)+z(\mathbf{M}\mathbf{e}_3)
\end{equation}
Therefore, $\boldsymbol{\eta} = (x,y,z)$ is the coordinate of $\pi(\mathbf{q})$ in $\pi$ under the basis of $\mathbf{M}\mathbf{e}_i$.

The Jacobian of the transformation $\mathbf{q}\rightarrow\boldsymbol{\eta}$ is given by
\begin{equation}
\footnotesize
\begin{aligned}
    \mathbf{J} &= \frac{\partial \mathbf{q}}{\partial \boldsymbol{\eta}} 
    = \mathbf{M}\frac{\partial \left(\mathbf{M}^T\mathbf{q}\right)}{\partial \boldsymbol{\eta}} \\
    &= \mathbf{M}\left[\begin{array}{cccc}
        -{x}/{w} & 1 & 0 & 0 \\
        -{y}/{w} & 0 & 1 & 0 \\
        -{z}/{w} & 0 & 0 & 0
        \end{array}\right]
\end{aligned}
\end{equation}
Therefore, the scaling factor from $\boldsymbol{\eta}$ to $\mathbf{q}$ is given by 
\begin{equation}
\footnotesize
    \frac{\mathrm{d}\mathbf{q}}{\mathrm{d}\boldsymbol{\eta}} = \operatorname{det}(\mathbf{J}\mathbf{J}^T)
    = 1+\frac{x^2+y^2+z^2}{w^2}+O(\len \eta^4)
    = 1+O(\len \eta^2).
\end{equation}
Thus
\begin{equation}
\scriptsize
\begin{aligned}
    \mathbf{q}^T\mathbf{MZM}^T \mathbf{q} 
    &= \left[\begin{array}{cccc}
        w & x & y & z
        \end{array}\right]
        \left[\begin{array}{cccc}
         0 &  &  &  \\
         & z_1 &  &  \\
         &  & z_2 &  \\
         &  &  & z_3 \\
        \end{array}\right]
        \left[\begin{array}{c}
        w \\  x \\ y \\ z
        \end{array}\right]\\
    &= \left[\begin{array}{ccc}
         x & y & z
        \end{array}\right]
        \left[\begin{array}{ccc}
          z_1 &  &  \\
           & z_2 &  \\
           &  & z_3 \\
        \end{array}\right]
        \left[\begin{array}{c}
        x \\ y \\ z
        \end{array}\right]\\
        &= \boldsymbol{\eta} \mathbf{\widetilde{Z}}\boldsymbol{\eta}
\end{aligned}
\end{equation}
where we define $\mathbf{\widetilde{Z}} = \operatorname{diag}(z_1, z_2, z_3)$.

For Bingham distribution, we have
\begin{equation}
\footnotesize
\begin{aligned}
    p(\mathbf{q})\mathrm{d}\mathbf{q} &\propto
    \exp\left(\mathbf{q}^T\mathbf{M}\mathbf{Z}\mathbf{M}^T\mathbf{q}\right)\mathrm{d}\mathbf{q} \\
    &= \exp\left(\boldsymbol{\eta}^T\mathbf{\widetilde{Z}}\boldsymbol{\eta}\right)(1+O(\len {\boldsymbol \eta}^2))\mathrm{d}\boldsymbol{\eta}\\
    &= \exp\left(-\boldsymbol{\eta}^T\boldsymbol{\Sigma}^{-1}\boldsymbol{\eta}\right)(1+O(\len {\boldsymbol \eta}^2))\mathrm{d}\boldsymbol{\eta}
\end{aligned}
\end{equation}
which follows the multivariate Gaussian distribution with the covariance matrix as $\boldsymbol{\Sigma}$, where $\boldsymbol{\Sigma}=-\operatorname{diag}(\frac{1}{z_1},\frac{1}{z_2},\frac{1}{z_3})$

For Quaternion Laplace distribution, we have
\begin{equation}
\footnotesize
\begin{aligned}
    p(\mathbf{q})\mathrm{d}\mathbf{q} &\propto \frac{\exp\left(-\sqrt{-{\mathbf{q}^T\mathbf{M}\mathbf{Z}\mathbf{M}^T\mathbf{q}}}\right)}{\sqrt{-{\mathbf{q}^T\mathbf{M}\mathbf{Z}\mathbf{M}^T\mathbf{q}}}}\mathrm{d}\mathbf{q} \\
    &= \frac{1}{\sqrt{2}}\frac{\exp\left(-\sqrt{-\boldsymbol{\eta}^T\mathbf{\widetilde{Z}}\boldsymbol{\eta}}\right)}{\sqrt{-\boldsymbol{\eta}^T\mathbf{\widetilde{Z}}\boldsymbol{\eta}}}(1+O(\len {\boldsymbol{\eta}}^2))\mathrm{d}\boldsymbol{\eta}\\
    &= \frac{1}{\sqrt{2}}\frac{\exp\left(-\sqrt{2\boldsymbol{\eta}^T\boldsymbol{\Sigma}^{-1}\boldsymbol{\eta}}\right)}{\sqrt{2\boldsymbol{\eta}^T\boldsymbol{\Sigma}^{-1}\boldsymbol{\eta}}}(1+O(\len {\boldsymbol{\eta}}^2))\mathrm{d}\boldsymbol{\eta}
\end{aligned}
\end{equation}
which follows the multivariate Laplace distribution 
with the covariance matrix as $\boldsymbol{\Sigma}$, where $\boldsymbol{\Sigma} = -2\operatorname{diag}(\frac{1}{z_1},\frac{1}{z_2},\frac{1}{z_3})$.
\end{proof}

\begin{prop4}
    Denote $\gamma$ as the standard transformation from unit quaternions to corresponding rotation matrices. For rotation matrix $\mathbf{R}\in \SO$ following \emph{Rotation Laplace distribution}, $\mathbf{q}=\gamma^{-1}(\mathbf{R})\in \mathbb{S}^3$ follows \emph{Quaternion Laplace distribution}.
\end{prop4}
\begin{proof}
For a quaternion $\mathbf{{q}}=[{q}_0,{q}_1,{q}_2,{q}_3]$, we use the standard transform function $\gamma$ to compute its corresponding rotation matrix:
\begin{equation}
\footnotesize
    \gamma(\mathbf{{q}}) = 
    \left[\begin{array}{ccc}
        1-2 {q}_2^2-2 {q}_3^2 & 2 {q}_1 {q}_2-2 {q}_0 {q}_3 & 2 {q}_1 {q}_3+2 {q}_0 {q}_2 \\
        2 {q}_1 {q}_2+2 {q}_0 {q}_3 & 1-2 {q}_1^2-2 {q}_3^2 & 2 {q}_2 {q}_3-2 {q}_0 {q}_1 \\
        2 {q}_1 {q}_3-2 {q}_0 {q}_2 & 2 {q}_2 {q}_3+2 {q}_0 {q}_1 & 1-2 {q}_1^2-2 {q}_2^2
        \end{array}\right]
\end{equation}
Let $\mathbf{u}=\gamma^{-1}(\mathbf{U})$,$\mathbf{v}=\gamma^{-1}(\mathbf{V})$ and
\begin{equation}
\footnotesize
    \mathbf{\widetilde{q}} = {[\widetilde{q}_0,\widetilde{q}_1,\widetilde{q}_2,\widetilde{q}_3]}^T=\gamma^{-1}\left(\mathbf{U}^T\mathbf{R}\mathbf{V}\right)
    =\overline{\mathbf{u}} \mathbf{q}\mathbf{v}
\end{equation}
Note that the transformation $\mathbf{q}\rightarrow\overline{\mathbf{u}} \mathbf{q}\mathbf{v}$ is an orthogonal transformation on $\mathbb{S}^3$.
    Therefore, there exists an orthogonal Matrix $\mathbf{M}$, such that
\begin{equation}
\footnotesize
    \mathbf{M}^T\mathbf{q} = \overline{\mathbf{u}} \mathbf{q}\mathbf{v} = \mathbf{\widetilde{q}}
\end{equation}
The scaling factor from quaternions to rotation matrices is given by 
\begin{equation}
\footnotesize
    \mathrm{d}\mathbf{R} = \frac{1}{2\pi^2}\mathrm{d}\mathbf{q}
\end{equation}
Suppose $\mathbf{R}$ follows Quaternion Laplace distribution as
\begin{equation}
\footnotesize
    p(\mathbf{R})\mathrm{d}\mathbf{R} = \frac{1}{F}\frac{\exp\left(-\sqrt{\tr{\mathbf{S}-\mathbf{A}^T\mathbf{R}}}\right)}{\sqrt{\tr{\mathbf{S}-\mathbf{A}^T\mathbf{R}}}}\mathrm{d}\mathbf{R}
\end{equation}
Given
\begin{equation}
\footnotesize
\begin{aligned}
    \tr{\mathbf{S}-\mathbf{A}^T\mathbf{R}} &= \tr{\mathbf{S}-\mathbf{S}\mathbf{U}^T\mathbf{RV}} 
    = \sum_{(i,j,k)\in I}2(s_j+s_k)q_i^2 \\
    &= 2\mathbf{\widetilde{q}}^T
    \left[\begin{smallmatrix}
        0 & & & \\
        & s_2 + s_3 &  &  \\
        & & s_1 + s_3 &  \\
        & &  & s_1 + s_2
        \end{smallmatrix}\right]
    \mathbf{\widetilde{q}}
\end{aligned}
\end{equation}
we have
\begin{equation}
\scriptsize
\begin{aligned}
    p(\mathbf{R})\mathrm{d}\mathbf{R} &= \frac{1}{2\pi^2F}\frac{\exp\left(-\sqrt{2\mathbf{\widetilde{q}}^T
    \left[\begin{smallmatrix}
        0 & & & \\
        & s_2 + s_3 &  &  \\
        & & s_1 + s_3 &  \\
        & &  & s_1 + s_2
        \end{smallmatrix}\right]
    \mathbf{\widetilde{q}}}\right)}{\sqrt{2\mathbf{\widetilde{q}}^T
    \left[\begin{smallmatrix}
        0 & & & \\
        & s_2 + s_3 &  &  \\
        & & s_1 + s_3 &  \\
        & &  & s_1 + s_2
        \end{smallmatrix}\right]
    \mathbf{\widetilde{q}}}}\mathrm{d}\mathbf{q} \\
    &= \frac{1}{2\pi^2F}\frac{\exp\left(-\sqrt{2\mathbf{q}^T\mathbf{M}
    \left[\begin{smallmatrix}
        0 & & & \\
        & s_2 + s_3 &  &  \\
        & & s_1 + s_3 &  \\
        & &  & s_1 + s_2
        \end{smallmatrix}\right]
    \mathbf{M}^T\mathbf{q}}\right)}{\sqrt{2\mathbf{q}^T\mathbf{M}
    \left[\begin{smallmatrix}
        0 & & & \\
        & s_2 + s_3 &  &  \\
        & & s_1 + s_3 &  \\
        & &  & s_1 + s_2
        \end{smallmatrix}\right]
    \mathbf{M}^T\mathbf{q}}}\mathrm{d}\mathbf{q} \\
    &= \frac{1}{2\pi^2F}\frac{\exp\left(-\sqrt{-\mathbf{q}^T\mathbf{M}\mathbf{Z}\mathbf{M}^T\mathbf{q}}\right)}{\sqrt{-\mathbf{q}^T\mathbf{M}\mathbf{Z}\mathbf{M}^T\mathbf{q}}}\mathrm{d}\mathbf{q},
\end{aligned}
\end{equation}
where $\mathbf{M}$ is an orthogonal matrix and $\mathbf{Z}=-2\operatorname{diag}(0,s_2+s_3,s_1+s_3,s_1+s_2)$ is a $4\times4$ diagonal matrix.
\end{proof}

\ree{
\textbf{Elaboration of Eq. \ref{eq:prdr} in the main paper}

Given $\mathbf{R}_0 = \mathbf{UV}^T$ and $\mathbf{\widetilde{R}}=\mathbf{R}_0^T\mathbf{R}$,
\begin{equation}
\scriptsize
\begin{aligned}
 p(\mathbf{R})\mathrm{d}\mathbf{R} &\propto 
 \frac{\exp\left(\sqrt{\operatorname{tr}(\mathbf{S}-{\mathbf{A}^T\mathbf{R}})}\right)}{\sqrt{\operatorname{tr}(\mathbf{S}-{\mathbf{A}^T\mathbf{R}})}} \mathrm{d}\mathbf{R} 
 = \frac{\exp\left(\sqrt{\operatorname{tr}(\mathbf{S}-\mathbf{V}\mathbf{S}\mathbf{U}^T\mathbf{R})}\right)}{\sqrt{\operatorname{tr}(\mathbf{S}-\mathbf{V}\mathbf{S}\mathbf{U}^T\mathbf{R})}} \mathrm{d}\mathbf{R} 
 = \frac{\exp\left(\sqrt{\operatorname{tr}(\mathbf{S}-\mathbf{S}\mathbf{U}^T\mathbf{R}\mathbf{V})}\right)}{\sqrt{\operatorname{tr}(\mathbf{S}-\mathbf{S}\mathbf{U}^T\mathbf{R}\mathbf{V})}} \mathrm{d}\mathbf{R} \\
 &= \frac{\exp\left(\sqrt{\operatorname{tr}(\mathbf{S}-\mathbf{S}\mathbf{U}^T\mathbf{R}_0\mathbf{\widetilde{R}}\mathbf{V})}\right)}{\sqrt{\operatorname{tr}(\mathbf{S}-\mathbf{S}\mathbf{U}^T\mathbf{R}_0\mathbf{\widetilde{R}}\mathbf{V})}} \mathrm{d}\mathbf{R} 
 = \frac{\exp\left(\sqrt{\operatorname{tr}(\mathbf{S}-\mathbf{S}\mathbf{V}^T\mathbf{\widetilde{R}}\mathbf{V})}\right)}{\sqrt{\operatorname{tr}(\mathbf{S}-\mathbf{S}\mathbf{V}^T\mathbf{\widetilde{R}}\mathbf{V})}} \mathrm{d}\mathbf{R}
\end{aligned}
\end{equation}
}

\section{More Implementation Details}
For fair comparisons, we follow the implementation designs of \cite{mohlin2020probabilistic} and merely change the distribution from matrix Fisher distribution to our Rotation Laplace distribution.
We use pretrained ResNet-101 as our backbone, and encode the object class information (for single-model-all-category experiments) by an embedding layer that produces a 32-dim vector. We apply a 512-512-9 MLP as the output layer. 

The batch size is set as 32.
We use the SGD optimizer and start with the learning rate of 0.01.
For ModelNet10-SO3 dataset, we train 50 epochs with learning rate decaying by a factor of 10 at epochs 30, 40, and 45. For Pascal3D+ dataset, we train 120 epochs with the same learning rate decay at epochs 30, 60 and 90.

\end{document}